\documentclass[11pt]{article}
\usepackage{graphicx}

\usepackage{indentfirst,csquotes}

\usepackage{amssymb,amsthm,amsmath}
\usepackage{xcolor,paralist,hyperref,titlesec,fancyhdr,etoolbox}
\newtheorem{theorem}{Theorem}[]

\newtheorem{lemma}[theorem]{Lemma}
\newtheorem{proposition}[theorem]{Proposition}
\newtheorem{corollary}[theorem]{Corollary}

\hypersetup{ colorlinks=true, linkcolor=black, filecolor=black, urlcolor=black }

\usepackage[utf8]{inputenc}
\usepackage[T1]{fontenc}
\usepackage{hyperref}
\usepackage{url}
\usepackage{booktabs}
\usepackage{array}
\usepackage{graphicx}
\usepackage{amsmath,amssymb,amsthm}
\usepackage{mathtools}
\usepackage{algorithm}
\usepackage{algpseudocode}
\usepackage{flafter}
\usepackage{placeins}
\usepackage{microtype}
\usepackage{xcolor}
\usepackage{tikz}
\usetikzlibrary{calc,decorations.pathreplacing}
\usepackage{natbib}

\usepackage{hyperref}       
\hypersetup{
	colorlinks   = true, 
	urlcolor     = blue, 
	linkcolor    = blue, 
	citecolor   = blue 
}
\usepackage{multirow}
\usepackage{url}            

\theoremstyle{remark}
\newtheorem{remark}[theorem]{Remark}

\newcommand{\E}{\mathbb E}
\newcommand{\Prob}{\mathbb P}
\newcommand{\one}{\mathbf 1}
\newcommand{\pos}[1]{\left[#1\right]_+}
\newcommand{\WKT}{W_{K,T}}
\newcommand{\Rfs}{R_T^{\mathrm{fs}}}
\newcommand{\Dist}{\mathrm{Dist}}
\newcommand{\polyinf}{\textnormal{\textsc{Poly-INF}}}
\DeclareMathOperator{\KL}{KL}

\usepackage{lipsum}

\begin{document}
\title{The Role of Coordinates in Pareto Regret for Adversarial Multi-Objective Bandits}
\author{%
  Changkun Guan\\
  H. Milton Stewart School of Industrial and Systems Engineering \\
  Georgia Institute of Technology \\
  \texttt{cguan62@gatech.edu} \\
\and 
  Mengfan Xu\\
  Mechanical and Industrial Engineering \\
  University of Massachusetts Amherst \\
  \texttt{mengfanxu@umass.edu} \\
}
\date{\today}


\maketitle

\let\thefootnote\relax

\begin{abstract}
Adversarial multi-objective bandits hold the potential to help us optimize choices (arms) whose reward is a multidimensional vector chosen by an adversary and whose performance is measured by Pareto regret. We define loss as one minus reward and measure the easiness of a coordinate by the smallest cumulative loss of the arms on it, and call the coordinate easier when this quantity is smaller. Existing work suggests that in theory an easier coordinate may reduce Pareto regret. However, in practice, one  may not know which coordinate is easier. On the negative side, we show that this lack of information eliminates the possibility: a smaller cumulative loss does not improve the worst-case order of Pareto regret. Precisely, let \(L_d\) be
the smallest cumulative loss along coordinate $d$ over $T$ rounds.
For \(K\ge4\) arms, \(T\ge6\) rounds, and at least 2 coordinates, we prove
that the minimax expected Pareto regret is
\(\Omega(\min\{T-L_0,\sqrt{K(T-L_0)}\})\). It is monotonically decreasing in \(L_0\), even when \(L_0=\min_d L_d\) itself is known. 
On the positive side, this result motivates the possibility that other coordinates, not just the easy one, may suffice to attain the optimal rate of Pareto regret.  When $L_0$ is known, we apply Poly-INF to a fixed coordinate and obtain an upper bound on Pareto regret that exhibits the same order and thus matches the lower bound. Without such knowledge, we develop a reward-doubling version of Poly-INF that adapts to this unknown quantity while still attaining the matching minimax rate. Another implication is that it has no extra \(\log T\) factor and is independent of the number of coordinates.
\end{abstract} 

\section{Introduction}
 
Multi-armed bandits \citep{auer2002finite, auer2002nonstochastic} have elegantly enabled sequential decision making in a broad spectrum of applications, dating back to classical recommender systems \citep{mahadik2020fast, yang2020exploring, zhu2023scalable, yi2023online} and extending to emerging large language models \citep{bouneffouf2026multi, chen2026greedy, poon2026online}. Here, a decision maker sequentially selects an option from multiple options (often called arms) and receives information about its performance in the form of a scalar value (often called a reward). The goal is to maximize the cumulative performance of the chosen options, usually in the language of regret, which measures the cumulative difference between the reward obtained by always selecting the best arm in hindsight and the reward received by the decision maker. Bandit models broadly fall into two categories, stochastic \citep{auer2002finite} and adversarial \citep{auer2002nonstochastic}, depending on if rewards are independently drawn from a fixed distribution or chosen by an adversary. Both have witnessed measurable successes in theory and practice via the community's joint efforts \citep{bubeck2012regret, bubeck2013bounded, liau2018stochastic, komiyama2015optimal, bouneffouf2020survey, allenberg2006hannan,neu2015first,lykouris2018small,wei2018more}.

The evolution of these applications has recently imposed the challenge of optimally selecting options whose performance includes multiple objectives (coordinates)  \citep{lacerda2017multi, mehrotra2020bandit, wanigasekara2019learning, li2026efficient, xue2026cost}. Multi-objective bandits (MO-MAB) have emerged as a promising remedy \citep{drugan2013designing, busa2017multi, garivier2024sequential, xu2023pareto, roijers2017interactive, yahyaa2014annealing, auer2016pareto,ararat2023vector, huyuk2021multi, lu2019multi}, where the reward of each arm is a multidimensional vector and the optimality of arms, namely Pareto optimality, is defined via Pareto dominance. Pareto dominance \citep{drugan2013designing} compares two arms across coordinates. An arm is Pareto optimal if no other arm is better than it on every coordinate; otherwise, it is suboptimal and Pareto dominated. Importantly, as Pareto dominance requires another arm to be better across all coordinates, performing well on even one coordinate can be sufficient to prevent an arm from being dominated. This geometry also motivates notions of Pareto regret \citep{drugan2013designing, xu2023pareto}, related to the minimum value to be added uniformly to \emph{all coordinates} when comparing with Pareto optimal arms in hindsight, although their precise definitions differ between stochastic and adversarial models. Meanwhile, the easiness of individual coordinates can differ, i.e., what it takes to optimize along them. This heterogeneity, coupled with Pareto dominance, suggests that coordinates, such as easier ones, may play a role in Pareto regret and thus in the performance of MO-MAB. 

To date, existing studies on MO-MAB have jointly addressed this question for the stochastic problem, where the reward vectors are independent and identically distributed across the time horizon. The seminal studies measure Pareto regret by cumulative Pareto gaps  \citep{drugan2013designing,turgay2018contextual,lu2019multi}, specify how objectives are combined or ordered \citep{busafekete2017gini,huyuk2021multi}, or identify Pareto optimal arms \citep{auer2016pareto,ararat2023vector}. Here, coordinate easiness is measured by the suboptimality gap in expected reward between its top two arms: a larger gap makes the best arm easier to distinguish. Under certain assumptions, \citet{guan2026stochastic} show that Pareto regret can follow the easiest coordinate under this measure. The algorithm searches for such a coordinate using confidence bounds across arms and coordinates, then certifies an arm as Pareto optimal. Notably, fixed reward distributions make this search feasible as repeated samples provide persistent statistical information and thus hope for learning. However, such an assumption can oftentimes measurably break for real-world problems. Consequently, these arguments, and the insights behind them, do not extend to adversarial MO-MAB.

Much less is known about adversarial MO-MAB in general, to the best of our knowledge. Encouragingly, several results from classical adversarial bandits and the limited literature on adversarial MO-MAB suggest that an easy coordinate may still help. Notably, 
\citet{xu2023pareto} introduce a definition of Pareto regret and show that it can be 
upper bounded by classical regret on any coordinate fixed before play. Accordingly, they select an arbitrary coordinate, as \emph{their guarantees depend only on the horizon and do not reflect the roles of coordinates}. However, as suggested by the line of work on adversarial bandits, the horizon dependence can be overly pessimistic. In particular, the minimum cumulative loss among the arms can be much smaller than the horizon. Here, loss is one minus reward. Seminal algorithms in the field have exploited this structure and achieved regret upper bounds that increase with this minimum cumulative loss
\citep{allenberg2006hannan,neu2015first,lykouris2018small,wei2018more}. Hence, \emph{coordinate 
easiness may also matter as in stochastic MO-MAB: if the coordinate with the smallest cumulative loss were known, 
one might hope to obtain smaller Pareto regret by optimizing along it.} This raises the underexplored question of whether coordinate-specific difficulty can similarly be exploited in adversarial MO-MAB.

Yet, the decision maker generally does not know which 
coordinate attains this minimum. This leaves a dilemma. Directly applying existing scalar algorithms 
to an arbitrary coordinate seems to yield a Pareto regret guarantee scaling with the loss on that coordinate instead of the horizon, 
which can be larger than the minimum across coordinates. Identifying the coordinate 
attaining this minimum, on the other hand, may itself incur substantial regret. Importantly, we do not yet know whether either difficulty is fundamental in adversarial MO-MAB. The dependence on the loss of the chosen coordinate may be an artifact of optimizing along that coordinate, and the apparent need to identify the minimum-loss coordinate may simply be a consequence of this dependence. This raises:
\begin{center}
\emph{Does the Pareto regret scale with minimum coordinate loss in some sense, and should we try to identify coordinates en route to the optimality?}
\end{center}

Counterintuitively, the optimal rate is governed by the minimum coordinate loss but is monotonically decreasing in it, and attaining this rate does not require identifying the coordinate that realizes the minimum. On coordinate $d$, let $L_d$ denote the 
minimum cumulative loss among the arms, and let $L^\star=\min_d L_d$. Thus, a small 
$L^\star$ means that some arm has small cumulative loss (thus large cumulative reward) on at least one coordinate, 
but neither the arm nor the coordinate attaining $L^\star$ is known. 
\emph{Our first part is surprisingly negative.} For every $L_0\in[0,T]$, we derive a lower bound of order
$\min\{T-L_0,\sqrt{K(T-L_0)}\}$
over the class of instances satisfying $L^\star=L_0$. This lower bound is monotonically decreasing in $L_0$: making the easiest coordinate easier, i.e., decreasing its minimum cumulative loss instead increases the lower bound. \emph{On the flip side, this negative result has an optimistic implication: perhaps we do not need identify the coordinate.} We confirm this by showing that the optimal rate can be attained using any fixed coordinate \emph{but with deliberately designed algorithms along the chosen coordinate}. We establish this through both existing and new algorithms. The resulting regret bounds match the lower bound, implying minimax optimality, without an extra logarithmic factor and without dependence on the number of coordinates.


Our main contributions are as follows. We extend Pareto regret in \citet{xu2023pareto} so that it is well defined. For each coordinate, we measure its easiness by the minimum cumulative loss among the arms on that coordinate, and a
coordinate is easy when some arm has small cumulative loss in it. Using this measure, we determine the optimal worst-case order of Pareto regret as a function of $L_0$, the easiness on the easiest coordinate. Surprisingly, this function is monotone decreasing in $L_0$. Hence in contrast to the scaling in bandits, it yields a new lower bound, authentic and unique to MO-MAB. Even when this loss is \(0\), the lower bound has the same order as in the unrestricted case. 

We then develop routes to optimize Pareto regret without finding this easy coordinate; in fact any coordinate suffices for minimax \emph{optimality} but the algorithm along the coordinate has to be targeted. We first apply Poly-INF
\citep{audibert2010regret} to one coordinate fixed before play. When the cumulative loss on the easiest coordinate is known, we set the reward scale using the horizon minus this loss. When it is unknown, we develop a reward-doubling version of Poly-INF that restarts the algorithm and doubles the scale whenever the reward collected in the current run reaches that scale. 

Perhaps surprisingly, we prove that the upper bound of the proposed methods matches the lower bound, particularly with respect to $L_0$, the easiness of the easiest coordinate. This also implies the minimax optimality of the methods and the tightness of the lower bound. As a separate contribution, they also sharpen the dependence on the problem parameters. They remove the common $\sqrt{\log T}$ factor and are independent of the number of coordinates, with the order governed by the intrinsic complexity ($L_0$). We also adapt EXP3, EXP3-IX, and Tsallis-INF and derive corresponding bounds to assess whether these popular methods attain the optimal rate; the guarantees incur extra factors.

\section{Problem Formulation}
\label{sec:pre}

\subsection{Setting and notation}
\label{sec:setting}

We study the adversarial multi-objective multi-armed bandit (MO-MAB)
problem. There are \(K\) arms, \(T\) rounds, and
\(D\) reward coordinates. Before the first round, an oblivious adversary specifies the
reward vectors \(r = r^{i,t}=(r_d^{i,t})_{1\le d\le D}\in[0,1]^D\) of every
arm \(i\in[K]\) in every round \(t\in[T]\). We call this collection $r$ 
the reward sequence. In each round \(1 \leq t \leq T\), the decision maker (learner) chooses
an arm \(a_t\) based on its past observations, possibly using randomization, 
and observes only the chosen arm's reward vector \(r^{a_t,t}\).

Pareto dominance is a standard tool for comparing vectors in multi-objective bandits, and we adopt it here. For vectors \(a,b\in\mathbb R^D\), \(a\) weakly dominates \(b\) if \(a_d\ge b_d\) for every coordinate \(d\); we write \(a\succeq b\), or equivalently \(b\preceq a\). If in addition \(a\ne b\), then \(a\) dominates \(b\), denoted by \(a\succ b\). Vectors \(a\) and \(b\) are incomparable, denoted by \(a\,\|\,b\), if \(a_d>b_d\) and \(a_{d'}<b_{d'}\) for some coordinates \(d,d'\). Finally, \(\one\in\mathbb R^D\) denotes the all-ones vector. 

Let \(G_i:=\sum_{t=1}^T r^{i,t}\) denote the cumulative reward vector of arm \(i\). Perhaps unsurprisingly, this vector is unknown until the end; in other words, it is unknown during the process. We define the set of distinct nondominated cumulative reward vectors as
\begin{equation}
 O':=\{G_i:i\in[K],\ \nexists j\in[K]\text{ with }G_j\succ G_i\}.
 \label{eq:vector-front}
\end{equation}
As \(O'\) is defined over reward vectors rather than arms, tied arms contribute only one vector, and \(O'\) is always nonempty. We call \(O'\) the Pareto front of the cumulative reward vectors. Ultimately, we care about which arms are optimal. An arm whose cumulative reward vector belongs to \(O'\) is Pareto optimal, and the collection of such arms, denoted by $O$ is called the Pareto-optimal set, or Pareto set for short. Naturally, both \(O'\) and $O$ are unknown to the learner.

\subsection{Regret}
\label{sec:criterion}

Regret has long been the go-to language in bandits for measuring how far our arm choices are from choosing the optimal arms as if with hindsight knowledge of \(G_i\), which, of course, is unavailable throughout. We first recap the scalar regret in classical bandits (also called coordinate regret) by considering an individual coordinate of the reward vectors. This is helpful for distinguishing coordinates and thus understanding their roles. Then we proceed to the main objective, Pareto regret, which measures performance considering the entire reward vectors with Pareto dominance. 

Let
\(A_T:=\sum_{t=1}^T r^{a_t,t}\) denote the cumulative reward vector of the learner, which chooses an arm sequence \(\mathcal{A} = \{a_t\}_{t=1}^{T}\); scalar $A_{T,d}$ is the value of this vector in coordinate $d$.  Along each coordinate \(d\), the standard comparison with the best fixed arm
\citep{auer2002nonstochastic,lattimore2020bandit,bubeck2012regret} defines the coordinate regret, which reads as
\begin{equation}
 R_T^d:=\max_i G_{i,d}-A_{T,d}. 
 \label{eq:coordinate-regret}
\end{equation} 

We next denote Pareto regret by \(\Rfs(r)\), or simply \(\Rfs\) with a slight abuse of notation: 
\begin{equation}
 \Rfs 
 :=\inf\{\epsilon\ge0:\nexists\sigma\in O'
               \text{ with }\sigma\succ A_T+\epsilon\one\}
 =\pos{\max_{i\in[K]}\min_{d\in[D]}(G_{i,d}-A_{T,d})}.
 \label{eq:pareto-regret}
\end{equation}
Eq.~\eqref{eq:pareto-regret} has a particular elegant interpretation.
For every \(\epsilon\ge0\), we have the  equivalence: 
\(
 \Rfs>\epsilon
 \quad\Longleftrightarrow\quad
 \exists i\in[K]\text{ s.t. }
 G_{i,d}>A_{T,d}+\epsilon\text{ for every }d\in[D].
\) 
Intuitively, it measures the largest uniform improvement that one arm offers over the learner's cumulative reward across all coordinates. Positive regret captures a uniform advantage of one fixed arm over the learner across all objectives, while zero regret means that every fixed arm fails to improve on the learner in at least one objective. 

To see why the second equality holds, we prove the following result regarding the minimax expression.  \begin{proposition}[Pathwise representation]
\label{prop:pareto-representation}
For every reward sequence and every realized play path, the two
expressions in \eqref{eq:pareto-regret} are equal. The maximum over arms
can equivalently be taken over the distinct nondominated vectors in $O'$.
The identity holds even when cumulative arm vectors coincide or no shift
makes the learner incomparable with every vector on the front.
\end{proposition}


However, the converse does not hold; zero Pareto regret does not imply that the learner's cumulative reward vector is Pareto nondominated. For example, let \(A_T=(T,0)\) and \(G_i=(T,T)\). Although \(G_i\succ A_T\), we have \(\Rfs=0\), because the two vectors are equal in the first coordinate and hence there is no positive uniform improvement attained at zero shift. More generally, adding a coordinate with identical rewards across arms in every round makes \(\Rfs=0\) for every policy. At its core, zero Pareto regret guarantees that no single fixed arm can strictly improve on the learner in every objective. 

Notably, our definition builds on and extends that of \citet{xu2023pareto}. They define Pareto regret as 
\(\min\{\epsilon\ge0:A_T+\epsilon\one\,\|\,\sigma
\text{ for every }\sigma\in O'\},
\)
on the domain where \(A_T\) is weakly Pareto dominated by every vector in \(O'\).
However, it does not cover three edge cases. First, the feasible set in the minimum can be empty. For \(A_T=(0,0)\) and \(O'=\{(1,1)\}\), no shift makes \(A_T+\epsilon\one\) incomparable with \((1,1)\), so \(\Dist(A_T,O)\) is undefined. Second, feasible shifts may exist but the minimum may not be attained. For \(A_T=(0,0)\) and \(O'=\{(1,2)\}\), the feasible shifts are \(\epsilon\in(1,2)\), which has an infimum but no minimum. Third, ties can make the Pareto front \(O'\) empty, as the Pareto set \(O\) is defined before \(O'\) therein. This order matters. If two arms have the same cumulative reward vector \(G_i=(1,1)\), neither satisfies the non-domination condition, so both can be excluded from \(O\).   This notion is also closely related to the Pareto suboptimality gap of \citet{jiang2023multiobjective} for multi-objective online convex optimization. In our setting, however, the comparators are fixed arms, and performance is evaluated without relying on convexity of the reward sequence. 

We address these issues by first deduplicating vectors in \(O'\) and then defining the Pareto set of arms, while using the boundary of the same uniform-shift condition. This preserves the Pareto comparison and ensures a well defined regret for \emph{every reward sequence}. To formalize the relation, let us define 
\refstepcounter{equation}\label{eq:xu-endpoints}%
\(a_-=\max_{\sigma\in O}\min_d(\sigma_d-A_{T,d}),\qquad
a_+=\min_{\sigma\in O}\max_d(\sigma_d-A_{T,d})\)~\textup{(\theequation)}.

The next proposition shows that, for distinct arm vectors, our definition agrees exactly with that of \citet{xu2023pareto} whenever theirs is well defined: replacing their minimum by an infimum yields \(\Rfs\). At the same time, \(\Rfs\) remains well defined even when their feasible set is empty.
\begin{proposition}[Comparison]
\label{prop:xu-relation}
Suppose \(O'\ne\varnothing\) and \(A_T\preceq\sigma\) for every
\(\sigma\in O'\). The feasible shifts $\epsilon$ in \citet{xu2023pareto} form
\((a_-,a_+)\). When this interval is nonempty, its infimum $a_-$ is at most \(\Rfs\),
but its minimum is not attained.
If vectors $G_i$ are distinct, their front equals \(O'\),
and the infimum equals \(\Rfs\).
Eq.~\eqref{eq:pareto-regret} holds on every outcome, including
outcomes where this interval is empty.
\end{proposition}

\subsection{Mathematical form of the main question}
\label{sec:objective}

We are now ready to present the precise form of our main question of interest. For coordinate \(d\), let
\(U_d(r):=\max_i G_{i,d}\) be the cumulative reward of its best arm,
and let \(L_d(r):=T-U_d(r)\) be that arm's cumulative loss, obtained by subtracting each step's reward coordinate from \(1\).  Notably, the reward and loss notation give identical regret. In MO-MAB, we have \(\Rfs\le\pos{\max_i(G_{i,d}-A_{T,d})}=\pos{R_T^d}.\) In classical bandits, the regret bound becomes smaller as \(L_d\) decreases~\citep{lykouris2018small}.    Thus, a coordinate with smaller \(L_d\) can give a tighter bound on Pareto regret. The smallest loss and the corresponding reward are 
\refstepcounter{equation}\label{eq:LstarU}%
\(L^\star(r):=\min_d L_d(r)=\min_{i,d}\sum_{t=1}^T\bigl(1-r_d^{i,t}\bigr)=T-U, 
U^{\star}(r):=\max_d U_d=\max_{i,d}G_{i,d}\)~\textup{(\theequation)}. We use \(d^\star(r)\) to denote a coordinate attaining this minimum. In other words, \(L^\star(r)=L_1\) when \(D=1\); the regret scales with \(L_1\). 

The first part of the main question thus concerns whether Pareto regret admits an analogue through \(L^\star(r)\) when \(D\geq 2\). If such a dependence exists, then all reward sequences with the same value \(L^\star(r)=L_0\) must share the same bound. The sharpest such guarantee at \(L_0\) is therefore obtained by taking the worst case over these sequences and then optimizing over policies. Precisely, for a fixed \(L_0\in[0,T]\), the minimax Pareto regret among reward sequences with \(L^\star=L_0\) is
\begin{align}
  & \mathcal V_{K,D,T}(L_0)
  :=
  \inf_{\mathcal A}\sup_{r:\,L^\star(r)=L_0}
  \E_{\mathcal A}[\Rfs] 
  \label{eq:conditional-value}
\end{align}

The infimum is over all policies \(\mathcal A\), so that it represents the optimal policy; the expectation accounts for the policy's randomness. The supremum is over all instances with the same \(L_0\), so that it represents the worst case. Thus, \(\mathcal V_{K,D,T}(L_0)\) gives the optimal worst-case Pareto regret for \(L_0\). Fixing \(L^\star=L_0\) means that the losses on other coordinates can still vary. We also consider the case where all coordinate losses are fixed, but the learner does not know which value corresponds to which coordinate. Specifically, define \(V^{\mathrm{perm}}_{K,D,T}(r):=\inf_{\mathcal A}\sup_{r\in\mathcal C(r)} \E_{\mathcal A}[\Rfs]\), where \(\mathcal C(r)\) contains all reward sequences obtained by permuting the coordinates of the same reward sequence \(r\). The question is whether they depend on \(L_0\) and, further, scale with \(L_0\), as suggested by classical bandits. 

This also motivates us the second part. Let \(\mathcal A\) denote a policy that
commits to one coordinate \(d\) before the first round, without knowing \(d^\star(r)\), and operates only on that coordinate. \citet{xu2023pareto} show that such a policy can attain the optimal order w.r.t. the horizon \(T\). With a more carefully designed policy \(\mathcal A\), 
can the upper bound \(\sup_{r:\,L^\star(r)=L_0}
\E_{\mathcal A}[\Rfs]\) attain the same order as 
\(\mathcal V_{K,D,T}(L_0)\), particularly w.r.t. the coordinate-related quantity \(L_0\)?







\section{Lower Bound Analyses}
\label{sec:lb}

To answer the first part, this section is dedicated to the analyses of the minimax Pareto regret. Before analyzing the general dependence on \(L_0\), we begin with the special case \(L_0=0\) to gather insights with more interpretable intuition and motivate the general-case analysis. Here, \(L_0=0\) means a problem class where there exists a coordinate \(d^\star\) on which the best arm incurs zero cumulative loss; equivalently, that arm receives reward one on coordinate \(d^\star\) in every round.  The main challenge is that neither a lower bound on coordinate regret nor a condition on coordinate loss directly implies a lower bound on Pareto regret, which depends on the joint rewards across coordinates. Our construction of instances for this class links the coordinates to reflect  coordinate-level properties. 

Let us first introduce several notations for ease of presentations. We denote \(q:=\lfloor K/2\rfloor\), \(W_q:=\min\{T,\sqrt{qT}\}\), \(h:=\lfloor W_q/256\rfloor\), \(n:=T-h,\) and \(\alpha:=\frac18\min\{1,\sqrt{q/n}\}.\) 
We consider instances with \(W_q\ge512\).  We split \(2q\) arms into groups \(A_1,\ldots,A_q\) and
\(B_1,\ldots,B_q\); each group has at least two arms when we assume the number of arms \(K\ge4\).
During the first \(h\) rounds, every \(A_i\)
receives \((1,0)\), and every \(B_i\) receives \((0,1)\). Before play,
choose a case \(J\in\{A,B\}\) uniformly and a hidden index
\(I\in[q]\) uniformly. The remaining rewards are shown in
Figure~\ref{fig:construction}. The \(Y_{i,t}\) are independent Bernoulli rewards with mean
\(1/2+\alpha\) for \(i=I\) and mean \(1/2\) otherwise. If \(K\) is
odd, the remaining arm receives \((0,0)\) during the first block and the
reward \((0,1/2)\) in case \(A\) or \((1/2,0)\) in case \(B\) during
the remaining rounds. The case, hidden index, and all Bernoulli rewards
are sampled before play. Each realization is therefore an oblivious
reward sequence. In case \(A\), every \(A_i\) receives reward one in
coordinate 1 on all \(T\) rounds. In case \(B\), every \(B_i\) receives
reward one in coordinate 2 on all \(T\) rounds. Hence every realization
has \(L^\star=0\). Probabilities and expectations below include these
draws and the algorithm's randomization.

\begin{figure}[!ht]
 \centering
 \begin{tikzpicture}[font=\small]
  \foreach \x in {0,7.1} {
   \draw[black!45,rounded corners=2pt] (\x,0) rectangle (\x+6.6,-3.2);
   \fill[black!7] (\x+1.2,-0.7) rectangle (\x+3.8,-2.45);
   \node at (\x+.6,-1) {Group};
   \node at (\x+2.5,-1) {First \(h\) rounds};
   \node at (\x+5.15,-1) {Remaining \(n\) rounds};
   \node at (\x+.6,-1.6) {\(A_i\)};
   \node at (\x+.6,-2.15) {\(B_i\)};
  }
  \node[font=\small\bfseries] at (3.3,-.35) {Case A, coordinate 1 has zero loss};
  \node[font=\small\bfseries] at (10.4,-.35) {Case B, coordinate 2 has zero loss};
  \node at (2.5,-1.6) {\((1,0)\)};
  \node at (2.5,-2.15) {\((0,1)\)};
  \node at (5.15,-1.6) {\((1,Y_{i,t})\)};
  \node at (5.15,-2.15) {\((0,1/2)\)};
  \node at (9.6,-1.6) {\((1,0)\)};
  \node at (9.6,-2.15) {\((0,1)\)};
  \node at (12.25,-1.6) {\((1/2,0)\)};
  \node at (12.25,-2.15) {\((Y_{i,t},1)\)};
  \node[align=center,font=\footnotesize] at (3.3,-2.8)
       {The arm with larger mean reward is in group \(A\).};
  \node[align=center,font=\footnotesize] at (10.4,-2.8)
       {The arm with larger mean reward is in group \(B\).};
  \node[align=center,font=\footnotesize] at (6.85,-3.65)
       {The first \(h\) rounds are identical in both cases.\\
        The case and all rewards are selected before play.};
 \end{tikzpicture}
 \caption{The zero-loss construction. The first block is identical in
 both cases; the remaining rewards determine whether group \(A\) or
 group \(B\) contains the zero-loss coordinate and the arm with larger
 mean reward in the other coordinate.}
 \label{fig:construction}
\end{figure}

The two cases create the two required shortfalls against the same hidden
arm. Let \(X_A\) and \(X_B\) be the numbers of first-block pulls from the
two groups. The algorithm's first-block shortfalls in the two possible
zero-loss coordinates are
\(
 d_A=h-X_A,\qquad d_B=h-X_B,\qquad d_A+d_B\ge h.
\)
The first-block observations do not reveal \(J\), so after any such
history,
\begin{equation}
 \Pr\{d_J\ge h/2\mid\text{first-block history}\}\ge\frac12.
 \label{eq:prefix-deficit-main}
\end{equation}
Every arm in group \(J\), including the hidden arm, receives reward one
thereafter in this coordinate. Its shortfall \(d_J\) therefore cannot
decrease.

In the other coordinate, the remaining rounds form a \(q\)-armed bandit
problem with hidden mean advantage \(\alpha\). We may reveal \(J\) before
these rounds; doing so can only help the algorithm. Let \(C_I\) be the
hidden arm's reward and \(S\) the algorithm's reward in this coordinate
over the remaining \(n\) rounds. With
\(V:=\min\{n,\sqrt{qn}\}\), the hard-bandit estimate in
Lemma~\ref{lem:bandit-positive-tail} gives, after every first-block
history,
\begin{equation}
 \Pr\{C_I-S\ge V/64\mid\text{first-block history},J\}
 \ge\frac1{12288}.
 \label{eq:suffix-deficit-main}
\end{equation}
The chosen lengths satisfy \(h\ge W_q/512\), \(V\ge(255/256)W_q\),
and \(V/64\ge2h\).
On the events in \eqref{eq:prefix-deficit-main} and
\eqref{eq:suffix-deficit-main}, the hidden arm is ahead by at least
\(h/2\) in the zero-loss coordinate. In the other coordinate, the
algorithm can gain at most \(h\) during the first block, so the hidden arm
is ahead there by at least \(V/64-h\ge h\). Hence the same arm is ahead
in both coordinates and \(\Rfs\ge h/2\). The two events occur together
with probability at least \(1/24576\), yielding
\(
 \E[\Rfs]\ge h/49152=\Omega(W_q).
\)

When \(W_q<512\), a one-round construction covers the remaining range.
Choose a hidden arm uniformly, give it reward \((1,1)\) on the first
round, give every other arm \((0,0)\), and give every arm \((1,1)\) on
all later rounds. The expected regret is \(1-1/K\ge3/4\). Since
\(q=\lfloor K/2\rfloor\ge K/4\) for \(K\ge4\), we have
\(W_q\ge\frac12\min\{T,\sqrt{KT}\}\). In the small range this implies
\(\min\{T,\sqrt{KT}\}<1024\), so the constant lower bound has the
required order. In the large range the preceding display has that order.
For any fixed randomized algorithm, averaging selects one deterministic
reward sequence generated by the construction with at least this expected
regret. For \(D>2\), we copy coordinate 1 into the remaining coordinates;
this changes neither the feedback nor \(\Rfs\).

\begin{theorem}[Zero-loss lower bound under bandit feedback]
\label{thm:zero-lower}
For \(K\ge4,T\ge6,D\ge2\) and every bandit algorithm, there is a deterministic
reward sequence fixed before play with \(L^\star=0\) satisfying
\(
 \mathcal V_{K,D,T}(0) 
 =\Omega\!\left(\min\{T,\sqrt{KT}\}\right).
\)
\end{theorem}


Theorem~\ref{thm:zero-lower} directly answers the main question. For
fixed \(K\), expected regret can grow as \(\sqrt T\) even though an arm
has zero loss in one coordinate. Knowing the value \(L^\star=0\) does
not identify that coordinate. In contrast, if the zero-loss coordinate
itself is given before play, an elimination rule guarantees
\(\Rfs\le\min\{T,K-1\}\) (Lemma~\ref{lem:known-zero-coordinate}). When
\(D=1\), the coordinate is known automatically. The gap therefore comes
from not knowing which coordinate has zero loss, rather than from not
knowing the value of that loss.

For any exact class \(L^\star=L_0\), write \(U_0:=T-L_0\).
If \(U_0=0\), the lower bound is zero. If
\(0<U_0<1\), give one uniformly hidden arm reward
\(U_0\one\) on the first round and give all other arm-round pairs zero
reward. Then \(L^\star=L_0\), and the expected regret is
\((1-1/K)U_0\).

If \(U_0\ge1\), apply the zero-loss construction to the first
\(m:=\lfloor U_0\rfloor\) decisions; the algorithm may still use
\(T\) and \(L_0\). When \(m<T\), give every arm \((U_0-m)\one\) on
round \(m+1\) and zero thereafter. The largest arm-coordinate reward
is exactly \(U_0\), so \(L^\star=L_0\); the common added rewards leave
\(\Rfs\) unchanged.
Since \(m\ge U_0/2\), we have
\(\min\{m,\sqrt{Km}\}\ge\frac12\min\{U_0,\sqrt{KU_0}\}\).
This gives the following full-horizon bound.

\begin{theorem}[Lower bound for every value of \(L^\star\)]
\label{thm:conditional-lower}
For \(K\ge4,T\ge6,D\ge2\) and every \(L_0\in[0,T]\), with
\(U_0=T-L_0\), the value in \eqref{eq:conditional-value} satisfies
\(
 \mathcal V_{K,D,T}(L_0)
 =\Omega\!\left(\min\{U_0,\sqrt{KU_0}\}\right).
\)
The algorithm may know \(L_0\) before play.
\end{theorem}

At \(L_0=0\), we have \(U_0=T\), and
Theorem~\ref{thm:conditional-lower} recovers the rate in
Theorem~\ref{thm:zero-lower}. For \(U_0\le K\), the lower bound is
\(\Omega(U_0)\), matching the general upper bound \(\Rfs\le U_0\). For
\(U_0\ge K\), it is \(\Omega(\sqrt{KU_0})\), which comes from the cost of
finding the hidden arm. As \(L_0\) increases, \(U_0=T-L_0\) decreases,
so no comparator can accumulate as much reward. 


\begin{corollary}[Permutation]
\label{cor:profile-lower}
Let \(K\ge4,T\ge6,D\ge2\) and \(L_0\in[0,T]\). Set
\(U_0=T-L_0\) and
\(
 \lambda(L_0):=
 (L_0,\{\frac{T+L_0}{2},\ldots,
                     \frac{T+L_0}{2}\}_{D-1\text{ entries}}).
\)
Then we have 
\(
 \mathcal V^{\mathrm{perm}}_{K,D,T}(r)
 \ge 2^{-28}\min\{U_0,\sqrt{KU_0}\}.
\)
\end{corollary}



\paragraph{Relation to existing lower bounds.} \citet[Theorems~5.3--5.4]{xu2023pareto} obtain adversarial lower bounds of order \(\sqrt{T}\) by reducing to scalar bandits with identical coordinate means or rewards; their stochastic Theorem~5.1 similarly transfers a logarithmic scalar lower bound to Pareto pseudo regret. In contrast, Theorem~\ref{thm:conditional-lower} gives adversarial lower bounds over reward sequences with \(L^\star=L_0\), and Corollary~\ref{cor:profile-lower} further fixes all coordinate losses up to permutation. This requires hiding which coordinate attains the minimum, but in the zero-loss case, simply duplicating a scalar zero-loss instance across coordinates would allow the learner to use any coordinate and obtain bounded regret by Lemma~\ref{lem:known-zero-coordinate}. \citet[Theorem~1]{lai1985asymptotically} give a KL-dependent asymptotic lower bound for each fixed stochastic instance, under policies whose regret is subpolynomial over the model class. Our lower bound instead holds at a finite horizon and for every policy.

\section{Algorithms and Upper Bound Analyses}
\label{sec:alg}\label{sec:ub}

The lower bound analysis shows that any guarantee on Pareto regret that holds uniformly over the class \(L^\star=L_0\) must incur the rate
\(
\min\{T-L_0,\sqrt{K(T-L_0)}\}.
\) 
We now turn to the second part of our question: whether this rate can be attained without identifying \(d^\star\), the coordinate attaining \(L_0\). The encouraging answer is yes. It is enough to choose any coordinate before the first round and carefully optimize along it. We introduce two algorithms, both of which attain the same rate. When \(L_0\) is known, the
condition \(L^\star=L_0\) provides the common reward bound
\(U_0=T-L_0\) across coordinates, and we run Poly-INF on an arbitrary
coordinate using this bound. When \(L_0\) is unknown, we develop a
reward-doubling version of Poly-INF that adapts the reward scale from
observations.

The intuition behind these results follows from an interplay between the rate identified by the lower bound and Pareto geometry. The rate from the lower bound is monotonically decreasing in \(L_0\), while \(L_0\) also lower bounds the cumulative loss on every coordinate. This contrasts with the intuition from classical adversarial bandits that smaller loss is better. However, the rate increases with the complementary reward quantity \(U_0=T-L_0\). For every coordinate \(d\), we have \(U_d\le U_0\). Importantly, some algorithms for classical adversarial bandits admit regret bounds that increase with \(U_d\). Thus, any coordinate regret can also scale with \(U_0\). The following lemma crucially builds the route from coordinate regret to Pareto regret: the positive part of the regret on any fixed coordinate controls Pareto regret. The key subsequent task  is therefore to control this positive part, rather than only the expected regret commonly analyzed before, thereby transferring the same rate to Pareto regret.  
\begin{lemma}[Reduction]
\label{lem:fixed-coordinate-comparison}\label{lem:scalarization}
For every reward sequence and play path, and every coordinate \(d'\),
\begin{equation}
 0\le\Rfs\le\pos{R_T^{d'}},\qquad \Rfs\le U.
 \label{eq:fixed-coordinate}
\end{equation}
More generally, let \(\Delta_D\) be the set of nonnegative coordinate
weights that sum to one. For any \(w\in\Delta_D\), put
\(g_{i,t}=\sum_dw_dr_d^{i,t}\) and
\(R_T^w=\max_i\sum_tg_{i,t}-\sum_tg_{a_t,t}\). Then
\begin{equation}
 \Rfs\le\pos{R_T^w},
 \qquad
 \max_i\sum_tg_{i,t}\le U.
 \label{eq:scalarization}
\end{equation}
\end{lemma}

We use the coordinate case in what follows, while, more generally, Lemma~\ref{lem:scalarization}
shows that the same comparison extends to any fixed weighted average of the
coordinates. One subtlety is that the reduction involves
\(\pos{R_T^{d'}}\): a bound on signed expected regret alone is insufficient,
as positive and negative regret can cancel. The weighted version may have
implications for scenarios where the decision maker has known a priori
preferences, in the form of weights, over the coordinates.

\begin{remark}
The coordinate regret \(R_T^d\) is the standard regret against the best fixed arm in hindsight and can be negative, since the learner may switch among arms and outperform every fixed arm on coordinate \(d\). The sign becomes relevant here because Pareto regret is nonnegative; consequently, our comparison involves the positive part \([R_T^d]_+\), rather than \(R_T^d\) itself. In particular, a bound on \(\mathbb{E}[R_T^d]\) does not by itself imply the same bound on \(\mathbb{E}[[R_T^d]_+]\). Moreover, the weighted version generalizes the result in \citet{xu2023pareto}, which considers only the coordinate case without weights.
\end{remark}

These results reduce the remaining task to developing scalar methods whose positive part of the regret on an arbitrary fixed coordinate can be controlled at the target rate in \(U_0/L_0\). We first consider the case where \(L_0\) is known (Section \ref{sec:known-l0}) and then the case where it is unknown (Section \ref{sec:unknown-l0}). 

\subsection{When \texorpdfstring{\(L_0\)}{L0} is known}
\label{sec:known-l0}

The lower bound holds even if \(L_0\) is known, which motivates us to first consider this setting. We use the Poly-INF policy of \citet[Theorem~18]{audibert2010regret},
whose parameters can be set using an upper bound \(B\ge81K\) on cumulative
reward. We denote the resulting policy by \(\polyinf(B)\). Since \(U_0\) is
known and the cumulative reward of every arm on every coordinate is at most
\(U_0\), we can set \(B=U_0\) for any coordinate fixed before the first round.

Algorithm~\ref{alg:known-scale} therefore runs \(\polyinf(U_0)\) when
\(U_0\ge81K\). When \(U_0<81K\), it instead plays a fixed arm throughout,
as the deterministic bound \(\Rfs\le U_0\) already gives the desired rate. To see this, for every \(0\le U_0<81K\),
\(\Rfs\le U_0\le9\min\{U_0,\sqrt{KU_0}\}.
\)
If \(U_0\le K\), the minimum equals \(U_0\). If
\(K<U_0<81K\), then \(U_0/\sqrt{KU_0}=\sqrt{U_0/K}<9\).

Perhaps naturally, it leads to the following upper bound on the Pareto regret. 

\begin{theorem}[Optimal upper bound when \(L_0\) is known]
\label{thm:conditional-upper}
For \(K\ge2,T\ge2,D\ge1\) and every \(L_0\in[0,T]\), with
\(U_0=T-L_0\), Algorithm~\ref{alg:known-scale} with \(B=U_0\) satisfies
\(
\E[\Rfs]\le10\min\{U_0,\sqrt{KU_0}\}
\)
on every reward sequence with \(L^\star=L_0\).
\end{theorem}


For \(K\ge4,T\ge6,D\ge2\), this bound holds uniformly over all reward sequences \(r\) satisfying \(L^\star(r)=L_0\), for each fixed known \(L_0\). Taking the supremum over this class therefore yields the same order, which matches the lower bound in Theorem~\ref{thm:conditional-lower} for \(\mathcal V_{K,D,T}(L_0)\). Hence, a policy using any coordinate fixed in advance is minimax optimal without identifying the coordinate attaining \(L_0\). 

A new aspect of our analysis is controlling the positive part of scalar regret. For the Poly-INF policy of \citet[Theorem~18]{audibert2010regret}, the parameters are the same for every confidence level. Integrating its tail bound therefore gives, for \(U_0\ge81K\), \(\E[[R_T^{d'}]_+]\le(8.5+\sqrt2)\sqrt{KU_0}\). Together with the reduction from Pareto regret to coordinate regret, this yields the upper bound.


Notably, our choice of Poly-INF is deliberate. To see why, we also examine several popular algorithms. In particular, modified EXP3 uses the logarithmic estimator of
\citet[Theorem~21]{audibert2010regret} with \(B=U_0\). Gain-based
EXP3 and EXP3.P retain their updates, with their exploration parameters set using
\(U_0\). EXP3-IX similarly uses a learning rate tuned to \(U_0\). We find that they may incur additional factors in their regret bounds or retain a dependence on the time horizon \(T\) as shown in
Lemma~\ref{lem:known-l0-comparisons}, which can be larger than \(U_0\) since \(U_0\le T\). Thus, \emph{while the choice of coordinate does not matter for attaining the minimax rate, how the learner optimizes along that coordinate does.}  

The precise results are summarized below.

\begin{proposition}[EXP3 variants]
\label{prop:known-l0-comparisons}\label{prop:exp3-log-rewards}
Fix \(K\ge2,T\ge2,D\ge1\) and a coordinate \(d'\) before play. The corresponding bounds on \(\E[\Rfs]\) are
\(O(\min\{U_0,\sqrt{KU_0\log(3K)}\})\) for modified EXP3,
\(O(\min\{U_0,(K\log K)^{1/3}U_0^{2/3}\})\) for gain-based EXP3,
\(O(\min\{U_0,\sqrt{KT\log(KT)}\})\) for EXP3.P and adversarial MO-KS,
and \(O(\min\{U_0,\sqrt{KT\log K}\})\) for EXP3-IX. 
\end{proposition}

At the level of these guarantees, Poly-INF is the only algorithm among those considered that matches the lower bound uniformly over \(U_0\). Modified EXP3 incurs an additional \(\sqrt{\log K}\) factor, gain-based EXP3 does not attain the target dependence on \(U_0\), and EXP3.P and EXP3-IX retain a dependence on \(T\). This comparison shows that attaining the minimax rate does not require
identifying a favorable coordinate, but does require a favorable algorithm with
the appropriate dependence on \(U_0\). 

\subsection{When \texorpdfstring{\(L_0\)}{L0} is unknown}
\label{sec:unknown-l0}

More generally, \(L_0\) may be unknown to the learner, for example when the problem setting is completely black-box. In this case, a direct fallback is to run Poly-INF with \(B=T\), the worst-case upper bound on \(U_0\), though this can clearly be overly pessimistic. This gives
\(O(\min\{U_0,\sqrt{KT}\})\), where the second term fails to reflect the role of \(U_0\). 
To recover this dependence without knowing \(L_0\), we propose
\textbf{reward-doubling Poly-INF,
Algorithm~\ref{alg:adaptive}.} It fixes one coordinate, starts with
\(B=81K\), and runs \(\polyinf(B)\). Whenever the reward collected
during the current run reaches \(B\), it doubles \(B\) and starts a
fresh copy on the next round. In this way, the scale adapts to the observed rewards while the chosen coordinate remains fixed.

We prove its regret upper bound as in the following theorem. 

\begin{theorem}[Upper bound on Pareto regret]
\label{thm:adaptive}
For \(K\ge2,T\ge2,D\ge1\), Algorithm~\ref{alg:adaptive} does not use
\(L_0\). For every \(L_0\in[0,T]\) and every reward
sequence fixed before play with \(L^\star=L_0\), it satisfies
\refstepcounter{equation}\label{eq:adaptive-expected}%
\(\E[\Rfs]\le100\min\{U_0,\sqrt{KU_0}\},\qquad U_0=T-L_0\)~\textup{(\theequation)}.
\end{theorem}

The doubling scheme allows the unknown reward scale to reveal itself through the observed rewards. Across epochs, we compare against the best arm on the chosen coordinate, so the epoch regrets sum to the full coordinate regret. Whenever this regret is positive, the largest scale reached by the algorithm is controlled by the cumulative reward \(U_d\) of that arm. Since the epoch scales grow geometrically, the regret contributions across epochs remain of order \(\sqrt{K U_d}\). A single supermartingale controls the estimation error across all epochs, avoiding a union bound over restarts and hence an additional logarithmic factor. Together with \(U_d\le U_0\) and Lemma~\ref{lem:fixed-coordinate-comparison}, this gives the stated Pareto-regret bound.  

An important implication is that the optimal dependence on \(L_0\) can be attained even when neither its value nor the coordinate attaining it is known. The learner adapts through the rewards observed on an arbitrary fixed coordinate, while \(U_d\le U_0=T-L_0\) connects this adaptation to the target rate. Knowing \(L_0\), however, reduces the constant in our upper bound from \(100\) to \(10\).


Another natural approach is to apply a scalar algorithm (that is an algorithm in classical bandits along one coordinate) with a small-loss guarantee to a fixed coordinate \(d\), yielding a bound in terms of \(L_d\). Such a bound can be sharper for a particular reward sequence, but \(L_d\) can vary widely among sequences with the same \(L_0=\min_d L_d\). To see whether they can nevertheless recover the optimal rate, we closely examine several algorithms for adversarial bandits that do not know \(L_0\) and derive their bounds on Pareto regret. Their available guarantees reveal two obstacles. For several EXP3 variants, the resulting bounds retain a dependence on \(T\), \(L_d\), or additional factors, and therefore do not recover the optimal dependence on \(U_0\). For Tsallis-INF, EXP3++, and MO-US, the available guarantees instead control signed expected scalar regret, whereas our reduction requires its positive part. Converting these guarantees introduces an additional sequence-dependent term that is not controlled by \(L_0\). Fixed-block variants remove this term but lead to a \(T^{3/4}\) dependence. We summarize the results in Table~\ref{tab:algorithm-guarantees}; the precise statements are provided in Appendix. 


\begin{table}[h]
\centering
\small
\setlength{\tabcolsep}{4pt}
\renewcommand{\arraystretch}{1.12}
\begin{tabular}{@{}p{0.36\linewidth}p{0.12\linewidth}p{0.45\linewidth}@{}}
\toprule
Scalar method & Knows $L_0$? & $b$ in $\E[\Rfs]\lesssim\min\{U_0,b\}$\\
\midrule
Poly-INF & Yes & $\sqrt{KU_0}$\\
Modified EXP3 & Yes & $\sqrt{KU_0\log(3K)}$\\
Gain-based EXP3 & Yes / no & $(K\log K)^{1/3}V^{2/3}$ with $V=U_0$ / $T$\\
\midrule
\textbf{Reward-doubling Poly-INF (Ours)} & \textbf{No} & $\mathbf{\sqrt{KU_0}}$\\
Poly-INF with $B=T$ & No & $\sqrt{KT}$\\
EXP3.P and adversarial MO-KS & Either & $\sqrt{KT\log(KT)}$\\
EXP3-IX & Either & $\sqrt{KT\log K}$\\
GREEN-IX & No & $\sqrt{KL_{d'}\log(KT)}+K\log(KT)$\\
Original EXP3.1 & No & $\sqrt{KU_0\log K}+(KU_0^3/\log K)^{1/4}+K\log K$\\
Tsallis-INF (IW) & No & $\sqrt{KT}+S_{d'}$\\
Tsallis-INF (RV) & No & $\sqrt{KT}+K\log T+S_{d'}$\\
EXP3++ & No & $\sqrt{KT\log K}+S_{d'}$\\
MO-US & No & $\sqrt{KT\log K}+K+S_{d'}$\\
\midrule
Fixed-block Tsallis-INF (IW) & No & $K^{1/4}T^{3/4}$\\
Fixed-block Tsallis-INF (RV) and \mbox{MO-US} & No & $K^{1/4}T^{3/4}\log(2T)$\\
Fixed-block EXP3++ & No & $K^{1/4}T^{3/4}\sqrt{\log K}$\\
\bottomrule
\end{tabular}
\caption{Proved upper bounds for a coordinate $d'$ fixed before play.}
\label{tab:algorithm-guarantees}
\end{table}

Collectively, the lower and upper bounds yield the following minimax characterization. 
\begin{theorem}[Minimax rate]
\label{thm:minimax}
For \(K\ge4,T\ge6,D\ge2\) and every \(L_0\in[0,T]\), Eq. \eqref{eq:conditional-value} satisfies
\refstepcounter{equation}\label{eq:conditional-curve}%
\(\mathcal V_{K,D,T}(L_0)=\Theta(\min\{U_0,\sqrt{KU_0}\})\)~\textup{(\theequation)}.
\end{theorem}

At \(L_0=0\), the minimax rate is
\(\Theta(\WKT)\) where \(\WKT:=\min\{T,\sqrt{KT}\}\), recovering the existing bound and removing the extra $\sqrt{\log{T}}$ factor in \citet{xu2023pareto}.  More generally, these bounds characterize the dependence of the minimax Pareto regret on \(K\), \(T\), and, for the first time, \(L_0\), an intrinsic coordinate-level quantity that may or may not be known to the learner. The horizon enters only through \(U_0=T-L_0\). When \(U_0\le K\), the rate is \(\Theta(U_0)\), corresponding to the deterministic reward cap. When \(U_0\ge K\), the rate becomes \(\Theta(\sqrt{KU_0})\), reflecting the cost of learning the best arm. For \(K\in\{2,3\}\), all upper bounds remain valid, while the matching lower bound remains open. Lastly, the role of the coordinate dimension is surprisingly limited. The rate is independent of \(D\), and a single coordinate fixed before the first round is sufficient to attain the upper bound. 

\paragraph{Extension to the full-information setting.} The same dependence on \(L_0\) also appears under full information, where the learner observes every arm's reward vector after each round. For \(K\ge4\), \(T\ge6\), \(D\ge2\), and every \(L_0\in[0,T]\), the minimax regret over reward sequences with \(L^\star=L_0\) is \(\Theta(\min\{U_0,\sqrt{U_0\log K}\})\), where \(U_0=T-L_0\). Full information replaces the \(\sqrt{K}\) dependence by \(\sqrt{\log K}\), while the dependence on \(U_0\) remains. In particular, when \(L_0=0\), the regret still grows as \(\sqrt{T}\) for fixed \(K\). Appendix~\ref{app:full-info} gives the construction and matching upper bound.

\section{Conclusions and Future Work}\label{sec:conclusion}

Multi-objective bandits have been widely used or hold the potential to be widely used across various application domains. Understanding the role of coordinates in adversarial multi-objective bandits is particularly important, yet underexplored, for optimizing performance with the right handle. Pareto dominance seems to complicate the problem, but we have also been fortunately equipped with numerous methods from classical single-objective bandits. We take steps toward this herein by studying several aspects of coordinates in the context of Pareto regret optimization. Several findings are perhaps counterintuitive or surprising at first sight, but later turn out to be quite interpretable. The most crucial role of a coordinate-level quantity is reflected in the minimum coordinate loss \(L_0\). It plays an important role in coordinate regret and, in theory, governs the best admissible coordinate regret. A surprising lower bound on Pareto regret is monotonically decreasing in \(L_0\), yielding an instance-dependency through \(L_0\). We also develop algorithms that achieve this rate uniformly over the class, proving both tightness and optimality for every \(L_0\) and over instances in the minimax sense. In contrast, the roles of the coordinate dimension (\(D\)) and coordinate choice (\(d\)) are quite limited. The bound has no dependence on \(D\) or \(d\). The consequence is that the choice of algorithm along the chosen coordinate matters and must therefore be made carefully. 

At the same time, we feel obligated to highlight that the findings are limited to these specific intrinsic quantities and do not exclude the possibility that there may be other quantities that are also related to coordinates and affect Pareto regret. We hope the work proposed herein could inspire future work toward fully understanding the role of coordinates for performances in multi-objective problems.

\FloatBarrier

\begingroup
\small
\bibliography{main}
\bibliographystyle{plainnat}
\endgroup

\clearpage
\appendix
\raggedbottom
\section{Technical appendices and supplementary material}
\label{app:overview}

The appendix is organized as follows.
\begin{enumerate}
  \item Appendix~\ref{app:related-work} reviews the closest stochastic and
  adversarial bandit results.
  \item Appendix~\ref{app:criterion} proves the identities and scalar
  comparisons in Section~\ref{sec:pre}.
  \item Appendix~\ref{app:bandit-lower} proves the bandit lower bounds in
  Section~\ref{sec:lb}, the known-coordinate result cited there, and the
  full-information extension.
  \item Appendices~\ref{app:algorithms}--\ref{app:minimax-proof} give the
  algorithms and upper-bound proofs in the order used in
  Section~\ref{sec:alg}.
\end{enumerate}

\subsection{Related work}
\label{app:related-work}

\subsubsection{Stochastic Pareto regret}

Most work on stochastic multi-objective bandits assumes that each arm has a
fixed probability distribution over reward vectors and compares arms through
their mean reward vectors. \citet{drugan2013designing} introduced Pareto UCB1,
which forms coordinatewise confidence bounds, constructs an optimistic Pareto
set, and selects an arm from that set. Its cumulative Pareto regret adds the
Pareto suboptimality gap of the selected arm on each round. This is a
round-by-round comparison with the Pareto front of the mean reward vectors.

Several extensions retain this stochastic, mean-based viewpoint.
\citet{turgay2018contextual} define contextual Pareto regret as the sum of the
distances from the selected context-arm pairs to the corresponding
context-dependent Pareto fronts. \citet{lu2019multi} study generalized linear
reward models, use confidence bounds to approximate the Pareto front, and
obtain a contextual Pareto-regret guarantee. \citet{guan2026stochastic} ask
whether stochastic Pareto regret can follow an easy objective. Under a
condition ensuring that an objective-wise maximizer is Pareto optimal, their
algorithm searches across objectives and certifies a best arm on an objective
with a large mean gap. Repeated samples from fixed reward distributions make
such certification possible.

The stochastic regret above is different from the criterion studied here. It
sums gaps defined from fixed mean reward vectors or, in contextual models,
fixed conditional mean functions. Our criterion instead makes
one comparison after \(T\) rounds between the learner's cumulative reward
vector and the cumulative reward vectors of the fixed arms. This distinction
matters under adversarial rewards because there are no fixed arm means or mean
gaps to estimate.

\subsubsection{Other multi-objective criteria}

Pareto order is one way to compare vector rewards without first choosing how
to trade one objective against another. Other work makes that tradeoff part of
the problem. \citet{busafekete2017gini} optimize the Generalized Gini Index,
which aggregates the coordinates according to a fairness-sensitive utility.
\citet{huyuk2021multi} study lexicographically ordered and satisficing
objectives. Linear and nonlinear scalarizations likewise convert each reward
vector into a scalar reward before applying a bandit algorithm
\citep{drugan2013designing}. These criteria are appropriate when the
coordinate preferences are specified, but two different preferences can rank
the same Pareto-efficient arms differently.

Multi-objective optimization also evaluates sets of solutions through
additive indicators and hypervolume. Additive indicators ask how far one set
must be shifted before it weakly covers another, while hypervolume measures
the region dominated relative to a chosen reference point
\citep{zitzler2003performance,zhang2020hypervolume}. Our criterion uses a
common additive shift, but it answers a narrower online-learning question. It
asks whether one fixed arm improves the learner's cumulative reward in every
coordinate. It does not evaluate the coverage or hypervolume of a set of
solutions.

Another branch studies Pareto-set identification rather than cumulative
regret. \citet{auer2016pareto} give confidence-based procedures for
approximately identifying the Pareto front. \citet{ararat2023vector} extend
this problem to preferences represented by a polyhedral ordering cone and
derive cone-dependent sample-complexity bounds. More recent work studies both
fixed-budget identification \citep{kone2024bandit} and fixed-confidence
identification with asymptotically optimal sample complexity for Gaussian arms
\citep{crepon2024sequential}. These methods seek to return the Pareto set, or
an approximation to it, after sampling. Our learner instead
incurs regret while acting and is compared with fixed arms over the given
horizon.

\subsubsection{Adversarial multi-objective bandits and online learning}

\citet{jiang2023multiobjective} introduce a sequence-wise Pareto
suboptimality gap in multi-objective online convex optimization. Their
sequence-wise perspective provides conceptual motivation for evaluating
the learner's cumulative vector against fixed decisions in hindsight.
Our setting uses finitely many arms and observes only the chosen arm's
reward vector. We study an explicitly nonnegative fixed-arm comparison;
its pathwise representation and bandit guarantees are established here,
rather than inferred from online-convex-optimization guarantees.

The closest bandit work is \citet{xu2023pareto}. They define Pareto regret
through an equal-shift comparison between the learner's cumulative reward
vector and an arm-based Pareto front, and relate it to scalar regret on a
coordinate chosen before play. Their MO-KS algorithm runs UCB or EXP3.P
on one fixed coordinate when it is told whether rewards are stochastic or
adversarial. Their MO-US algorithm modifies EXP3++ to avoid requiring
this information. They also give worst-case lower bounds, including an
adversarial lower bound based on repeated coordinates.

Our criterion admits a uniform-shift representation using distinct
nondominated cumulative arm vectors and an infimum. Proposition~\ref{prop:xu-relation}
relates it precisely to the incomparability-based criterion, including
cases involving ties or an unattained minimum. For the nonnegative
criterion studied here, scalar comparisons require the positive part of
coordinate regret, not just its signed expectation.

Our lower bounds concern reward sequences with prescribed minimum
coordinate loss $L^\star=L_0$. The worst-case order persists at $L_0=0$,
even though some arm has zero loss in one coordinate, and the complete
rate depends on $K$ and $U_0=T-L_0$.
Corollary~\ref{cor:profile-lower} further preserves this rate for an
explicit family of prescribed unordered coordinate-loss profiles.
These results distinguish knowledge of coordinate-loss values from
knowledge of which coordinates attain them.
\par

\subsubsection{Scalar adversarial bandits}

The standard adversarial bandit benchmark is regret against the best fixed arm
in hindsight. EXP3 and EXP3.P give expected and high-probability guarantees at
the usual worst-case scale \citep{auer2002nonstochastic}. A separate line of
work obtains first-order bounds that become smaller when the best arm's
cumulative loss is small
\citep{allenberg2006hannan,neu2015first,lykouris2018small,wei2018more}.
Lower bounds match this dependence up to logarithmic factors in the
applicable small-loss regime \citep{gerchinovitz2016refined}. These results motivate our
question of whether a small loss in one unknown coordinate also reduces
Pareto regret.

The upper bounds in this paper also use scalar guarantees that depend on the
best fixed arm's cumulative reward. Poly-INF provides the bound used for our
matching result \citep{audibert2010regret}. We additionally analyze EXP3.P,
EXP3-IX \citep{neu2015explore}, Tsallis-INF
\citep{zimmert2021tsallis}, EXP3++ \citep{seldin2017improved}, and the scalar
algorithms underlying MO-KS and MO-US. This comparison requires some care.
A bound on signed expected scalar regret does not by itself bound the expected
positive part of scalar regret, which controls our Pareto regret. We therefore
state whether each result uses the published algorithm unchanged, uses
parameters determined by \(U_0\), or changes the reward estimator. The modified EXP3
estimator and the Poly-INF policy are existing scalar constructions; our
contribution is their analysis under the Pareto criterion and the matching
reward-doubling procedure when \(L_0\) is unknown.

With full-information feedback, the scalar comparison becomes the classical
prediction-with-expert-advice problem, whose worst-case arm dependence is
\(\sqrt{\log K}\) rather than \(\sqrt K\)
\citep{cesabianchi2006prediction}. Our full-information result shows that this
stronger feedback changes the dependence on the number of arms but does not
make an unidentified zero-loss coordinate sufficient for bounded regret.

\FloatBarrier
\subsection{Problem-formulation properties}
\label{app:criterion}

We prove the formula defining \(\Rfs\), its relation to the original
distance, and the comparisons with scalar regret, in that order.

\subsubsection{The formula for Pareto regret}

\begin{proof}[Proof of Proposition~\ref{prop:pareto-representation}]
For a fixed cumulative reward vector \(\sigma\), put
\(m_\sigma=\min_d(\sigma_d-A_{T,d})\), and let
\(M=\max_{\sigma\in O'}m_\sigma\), which exists because \(O'\) is finite
and nonempty. If \(0\le\epsilon<M\), choose \(\sigma\in O'\) with
\(m_\sigma=M\). Then
\(\sigma_d>A_{T,d}+\epsilon\) for every \(d\), so this value of
\(\epsilon\) is not feasible in the infimum. If
\(\epsilon>\max\{0,M\}\), then for every \(\sigma\in O'\) there is a
coordinate \(d\) with
\(\sigma_d-A_{T,d}\le M<\epsilon\). Hence \(\sigma\) does not weakly
dominate \(A_T+\epsilon\one\), and \(\epsilon\) is feasible. These two
implications show that the infimum is \(\max\{0,M\}=\pos M\); whether the
boundary itself is feasible does not change its value.

It remains to replace the maximum over \(O'\) by a maximum over arms. For
every arm \(i\), either \(G_i\in O'\), or repeatedly following a dominating
arm reaches some \(\sigma\in O'\) with \(\sigma\succeq G_i\), because there
are only finitely many distinct arm vectors. Therefore
\[
 \min_d(G_{i,d}-A_{T,d})
 \le \min_d(\sigma_d-A_{T,d}).
\]
The reverse inequality between the two maxima holds because every vector in
\(O'\) is an arm vector. Thus
\(M=\max_i\min_d(G_{i,d}-A_{T,d})\), proving the identity on every
outcome, including those for which the value is zero.
\end{proof}

\subsubsection{Relation to the original distance}
\label{app:xu-relation}

%

We compare \eqref{eq:pareto-regret} with the uniform-shift
incomparability criterion of \citet{xu2023pareto}. Write their front as
\[
 F_X:=\{G_i:i\in[K],\ \nexists j\ne i\text{ with }G_j\succeq G_i\}.
\]
This definition excludes an arm whenever another arm weakly dominates
it, including when the two arms have the same cumulative reward vector.
For example, two arms with cumulative vector $(1,1)$ are both excluded,
leaving $F_X$ empty. Our vector front $O'$ keeps one copy of this vector
and is always nonempty.

When $A_T\preceq\sigma$ for every $\sigma\in F_X$, their criterion uses
the smallest $\epsilon\ge0$ for which
$A_T+\epsilon\one\,\|\,\sigma$ for every $\sigma\in F_X$.
Strict incomparability can leave this minimum undefined. With
$A_T=(0,0)$ and $F_X=\{(1,1)\}$, no shift is feasible. Replacing the
front by $F_X=\{(1,2)\}$ gives the open interval $(1,2)$, which has
infimum one but no smallest element.
For nonempty $F_X$, define
\begin{equation}
 a_-:=\max_{\sigma\in F_X}\min_d(\sigma_d-A_{T,d}),
 \qquad
 a_+:=\min_{\sigma\in F_X}\max_d(\sigma_d-A_{T,d}).
 \label{eq:xu-endpoints}
\end{equation}

\begin{proposition}[Comparison]
\label{prop:xu-relation}
Suppose $F_X\ne\varnothing$ and $A_T\preceq\sigma$ for every
$\sigma\in F_X$. The feasible shifts in the incomparability condition
of \citet{xu2023pareto} form $(a_-,a_+)$. If this interval is nonempty,
its infimum is $a_-\le\Rfs$, and its minimum is not attained.
If the cumulative arm vectors are distinct, then $F_X=O'$ and the
infimum equals $\Rfs$. Proposition~\ref{prop:pareto-representation}
continues to apply when $F_X$ or the feasible interval is empty.
\end{proposition}

\begin{proof}[Proof of Proposition~\ref{prop:xu-relation}]
A vector $A_T+\epsilon\one$ is incomparable with $\sigma$ precisely
when at least one of its coordinates is strictly smaller and another is
strictly larger. Equivalently,
\[
 \min_d(\sigma_d-A_{T,d})<\epsilon
 <\max_d(\sigma_d-A_{T,d}).
\]
Intersecting these open intervals over $F_X$ gives $(a_-,a_+)$. Under
the proposition's assumptions, $a_-\ge0$, so requiring $\epsilon\ge0$
leaves this interval unchanged. If the interval is nonempty, its infimum
is $a_-$ and it has no minimum.

Every vector retained in $F_X$ belongs to $O'$, although a vector shared
by tied arms can be absent from $F_X$. Hence $F_X\subseteq O'$ and
\[
 a_-\le\max_{\sigma\in O'}\min_d(\sigma_d-A_{T,d}).
\]
Together with $a_-\ge0$ and
Proposition~\ref{prop:pareto-representation}, this proves $a_-\le\Rfs$.
If all cumulative arm vectors are distinct, both fronts contain the same
vectors, giving $a_-=\Rfs$.
\end{proof}
\par

\subsubsection{Scalar comparisons}

\begin{proof}[Proof of Lemma~\ref{lem:scalarization}]
For \(x_{i,d}=G_{i,d}-A_{T,d}\) and any \(w\in\Delta_D\),
\[
 \max_i\min_d x_{i,d}
 \le\max_i\sum_dw_dx_{i,d}=R_T^w.
\]
Taking positive parts proves the weighted comparison. Taking \(w\)
supported on coordinate \(d'\) gives the coordinate comparison.
Finally, all learner rewards are
nonnegative, and \(G_{i,d}\le U^\star\) for every \(i,d\). Hence
\(\Rfs\le U^\star\) and \(\max_i\sum_dw_dG_{i,d}\le U^\star\).

When a scalar algorithm is used later, its weights are fixed before play.
Its scalar observations then lie in \([0,1]\), use only the selected arm's
reward vector, and form a reward sequence fixed before play. The same proof
applies after conditioning on weights sampled independently before play.
\end{proof}

\FloatBarrier
\subsection{Lower-bound analysis}
\label{app:bandit-lower}

We first establish a scalar bandit estimate, then use a positive-probability
version of it to compare the algorithm with one arm in both coordinates.

\subsubsection{A scalar bandit estimate}

\begin{lemma}[Hard bandit suffix]
\label{lem:hard-suffix-app}
Let \(m\ge1\) and \(q\ge2\) be integers, and put
\[
 \alpha:=\frac18\min\!\left\{1,\sqrt{q/m}\right\},
\]
the advantage in mean reward that the hidden candidate will carry.
Consider \(m\) bandit-feedback rounds with \(q\) candidate actions and any
number of additional actions.  The causal learner may begin with an
arbitrary random state \(Z\), which models whatever it inherits from
earlier rounds.  Draw \(I\sim\mathrm{Unif}[q]\) independently of \(Z\).
For each \(i\in[q]\), let \(P_i\) denote the probability distribution of the
complete interaction conditional on \(I=i\). Let \(P_0\) denote the
corresponding null distribution, in which no candidate is advantaged. Assume
the following.
\begin{enumerate}
  \item The distribution of \(Z\) is the same under
  \(P_0,P_1,\ldots,P_q\).
  Under \(P_i\), candidate rewards are independent of \(Z\), mutually
  independent across candidates and rounds, candidate \(i\) has Bernoulli
  mean \(1/2+\alpha\), and every other candidate is fair, that is,
  Bernoulli with mean \(1/2\).  Under \(P_0\), every candidate is fair.
  \item Under each fixed distribution \(P_i\) or \(P_0\), if
  \(\mathcal F_{t-1}\) is the complete history before suffix round
  \(t\), including \(Z\), then the current candidate-reward vector is
  independent of \(\mathcal F_{t-1}\) and of the current action, with the
  distribution stated in item~1. Each additional action has conditional
  expected scalar reward at most \(1/2\), given \(\mathcal F_{t-1}\) and
  its selection.
  \item Feedback is nonanticipating and contains no unselected current or
  future candidate reward.  For each fixed \(i\), the conditional
  distributions of the observation given the history and selected action
  coincide under \(P_i\) and \(P_0\) when the selected action is not \(i\).
  When action \(i\) is selected, the observation contains its scalar
  Bernoulli reward and, conditional on that reward, all remaining feedback
  has the same conditional distribution under \(P_i\) and \(P_0\).
\end{enumerate}
Under these assumptions, a round on which \(i\) is selected contributes
exactly its Bernoulli likelihood ratio to the ratio between \(P_i\) and
\(P_0\), and the arbitrary inherited state contributes no likelihood ratio.

If \(S\) is the learner's cumulative suffix reward and \(C_i\) is candidate
\(i\)'s cumulative suffix reward, then every causal learner satisfies
\begin{equation}
 \E\!\left[\max_{i\in[q]}C_i-S\right]
 \ge\frac1{32}\min\{m,\sqrt{qm}\},
 \label{eq:hard-suffix-app}
\end{equation}
where the expectation is over the hidden arm, the Bernoulli rewards, the
remaining feedback, and the learner's randomization.
\end{lemma}

\begin{proof}
The proof compares every alternative \(P_i\) with the null distribution \(P_0\),
uses relative entropy and Pinsker's inequality, which bounds the total
variation distance between two probability distributions by the square root
of half their relative entropy, to limit how often the hidden arm can be
sampled, and charges \(\alpha\) expected reward on every missed pull.

Write \(N_i\) for the number of pulls of candidate \(i\) and
\(n_i=\E_0[N_i]\).  By the feedback-kernel assumption, item~3, a round
contributes nothing to the relative entropy between \(P_0\) and \(P_i\)
unless \(i\) is selected, and when \(i\) is selected it contributes the
divergence \(d_0\) between a fair coin and a coin with mean \(1/2+\alpha\).
The chain rule for relative entropy, which sums these conditional
contributions round by round, therefore gives
\begin{align}
 \KL(P_0\Vert P_i)
 &=n_i d_0,\nonumber\\
 d_0
 &:=\operatorname{kl}(1/2\Vert1/2+\alpha)
 =-\frac12\log(1-4\alpha^2)
 \le\frac{32}{15}\alpha^2.\nonumber
\end{align}
For the last inequality, put \(x=4\alpha^2\le1/16\) and use
\(-\log(1-x)\le x/(1-x)\), which gives
\(d_0\le2\alpha^2/(1-4\alpha^2)\le(32/15)\alpha^2\).

Since \(0\le N_i\le m\), replacing \(P_0\) by \(P_i\) changes the
expectation of \(N_i\) by at most \(m\) times their total variation distance,
and Pinsker's inequality bounds that distance by
\(\sqrt{\KL(P_0\Vert P_i)/2}\), so
\begin{align*}
 \E_i[N_i]
 &\le \E_0[N_i]+m\,\lVert P_i-P_0\rVert_{\mathrm{TV}}\\
 &\le n_i+m\sqrt{\frac{\KL(P_0\Vert P_i)}2}
 =n_i+m\sqrt{\frac{d_0n_i}{2}}.
\end{align*}
Because \(\sum_i n_i\le m\), averaging over \(i\) and applying
Cauchy--Schwarz in the form \(\sum_i\sqrt{n_i}\le\sqrt{q\sum_in_i}\) yield
\begin{align}
 \frac1q\sum_{i=1}^q\E_i[N_i]
 &\le\frac mq+\frac mq\sqrt{\frac{d_0}{2}}
       \sum_{i=1}^q\sqrt{n_i}\nonumber\\
 &\le\frac mq+m\sqrt{\frac{d_0m}{2q}}
 \le\frac mq+\frac{m}{2\sqrt{15}}.
 \label{eq:average-count}
\end{align}
For the last step, if \(q\le m\), then
\[
 \sqrt{\frac{d_0m}{2q}}
 \le\sqrt{\frac{16\alpha^2m}{15q}}
 =\frac1{2\sqrt{15}}.
\]
If \(q>m\), then \(\alpha=1/8\), and the same left-hand side is at most
\(\sqrt{m/(60q)}<1/(2\sqrt{15})\).

Under \(P_i\), a round on which \(i\) is not pulled gives the learner a
reward of conditional expectation at most \(1/2\), by items 1 and~2, while
the unobserved reward of candidate \(i\) is independent with mean
\(1/2+\alpha\).  On a round when \(i\) is pulled, the learner and the
comparator collect the same realized reward.  Consequently
\[
 \E_i[C_i-S]\ge\alpha(m-\E_i[N_i]).
\]
Averaging, using \eqref{eq:average-count}, and noting \(q\ge2\),
\begin{align}
 \frac1q\sum_i\E_i[C_i-S]
 &\ge\alpha m
 \left(1-\frac1q-\frac1{2\sqrt{15}}\right)\nonumber\\
 &\ge\frac1{32}\min\{m,\sqrt{qm}\}.\nonumber
\end{align}
Indeed, the parenthesized factor is at least
\(1/2-1/(2\sqrt{15})>1/4\), while
\(\alpha m=\min\{m,\sqrt{qm}\}/8\).
Finally \(\max_iC_i-S\ge C_I-S\) on every play path, and the left side of
the previous display is \(\E[C_I-S]\) because \(I\) is uniform.  This
proves \eqref{eq:hard-suffix-app}.
\end{proof}

\begin{lemma}[A positive scalar shortfall with fixed probability]
\label{lem:bandit-positive-tail}
In Lemma~\ref{lem:hard-suffix-app}, suppose that every additional action
has deterministic scalar reward \(1/2\). Put
\(V=\min\{m,\sqrt{qm}\}\) and \(Z=C_I-S\). Then, for every inherited
state independent of the suffix experiment,
\[
 \E[Z]\ge V/32,\qquad \E[Z^2]\le3V^2,\qquad
 \Prob\{Z\ge V/64\}\ge1/12288.
\]
\end{lemma}

\begin{proof}
The proof of Lemma~\ref{lem:hard-suffix-app} gives the stated expectation
for the hidden arm \(I\) itself. Let
\(X_t\) be candidate \(I\)'s reward minus the learner's reward on suffix
round \(t\). Conditional on \(I\) and the history before that round,
\(\E[X_t]=\alpha\one\{a_t\ne I\}\). Therefore
\[
 M:=\sum_{t=1}^m
 \bigl(X_t-\alpha\one\{a_t\ne I\}\bigr)
\]
is a martingale with \(M_0=0\). Each difference \(X_t\) lies in
\([-1,1]\), so each martingale increment has conditional variance at most
one. Martingale orthogonality then gives \(\E[M^2]\le m\), and
\[
 Z=\alpha(m-N_I)+M,\qquad \E[M^2]\le m.
\]
In both cases \(q\le m\) and \(q>m\), the definition of \(\alpha\) gives
\(\alpha m=V/8\). Also \(V^2\ge m\) because \(q\ge2\), and
\(0\le N_I\le m\). Hence
\[
 \E[Z^2]\le2\alpha^2m^2+2\E[M^2]
 \le V^2/32+2m\le3V^2.
\]
Writing \(p=\Prob\{Z\ge V/64\}\), split the positive part of \(Z\) at
\(V/64\) and apply Cauchy--Schwarz on the upper event. This gives
\[
 V/32\le\E[Z]\le\E[\pos{Z}]\le V/64+\sqrt{\E[Z^2],p}
 \le V/64+\sqrt3\,V\sqrt p.
\]
Subtracting \(V/64\), dividing by \(\sqrt3V\), and squaring proves
\(p\ge1/(64^2\cdot3)=1/12288\). All calculations hold for an
arbitrary inherited state, so they also hold conditional on any
prefix history independent of the suffix experiment.
\end{proof}

\subsubsection{Proof of Theorem~\ref{thm:zero-lower}}

\begin{proof}[Proof of Theorem~\ref{thm:zero-lower}]
We prove a slightly stronger statement covering every integer \(T\ge1\).
Let \(q=\lfloor K/2\rfloor\ge2\) and \(W=\min\{T,\sqrt{qT}\}\).

First suppose \(W\ge512\), and set
\[
 h=\lfloor W/256\rfloor,\qquad n=T-h,\qquad
 \alpha=\tfrac18\min\{1,\sqrt{q/n}\}.
\]
Divide \(2q\) arms into groups \(A_1,\ldots,A_q\) and
\(B_1,\ldots,B_q\). If \(K\) is odd, there is one additional arm \(E\).
For the first \(h\) rounds their rewards are respectively
\((1,0),(0,1),(0,0)\).

Before play, choose a case \(w\in\{A,B\}\) uniformly and an independent
index \(I\in[q]\) uniformly. Generate independent Bernoulli variables
\(Y_{i,t}\) for the remaining \(n\) rounds. Their means are \(1/2+\alpha\)
when \(i=I\) and \(1/2\) otherwise. The later rewards are
\[
\begin{array}{c|ccc}
 &A_i&B_i&E\\ \hline
 w=A &(1,Y_{i,t})&(0,1/2)&(0,1/2)\\
 w=B &(1/2,0)&(Y_{i,t},1)&(1/2,0).
\end{array}
\]
Omit column \(E\) when \(K=2q\). All rewards are chosen before the
algorithm acts. In case \(A\), every \(A_i\) receives reward one in
coordinate 1 throughout; in case \(B\), every \(B_i\) does so in
coordinate 2. Thus \(L^\star=0\) for every reward sequence in this
finite distribution.

The algorithm cannot distinguish the cases during the first \(h\)
rounds. Condition on its history through those rounds, and let
\(X_A,X_B,X_E\) be its group pull counts. They sum to \(h\).
Its initial deficits in the two possible zero-loss coordinates are
\[
 d_A=h-X_A,\qquad d_B=h-X_B,\qquad d_A+d_B=h+X_E\ge h.
\]
The case is still uniform and independent of this history. Therefore
\(d_w\ge h/2\) with conditional probability at least \(1/2\).

We may now reveal the case to the algorithm, which only makes the
problem easier. In case \(A\), use coordinate 2 as the scalar suffix
problem, with the \(A_i\) as candidates. In case \(B\), use coordinate
1, with the \(B_i\) as candidates. The hidden index is independent of
the inherited history. The other observed coordinate is deterministic
given the group and case, so supplies no extra information about that
index. Every noncandidate has scalar reward \(1/2\).
Lemma~\ref{lem:bandit-positive-tail} applies. If \(C_I\) is the hidden
arm's scalar reward and \(S\) the algorithm's scalar reward in these
\(n\) rounds, then, with \(V=\min\{n,\sqrt{qn}\}\),
\[
 \Prob\{C_I-S\ge V/64\mid\text{prefix history},w\}\ge1/12288.
\]

We next verify the four bounds on the chosen lengths. Since \(W\ge512\),
we have \(W/256\ge2\), and hence
\(h=\lfloor W/256\rfloor\ge W/512\). Also
\(h\le W/256\le T/256\), so \(n=T-h\ge255T/256\). Therefore
\[
 n\ge\frac{255}{256}T,\qquad
 \sqrt{qn}\ge\sqrt{\frac{255}{256}}\sqrt{qT}
 \ge\frac{255}{256}\sqrt{qT}.
\]
Taking the minimum of the two left-hand terms gives
\(V\ge(255/256)W\). Finally,
\[
 \frac V{64}\ge\frac{255W}{16384}
 \ge\frac{2W}{256}\ge2h.
\]

Conditional on every prefix history and the selected case, the suffix event
has probability at least \(1/12288\). The prefix event has probability at
least \(1/2\), so their intersection has probability at least
\(1/24576\). On this intersection,
the same hidden arm has cumulative advantage at least \(h/2\) in its
zero-loss coordinate. That advantage cannot decrease after the prefix,
because the arm then receives one on every round. In the other
coordinate, its prefix reward is zero and the algorithm's prefix reward
is at most \(h\), so its final advantage is at least
\(C_I-S-h\ge h\). Consequently \(\Rfs\ge h/2\) on this event, and
\[
 \E[\Rfs]\ge\frac{h}{49152}\ge\frac{W}{25165824}.
\]

For \(W<512\), use a simpler finite distribution. Choose an arm uniformly
before play, give it reward \((1,1)\) on round 1 and all other arms
\((0,0)\), and give every arm \((1,1)\) on all later rounds. The
easiest-coordinate loss is zero. The regret is one if the algorithm
misses the hidden arm on the first round, so its expectation is
\(1-1/K\ge3/4\).

For \(K\ge4\), \(q=\lfloor K/2\rfloor\ge K/4\), and therefore
\[
 W=\min\{T,\sqrt{qT}\}
 \ge\frac12\min\{T,\sqrt{KT}\}.
\]
The large-\(W\) calculation gives
\[
 \E[\Rfs]\ge\frac{W}{25165824}
 \ge\frac{\min\{T,\sqrt{KT}\}}{50331648}
 \ge2^{-26}\min\{T,\sqrt{KT}\}.
\]
The small-\(W\) construction also
satisfies this bound because \(W<512\) implies
\(\min\{T,\sqrt{KT}\}<1024\), while
\(3/4>2^{-26}\cdot1024\).
For \(D>2\), duplicate coordinate 1. This changes neither the regret
nor the information available under bandit feedback.

Finally fix any randomized algorithm, including one told \(L^\star=0\).
The average expected regret over our finite distribution has the stated
lower bound. At least one deterministic reward sequence in its support
therefore has at least that much expected regret, with expectation only
over the algorithm's randomness. The support consists entirely of
oblivious zero-loss sequences. This gives the claimed quantifiers.
The factor lost when the construction is shortened below changes this
constant from \(2^{-26}\) to \(2^{-27}\).
\end{proof}

\subsubsection{When the zero-loss coordinate is known}

The lower bound hides which coordinate has zero loss. The next result shows
why this uncertainty matters.

\begin{lemma}[A known zero-loss coordinate]
\label{lem:known-zero-coordinate}
Let \(K,T,D\ge1\). Suppose a coordinate \(d\) with \(L_d=0\) is
specified before play, but the arm attaining that loss is unknown. Start
with all arms available. On each round, play the first available arm in a
fixed ordering, and remove it if its observed reward in coordinate \(d\) is
less than one. Then
\[
 \Rfs\le\min\{T,K-1\}.
\]
When \(D=1\) and \(L^\star=0\), the same conclusion holds without any
additional information.
\end{lemma}

\begin{proof}
Because rewards lie in \([0,1]\), the condition \(L_d=0\) means that some
fixed arm receives reward one in coordinate \(d\) on every round. This arm
is never removed, so the algorithm always has an available arm. Whenever
the algorithm receives a reward smaller than one in coordinate \(d\), it
removes the arm it just played. No arm can be removed twice, and the
zero-loss arm is never removed. There are therefore at most \(K-1\) rounds
with positive loss in coordinate \(d\). Each such loss is at most one, so
the algorithm's cumulative loss in that coordinate is at most
\(\min\{T,K-1\}\).

The best fixed arm receives cumulative reward \(T\) in coordinate \(d\).
Consequently, \(R_T^d\) equals the algorithm's cumulative loss in that
coordinate. Lemma~\ref{lem:scalarization} now gives
\[
 \Rfs\le\pos{R_T^d}=R_T^d\le\min\{T,K-1\}.
\]
If \(D=1\), the only coordinate is known without being supplied separately.
\end{proof}

\subsubsection{Proof of Theorem~\ref{thm:conditional-lower}}
\label{app:conditional}

The preceding construction gives zero loss in one coordinate. We now
adjust its length and append common rewards to obtain every specified
value of \(L^\star\).

\begin{proof}[Proof of Theorem~\ref{thm:conditional-lower}]
Fix \(U_0=T-L_0\). If \(U_0=0\), the claimed lower bound is zero.
If \(0<U_0<1\), choose a hidden arm uniformly, give it reward
\(U_0\one\) on round 1, and give all other arm-round pairs reward zero.
Every sequence has maximum arm-coordinate cumulative reward \(U_0\).
The first action misses the hidden arm with probability \(1-1/K\),
giving expected regret \((1-1/K)U_0\ge3U_0/4\). Since \(U_0<K\),
\(\min\{U_0,\sqrt{KU_0}\}=U_0\), so this proves the required bound in
this case.

For \(U_0\ge1\), let \(m=\lfloor U_0\rfloor\ge1\). Use the zero-loss
construction from the proof of Theorem~\ref{thm:zero-lower} for \(m\)
rounds. If \(m<T\), give every arm in every coordinate the same reward
\(U_0-m\) on round \(m+1\), and zero on all later rounds. If \(m=T\),
no padding is needed, since then \(U_0=T\). The largest cumulative
arm-coordinate reward is exactly \(m+(U_0-m)=U_0\). The added rounds
contribute the same reward to every comparator and the algorithm, so
every coordinate advantage and \(\Rfs\) are unchanged.

To apply the shorter-horizon bound, simulate the first \(m\) actions of
any proposed \(T\)-round algorithm, with its supplied values \(T,L_0\)
and its arbitrary randomization. This is an admissible \(m\)-round
algorithm. The zero-loss proof covers every \(m\ge1\), not only the
horizon range stated in the main theorem. Since \(U_0\ge1\),
\(m=\lfloor U_0\rfloor\ge U_0/2\). Thus
\(\sqrt{Km}\ge\sqrt{KU_0/2}\ge\frac12\sqrt{KU_0}\), and
\[
 \min\{m,\sqrt{Km}\}\ge
 \tfrac12\min\{U_0,\sqrt{KU_0}\}.
\]
Multiplying the zero-loss constant \(2^{-26}\) by the factor \(1/2\)
in the preceding display gives
\[
 \E[\Rfs]\ge2^{-27}\min\{U_0,\sqrt{KU_0}\}.
\]
As before, averaging over a finite
distribution yields a deterministic sequence for each algorithm.
Knowing \(L_0\) reveals neither the hidden case nor the hidden arm.
\end{proof}

\subsubsection{Proof for a prescribed unordered loss profile}
\label{app:profile-lower}

\begin{proof}[Proof of Corollary~\ref{cor:profile-lower}]
Fix \(U_0=T-L_0\). If \(U_0=0\), the all-zero sequence has profile
\((T,\ldots,T)\) and regret zero, as required. Assume first
\(0<U_0<2\). Before play, choose an arm \(I\) uniformly and give it
reward \((U_0/2)\one\) on the first round, while every other arm receives
zero. Choose a coordinate \(d^\star\) independently. On every later
round, give every arm reward \(U_0/(2(T-1))\) in coordinate
\(d^\star\) and zero in every other coordinate. All rewards lie in
\([0,1]\). The best arm's cumulative reward is \(U_0\) in coordinate
\(d^\star\) and \(U_0/2\) in every other coordinate, so the unordered
loss profile is exactly \(\lambda(L_0)\). Every arm receives the same
rewards after the first round, so the final regret is \(U_0/2\) if the
learner does not choose \(I\) on the first round, and zero otherwise. Its expectation is therefore
\((1-1/K)U_0/2\ge3U_0/8\). Since \(U_0<2\le K\), this proves
the lower bound for this case.

Now suppose \(U_0\ge2\), and set
\(m=\lfloor U_0/2\rfloor\ge1\), \(N=T-m\).
Use the two-coordinate finite distribution in the proof of
Theorem~\ref{thm:zero-lower} for the first \(m\) rounds. That proof
covers every integer horizon \(m\ge1\) and yields expected regret at
least \(2^{-26}\min\{m,\sqrt{Km}\}\) against any policy. For
\(D>2\), set each additional reward coordinate to the average of the
first two during these rounds. The learner can compute these averages
from the first two coordinates, so the added coordinates give no new
information. For each fixed arm, its reward advantage over the learner
on an added coordinate is the average of its advantages on the first two
coordinates. This average is at least their minimum, so adding these
coordinates leaves the Pareto regret over the first \(m\) rounds unchanged.

Generate this entire prefix before play, and write
\[
 M_d:=\max_i\sum_{t=1}^{m}r_d^{i,t}.
\]
There is a coordinate \(d^\star\in\{1,2\}\) with
\(M_{d^\star}=m\). When both coordinates satisfy this equality, choose
coordinate 1. On each of the last \(N\) rounds, give every arm the same reward
in coordinate \(d\), equal to
\[
 b_d:=
 \begin{cases}
  (U_0-m)/N,&d=d^\star,\\
  (U_0/2-M_d)/N,&d\ne d^\star.
 \end{cases}
\]
The inequalities \(0\le M_d\le m\le U_0/2\),
\(U_0-m\le T-m=N\), and
\(U_0/2\le T-m=N\) imply \(0\le b_d\le1\) for all \(d\),
so the appended rewards are valid.
The final best-arm cumulative reward is exactly \(U_0\) on
\(d^\star\) and \(U_0/2\) elsewhere. Every realization therefore
belongs to \(\mathcal C(\lambda(L_0))\).

We choose the entire reward sequence before play, including the
rewards added after round \(m\), so the adversary is oblivious. In
coordinate \(d\), these later rounds add the same amount \(Nb_d\) to
the learner's reward and to every arm's reward. Each arm's advantage
therefore stays the same on every coordinate, and the final Pareto
regret equals the regret over the first \(m\) rounds.

Any \(T\)-round policy that knows the unordered profile can also be
run for just the first \(m\) rounds. Knowing \(T,L_0,\lambda(L_0)\)
gives the learner no information about the hidden case, the hidden arm,
or the generated rewards, because these values are fixed independently
of the construction. The zero-loss lower bound therefore applies to
these first \(m\) rounds. Since
\(m=\lfloor U_0/2\rfloor\ge U_0/4\),
\[
 \min\{m,\sqrt{Km}\}
 \ge\frac14\min\{U_0,\sqrt{KU_0}\}.
\]
The expected regret over the finite construction is at least
\(2^{-28}\min\{U_0,\sqrt{KU_0}\}\). For every randomized policy,
at least one deterministic sequence in this finite construction gives
at least this expected regret, with expectation over the policy's
randomness. This proves the minimax lower bound. The upper bound follows
from Theorem~\ref{thm:conditional-upper}, since the learner can compute
\(L_0=\min_d\lambda_d\) from the profile and every sequence in the
profile class has \(L^\star=L_0\).
\end{proof}

\subsubsection{Full-information extension}
\label{app:full-info}

Let \(\mathcal V^{\mathrm{fi}}_{K,D,T}(L_0)\) denote the value in
\eqref{eq:conditional-value} when the algorithm observes every arm's reward
vector after each round. As in \eqref{eq:conditional-value}, the algorithm may
know \(L_0\) before play.

\begin{theorem}[Full-information minimax value]
\label{thm:full-information-app}
For \(K\ge4\), \(T\ge6\), \(D\ge2\), and every \(L_0\in[0,T]\), put
\(U_0=T-L_0\). Then
\[
 \mathcal V^{\mathrm{fi}}_{K,D,T}(L_0)
 =\Theta\!\left(\min\{U_0,\sqrt{U_0\log K}\}\right).
\]
\end{theorem}

We prove the zero-loss lower bound first and then obtain every value of
\(L_0\) by shortening and padding the construction. The matching upper bound
uses exponential weights on one coordinate fixed before play.

\begin{lemma}[Zero-loss lower bound with full information]
\label{lem:full-information-zero-lower}
For \(K\ge4\), \(T\ge1\), and \(D\ge2\), every full-information algorithm
has a deterministic reward sequence with \(L^\star=0\) such that
\[
 \E[\Rfs]\ge2^{-17}\min\{T,\sqrt{T\log K}\}.
\]
The algorithm may know that \(L^\star=0\) before play.
\end{lemma}

\begin{proof}[Full-information lower bound when \(L_0=0\)]
We again prove the zero-loss statement for every integer \(T\ge1\).
Put \(q=\lfloor K/2\rfloor\), \(b=\lfloor\log_2q\rfloor\ge1\), and
\(W_f=\sqrt{T\min\{T,b\}}\). If \(W_f\ge128\), let
\[
 h=\lfloor W_f/64\rfloor,\quad n=T-h,\quad
 \ell=\min\{n,b\},\quad \tau=\lfloor n/\ell\rfloor,\quad
 V=\sqrt{\ell n}.
\]
Split \(2q\) arms into groups \(A_1,\ldots,A_q\) and
\(B_1,\ldots,B_q\), with one unused arm \(E\) when \(K\) is odd. During
the first \(h\) rounds, give every \(A_i\) reward \((1,0)\), every
\(B_i\) reward \((0,1)\), and \(E\) reward \((0,0)\). Independently choose
a case \(w\in\{A,B\}\) uniformly before play. In case \(A\), coordinate
1 will be the zero-loss coordinate of group \(A\); in case \(B\),
coordinate 2 will be the zero-loss coordinate of group \(B\). The prefix
reward matrices are identical in the two cases, so full information does not
reveal \(w\). If \(X_A,X_B,X_E\) are the algorithm's prefix pull counts,
then its deficits in the two possible zero-loss coordinates are
\[
 d_A=h-X_A,\qquad d_B=h-X_B,\qquad
 d_A+d_B=h+X_E\ge h.
\]
Conditional on the prefix history, the uniform case therefore satisfies
\(d_w\ge h/2\) with probability at least \(1/2\). We reveal \(w\) after
the prefix, which can only help the algorithm.

Within the group that keeps the zero-loss coordinate, label
\(2^\ell\le q\) candidate arms by all sign vectors in \(\{-1,1\}^\ell\).
Duplicate any labels to fill the remaining positions in this group.
Partition \(\ell\tau\) suffix rounds into \(\ell\) blocks of length
\(\tau\). On block \(j\), let candidate \(i\)'s other coordinate be
\[
 Y_{i,t}=(1+\sigma_j(i)\epsilon_t)/2,
\]
where the \(\epsilon_t\) are independent uniform signs generated before
play. For the remaining suffix rounds use \(Y_{i,t}=1/2\).
The zero-loss coordinate is one for all candidates and zero outside
their group. Every noncandidate has scalar reward \(1/2\) in the
other coordinate. Thus every support sequence still has \(L^\star=0\).

Write \(S_j=\sum_{t\text{ in block }j}\epsilon_t\). In each block \(j\),
choose the sign \(\sigma_j\) to agree with the sign of \(S_j\), using either
sign when \(S_j=0\). One candidate has this entire sign vector, so the best
candidate's scalar suffix reward is \(n/2+Q\), where
\[
 Q=\tfrac12\sum_{j=1}^{\ell}|S_j|.
\]
No other candidate can obtain more in any block, so this maximum is attained.
For a suffix round \(t\) in block \(j\), define
\[
 \xi_t:=\begin{cases}
  \sigma_j(i)\epsilon_t/2,&\text{if the algorithm selects candidate }i,\\
  0,&\text{if it selects a noncandidate.}
 \end{cases}
\]
Its scalar reward on that round is \(1/2+\xi_t\). The selected arm and its
sign are fixed before the independent sign \(\epsilon_t\) is revealed, so
\(\E[\xi_t\mid\mathcal F_{t-1}]=0\) and
\(\E[\xi_t^2\mid\mathcal F_{t-1}]\le1/4\). Full information arrives only
after the action and does not change these conditional statements. Thus
\(M:=\sum_t\xi_t\) is a martingale, the algorithm's scalar suffix reward is
\(n/2+M\), and martingale orthogonality gives \(\E[M^2]\le n/4\).

For completeness, a sum \(S_j\) of \(\tau\) independent signs satisfies
\(\E[S_j^2]=\tau\) and \(\E[S_j^4]=3\tau^2-2\tau\le3\tau^2\).
Hölder's inequality gives
\(\E[S_j^2]\le(\E[|S_j|])^{2/3}(\E[S_j^4])^{1/3}\). Rearranging and using
the two moments yields
\[
 \E[|S_j|]\ge
 \frac{(\E[S_j^2])^{3/2}}{(\E[S_j^4])^{1/2}}
 \ge\sqrt{\frac\tau3}.
\]
Because \(n/\ell\ge1\), the inequality
\(\lfloor x\rfloor\ge x/2\) for \(x\ge1\) gives
\(\ell\tau\ge n/2\). Hence
\[
 \E[Q]\ge\frac{\ell\sqrt{\tau}}{2\sqrt3}
 \ge\frac{\sqrt{\ell n}}{2\sqrt6}>V/5.
\]
By Cauchy--Schwarz,
\(Q^2=(\sum_j|S_j|)^2/4\le(\ell/4)\sum_jS_j^2\). Therefore
\(\E[Q^2]\le\ell^2\tau/4\le\ell n/4=V^2/4\).
Set \(Z=Q-M\), the best candidate's scalar suffix advantage. Then
\[
 \E[Z]\ge V/5,\qquad
 \E[Z^2]\le2\E[Q^2]+2\E[M^2]\le V^2,
\]
because \(V^2=\ell n\ge n\). Let
\(p=\Prob\{Z\ge V/10\}\). Splitting \(\pos Z\) at \(V/10\) and applying
Cauchy--Schwarz on the upper event gives
\[
 \frac V5\le\E[Z]\le\E[\pos{Z}]
 \le\frac V{10}+\sqrt{\E[Z^2],p}
 \le\frac V{10}+V\sqrt p.
\]
Subtracting \(V/10\) and dividing by \(V>0\) gives
\(1/10\le\sqrt p\). Squaring both sides gives \(p\ge1/100\).
These inequalities hold uniformly conditional on the prefix history
and the revealed case.

Since \(W_f\ge128\), \(W_f/64\ge2\), and therefore
\(h=\lfloor W_f/64\rfloor\ge W_f/128\). Also
\(h\le W_f/64\le T/64\), so \(n\ge63T/64\). This also gives
\[
 \min\{n,b\}\ge\frac{63}{64}\min\{T,b\}.
\]
Indeed, if \(b\le n\), the left side is \(b\), which is at least the
right side. If \(b>n\), then the left side is \(n\ge63T/64\), while the
right side is at most \(63T/64\). Multiplying this inequality by
\(n\ge63T/64\) and taking square roots yields
\[
 V=\sqrt{n\min\{n,b\}}
 \ge\frac{63}{64}\sqrt{T\min\{T,b\}}
 =\frac{63}{64}W_f.
\]
Since \(h\le W_f/64\) and \(63/640>1/32\), we have \(V/10\ge2h\).
On the joint event that the initial deficit is at least \(h/2\) and
\(Z\ge V/10\), choose a scalar-maximizing candidate. This same arm has
final advantage at least \(h/2\) in the zero-loss coordinate and at least
\(Z-h\ge h\) in the other coordinate. The two events have joint
probability at least \((1/2)(1/100)=1/200\), so
\[
 \E[\Rfs]\ge\frac{h}{400}\ge\frac{W_f}{51200}.
\]

To compare \(W_f\) with the claimed rate, first note that \(K\le3q\).
For \(q\ge2\),
\[
 b=\lfloor\log_2q\rfloor\ge\tfrac12\log_2q
 =\frac{\log q}{2\log2}.
\]
Moreover, \(\log K\le\log(3q)\le(\log6/\log2)\log q\), so
\(b\ge(\log K)/(2\log6)\ge(\log K)/4\). Therefore
\[
 W_f=\sqrt{T\min\{T,b\}}
 \ge\tfrac12\sqrt{T\min\{T,\log K\}}
 =\tfrac12\min\{T,\sqrt{T\log K}\}.
\]
If \(W_f<128\), the one-round hidden-arm construction in the bandit
proof gives expected regret \(1-1/K\ge3/4\) even with full information,
because the reward row is revealed only after the first action.
The target rate is then smaller than \(256\), and
\(3/4>2^{-17}\cdot256\). In the large case,
\(W_f/51200\ge2^{-17}\min\{T,\sqrt{T\log K}\}\).
Both cases therefore give the stated constant \(2^{-17}\).
Duplication of a coordinate handles \(D>2\), and averaging over the
finite reward distribution gives the deterministic-sequence statement.
\end{proof}

\label{app:full-info-conditional}

We next extend the construction from zero loss to a specified value of
\(L_0\).

\begin{lemma}[Full-information lower bound at a specified loss]
\label{lem:full-information-conditional-lower}
Under the conditions of Theorem~\ref{thm:full-information-app}, every
full-information algorithm has a deterministic reward sequence with
\(L^\star=L_0\) such that
\[
 \E[\Rfs]\ge2^{-20}\min\{U_0,\sqrt{U_0\log K}\}.
\]
\end{lemma}

\begin{proof}[Full-information lower bound for a specified \(L_0\)]
For \(U_0\ge1\), apply the preceding full-information zero-loss
construction for \(m=\lfloor U_0\rfloor\) rounds, and append common
rewards of total \(U_0-m\) as in the proof of
Theorem~\ref{thm:conditional-lower}. The cumulative advantages are
unchanged and the easiest-coordinate loss is exactly \(L_0\).
Since \(m\ge U_0/2\), we have both
\(m\ge U_0/2\) and
\(\sqrt{m\log K}\ge\frac12\sqrt{U_0\log K}\). Therefore
\[
 \min\{m,\sqrt{m\log K}\}
 \ge\tfrac12\min\{U_0,\sqrt{U_0\log K}\}.
\]
For \(0<U_0<1\), use the one-round hidden reward \(U_0\one\) from that
proof. The expected regret is \((1-1/K)U_0\ge3U_0/4\). Since
\(K\ge4\) implies \(\log K>1>U_0\), the target minimum is \(U_0\) in
this range. When \(U_0=0\), the target lower bound is zero. The first case
gives constant \(2^{-18}\), and the second gives constant \(3/4\), so the
smaller constant \(2^{-20}\) is valid in all three cases.
\end{proof}

For the upper bound, exponential weights is run on one coordinate fixed
before play.

\begin{lemma}[Full-information upper bound at a specified loss]
\label{lem:full-information-conditional-upper}
Under the conditions of Theorem~\ref{thm:full-information-app}, there is a
full-information algorithm that knows \(L_0\) and satisfies, on every reward
sequence with \(L^\star=L_0\),
\[
 \E[\Rfs]\le4\min\{U_0,\sqrt{U_0\log K}\}.
\]
\end{lemma}

\begin{proof}[Full-information upper bound for a specified \(L_0\)]
The learner may use the class parameter \(L_0\), hence \(U_0\), and it
attends to a single coordinate.  Fix a coordinate \(d'\) before play and a
deterministic reward sequence in the exact class \(L^\star=L_0\), and write
\[
 g_{i,t}:=r_{d'}^{i,t},\qquad
 U_{d'}=\max_i\sum_{t=1}^Tg_{i,t}\le U_0,
\]
where the inequality holds because \(U_{d'}\le U^\star=T-L^\star=U_0\) on the
class.  If \(U_0\le\log K\), play any fixed arm.  The pathwise cap
\eqref{eq:fixed-coordinate}, which reads \(\Rfs\le U_0\) on the class, gives
\(\E[\Rfs]\le U_0=\min\{U_0,\sqrt{U_0\log K}\}\), the equality because
\(U_0\le\log K\) makes \(U_0\) the smaller term, and this includes
\(U_0=0\).

Suppose now that \(U_0>\log K\).  Run exponential weights on the rewards
of coordinate \(d'\), that is, the Hedge algorithm applied to gains rather
than losses, with its learning rate set using the reward scale \(U_0\)
instead of the horizon.  Before round \(t\), set
\[
 p_{i,t}
 =\frac{\exp(\eta_0\sum_{s<t}g_{i,s})}
        {\sum_{j=1}^K\exp(\eta_0\sum_{s<t}g_{j,s})},
 \qquad
 \eta_0:=\sqrt{\frac{\log K}{2U_0}}<1,
\]
and draw \(a_t\sim p_t\) with fresh randomness.  The update uses the full
reward row and not the sampled action.  Consequently, for each fixed reward
sequence, the sequence \((p_t)_{t\le T}\) and the mixed cumulative gain
\[
 V:=\sum_{t=1}^T p_t\mathbin\cdot g_t,
\]
the reward that the distributions \(p_t\) collect in expectation, are
deterministic.

Let \(W_t=\sum_i\exp(\eta_0\sum_{s<t}g_{i,s})\) be the total weight
before round \(t\), so \(W_1=K\). Convexity gives
\(e^{\eta_0x}\le1+(e^{\eta_0}-1)x\) for \(x\in[0,1]\). Hence
\begin{align*}
 \frac{W_{t+1}}{W_t}
 &=\sum_{i=1}^Kp_{i,t}e^{\eta_0g_{i,t}}\\
 &\le1+(e^{\eta_0}-1)p_t\mathbin\cdot g_t.
\end{align*}
Using \(\log(1+x)\le x\) and summing over \(t\) gives
\[
 \log\frac{W_{T+1}}{W_1}
 \le(e^{\eta_0}-1)\sum_{t=1}^Tp_t\mathbin\cdot g_t
 =(e^{\eta_0}-1)V.
\]
If \(i^\star\) is a best arm, its term in \(W_{T+1}\) equals
\(e^{\eta_0U_{d'}}\). Therefore
\[
 \eta_0U_{d'}-\log K
 \le\log\frac{W_{T+1}}{W_1}
 \le(e^{\eta_0}-1)V.
\]
Subtracting \(\eta_0V\), dividing by \(\eta_0\), and using
\(e^x\le1+x+x^2\) for \(x\in[0,1]\) gives
\begin{equation}
 U_{d'}-V
 \le\frac{\log K}{\eta_0}+\eta_0V.
 \label{eq:conditional-fi-potential}
\end{equation}
Put
\[
 X:=\sum_{t=1}^Tg_{a_t,t},\qquad M:=V-X,
\]
so that \(X\) is the learner's realized reward in coordinate \(d'\) and
\(M\) is the sampling noise. Each increment
\(p_t\mathbin\cdot g_t-g_{a_t,t}\) has mean zero given the history
\(\mathcal F_{t-1}\) before round \(t\), so \(M\) is a mean-zero
martingale. Martingale orthogonality and \(g_{i,t}^2\le g_{i,t}\) give
\begin{equation}
 \E[M^2]
 =\sum_{t=1}^T\E[\operatorname{Var}(g_{a_t,t}\mid\mathcal F_{t-1})]
 \le\sum_{t=1}^T\E[p_t\mathbin\cdot g_t^2]
 \le V.
 \label{eq:conditional-fi-variance}
\end{equation}
In particular, since \(M\) has mean zero, its positive and negative parts
have equal expectation, and the Cauchy--Schwarz inequality gives
\(\E[\pos{M}]=\frac12\E[|M|]\le\frac12\sqrt V\).

If \(V\le2U_0\), the mixed gain stays within the reward scale and the
potential bound controls the regret.  Then \eqref{eq:fixed-coordinate},
which gives \(\Rfs\le\pos{U_{d'}-X}=\pos{(U_{d'}-V)+M}\), the bound
\eqref{eq:conditional-fi-potential}, and the positive-part triangle
inequality \(\pos{x+y}\le\pos{x}+\pos{y}\) give
\begin{align*}
 \E[\Rfs]
 &\le \pos{U_{d'}-V} + \E[\pos{M}]\\
 &\le \frac{\log K}{\eta_0}+2\eta_0U_0
      +\frac12\sqrt{2U_0}\\
 &=2\sqrt{2U_0\log K}+\sqrt{U_0/2}
 <4\sqrt{U_0\log K},
\end{align*}
where the last inequality follows from
\(2\sqrt2+1/\sqrt{2\log K}<4\) for \(K\ge4\).

If \(V>2U_0\), the mixed gain is so far ahead of every fixed arm that only
a large sampling deviation can make the regret positive.  Let
\(h:=V-U_{d'}\) be this lead, so \(h>V/2\) because \(U_{d'}\le U_0<V/2\).
Pathwise,
\(\Rfs\le\pos{U_{d'}-X}=\pos{M-h}\). For \(z<h\),
\(\pos{z-h}=0\). For \(z\ge h\), the inequality \((z-2h)^2\ge0\) is
equivalent to \(z-h\le z^2/(4h)\). Thus
\(\pos{z-h}\le z^2/(4h)\) for every real \(z\) and \(h>0\). Apply this
inequality with \(z=M\), and then use
\eqref{eq:conditional-fi-variance}, to obtain
\[
 \E[\Rfs]
 \le\frac{\E[M^2]}{4h}
 \le\frac{V}{4(V-U_{d'})}
 <\frac12
 <\sqrt{U_0\log K}.
\]
The final inequality holds because
\(U_0\log K>(\log K)^2\ge(\log4)^2>1/4\).  Combining
the two cases with the pathwise cap \(\Rfs\le U_0\) proves
\[
 \E[\Rfs]\le4\min\{U_0,\sqrt{U_0\log K}\}.
\]
The policy uses \(U_0=T-L_0\), which the learner may know because the
conditional minimax value lets the learner use the class parameter \(L_0\)
(Section~\ref{sec:objective}).
\end{proof}

\begin{proof}[Proof of Theorem~\ref{thm:full-information-app}]
Lemma~\ref{lem:full-information-conditional-lower} says that for every
full-information algorithm there is a reward sequence in the class
\(L^\star=L_0\) on which its expected regret is at least
\(2^{-20}\min\{U_0,\sqrt{U_0\log K}\}\). Taking the supremum over that
class and then the infimum over algorithms therefore gives the same lower
bound on \(\mathcal V^{\mathrm{fi}}_{K,D,T}(L_0)\).
Lemma~\ref{lem:full-information-conditional-upper} provides one algorithm
whose expected regret is at most
\(4\min\{U_0,\sqrt{U_0\log K}\}\) on every sequence in the class. Taking
the same supremum and infimum gives the upper bound. The two inequalities
prove the theorem.
\end{proof}

\FloatBarrier
\subsection{Detailed comparison guarantees}
\label{app:comparison-statements}

We first give the guarantees for EXP3 and its variants. Each method
uses its original update rule, and its parameters are chosen without
knowing \(L_0\).

\begin{proposition}[EXP3 variants when \(L_0\) is unknown]
\label{prop:existing-bounds}\label{prop:original-gain-bounds}
Fix \(K\ge2,T\ge2,D\ge1\) and choose \(d'\) before play. On every
reward sequence with \(L^\star=L_0\),
\begin{equation}
\begin{aligned}
 \E[\Rfs]
 &\lesssim \min\{U_0,\sqrt{KT\log(KT)}\}
 &&\text{for EXP3.P and adversarial MO-KS},\\
 \E[\Rfs]
 &\lesssim \min\{U_0,\sqrt{KT\log K}\}
 &&\text{for EXP3-IX},\\
 \E[\Rfs]
 &\lesssim \min\{U_0,\sqrt{K L_{d'}\log(KT)}+K\log(KT)\}
 &&\text{for GREEN-IX},\\
 \E[\Rfs]
 &\lesssim \min\{U_0,(K\log K)^{1/3}T^{2/3}\}
 &&\text{for gain-based EXP3}.
\end{aligned}
\label{eq:unknown-l0-standard-bounds}
\end{equation}
Original EXP3.1 satisfies
\[
 \E[\Rfs]
 \lesssim \min\left\{U_0,\sqrt{KU_0\log K}
 +\left(\frac{KU_0^3}{\log K}\right)^{1/4}+K\log K\right\}.
\]
\end{proposition}

The first four bounds depend on \(T\) or on the chosen coordinate's
loss \(L_{d'}\). The EXP3.1 bound depends on \(U_0\), with extra terms
beyond \(\sqrt{KU_0\log K}\). These bounds can therefore be larger
than the bound in Theorem~\ref{thm:adaptive}.

For Tsallis-INF, EXP3++, and MO-US, the available guarantees bound
signed expected scalar regret. To bound the expectation of its positive part, we add a
term measuring the extra reward from choosing the best arm separately
on each round.

\begin{proposition}[Original Tsallis-INF, EXP3++, and MO-US]
\label{prop:tsallis-mous-positive}
For a coordinate \(d'\) fixed before play, define
\[
 S_{d'}:=\sum_{t=1}^T\max_i r_{d'}^{i,t}-U_{d'}.
\]
The quantity \(S_{d'}\) measures how much more reward we could obtain
by choosing the best arm on each round than by choosing the best fixed
arm. Each method satisfies
\begin{equation}
\begin{aligned}
 \E[\Rfs]
 &\lesssim \min\{U_0,\sqrt{KT}+S_{d'}\}
 &&\text{for Tsallis-INF (IW)},\\
 \E[\Rfs]
 &\lesssim \min\{U_0,\sqrt{KT}+K\log T+S_{d'}\}
 &&\text{for Tsallis-INF (RV)},\\
 \E[\Rfs]
 &\lesssim \min\{U_0,\sqrt{KT\log K}+S_{d'}\}
 &&\text{for EXP3++},\\
 \E[\Rfs]
 &\lesssim \min\{U_0,\sqrt{KT\log K}+K+S_{d'}\}
 &&\text{for MO-US}.
\end{aligned}
\label{eq:unknown-l0-switching-bounds}
\end{equation}
The same bounds hold with \(L_{d'}\) in place of \(S_{d'}\), because
\(S_{d'}\le L_{d'}\).
\end{proposition}

The algorithms do not need to know \(S_{d'}\) or \(L_{d'}\).
These quantities appear in the bounds and can be larger than
\(\min\{U_0,\sqrt{KU_0}\}\), even when \(L^\star=L_0\). For example, if \(D=1\), \(T\) is divisible by
\(K\), and the arms take turns receiving the only unit reward, then
\(U_0=T/K\) and \(S_{d'}=T-T/K\). The bound that includes \(S_{d'}\) can therefore be larger than the
optimal rate in this case.

We also consider versions that restart with fresh randomness after
blocks of a fixed length. Their bounds depend on \(K\) and \(T\),
so we no longer need the term \(S_{d'}\).

\begin{proposition}[Fixed-block versions of Tsallis-INF, EXP3++, and MO-US]
\label{prop:restarted-methods}
Algorithm~\ref{alg:block-restarts}, with block length
\(h=\min\{T,\lceil\sqrt{KT}\rceil\}\), does not use \(L_0\) and
satisfies
\begin{equation}
\begin{aligned}
 \E[\Rfs]
 &\lesssim \min\{U_0,K^{1/4}T^{3/4}\}
 &&\text{for Tsallis-INF (IW)},\\
 \E[\Rfs]
 &\lesssim \min\{U_0,K^{1/4}T^{3/4}\log(2T)\}
 &&\text{for Tsallis-INF (RV) and MO-US},\\
 \E[\Rfs]
 &\lesssim \min\{U_0,K^{1/4}T^{3/4}\sqrt{\log K}\}
 &&\text{for EXP3++}.
\end{aligned}
\label{eq:unknown-l0-restart-bounds}
\end{equation}
\end{proposition}

The \(T^{3/4}\) terms can be larger than the optimal rate, so the
bounds for these restarted methods are weaker than the bound in
Theorem~\ref{thm:adaptive}.

The bounds also hold when the algorithms use a fixed weighted average
of coordinates. For GREEN-IX, Tsallis-INF, EXP3++, and MO-US, we then compute the
loss or switching term from these weighted rewards.

\subsection{Algorithm specifications}
\label{app:algorithms}

This section gives pseudocode only for algorithms introduced or modified
in Section~\ref{sec:alg}. The order follows the main text. Every method
uses rewards already observed, and the coordinate or fixed weighted
average is chosen before play. 

\subsubsection{Scalar Poly-INF}

Both optimal algorithms use the scalar policy below. Its scalar input is
\(g_{i,t}=r_{d'}^{i,t}\) for a coordinate \(d'\) fixed before play. The
policy and its logarithmic reward estimator are from
\citet{audibert2010regret}.

\begin{algorithm}[!ht]
  \caption{\(\polyinf(B)\) with a logarithmic reward estimator}
  \label{alg:polyinf-full}
  \begin{algorithmic}[1]
    \Require Number of arms \(K\ge2\), scale parameter \(B\ge81K\)
    \State Set \(\eta\gets2\sqrt B\),
    \(\gamma\gets3\sqrt{K/B}\), and
    \(\lambda\gets1/\sqrt{2KB}\).
    \State Define \(\psi:(-\infty,0)\to(0,\infty)\) by
    \(\psi(x)=(\eta/(-x))^2+\gamma/K\).
    \State Set \(V_{i,0}\gets0\) and \(p_{i,1}\gets1/K\) for all \(i\).
    \For{local rounds \(s=1,2,\ldots\)}
      \State Draw \(I_s\sim p_s\), receive reward \(g_{I_s,s}\in[0,1]\).
      \State Set \(v_{i,s}\gets0\) for every \(i\ne I_s\), and set
      \[
        v_{I_s,s}
        \gets-\frac1\lambda
        \log\!\left(1-\frac{\lambda g_{I_s,s}}{p_{I_s,s}}\right).
      \]
      \State Set \(V_{i,s}\gets V_{i,s-1}+v_{i,s}\) for every \(i\).
      \State Find the unique \(z_s>\max_iV_{i,s}\) such that
      \(\sum_{i=1}^K\psi(V_{i,s}-z_s)=1\).
      \State Set \(p_{i,s+1}\gets\psi(V_{i,s}-z_s)\) for every \(i\).
    \EndFor
  \end{algorithmic}
\end{algorithm}
\FloatBarrier

The estimator and normalization in Algorithm~\ref{alg:polyinf-full} are
well defined. Since \(p_{I_s,s}\ge\gamma/K\),
\[
 \frac{\lambda g_{I_s,s}}{p_{I_s,s}}
 \le\frac{\lambda K}{\gamma}=\frac1{3\sqrt2}<1,
\]
so the logarithm has a positive argument. The map
\(c\mapsto\sum_i\psi(V_{i,s}-c)\) is continuous and strictly decreasing
on \((\max_iV_{i,s},\infty)\), diverges at the left endpoint, and tends
to \(\gamma\le1/3\) at infinity. Hence \(z_s\) exists and is unique. It
lies in
\[
 \left(\max_iV_{i,s}+\eta,\;
 \max_iV_{i,s}+\eta\sqrt{K/(1-\gamma)}\right].
\]
At the left endpoint the largest summand exceeds one. At the right
endpoint each summand is at most \((1-\gamma)/K+\gamma/K\), so bisection
on this interval computes \(z_s\).

\subsubsection{Algorithms used when \texorpdfstring{\(L_0\)}{L0} is known}

The first algorithm receives \(B=U_0=T-L_0\). It runs Poly-INF when the
supplied reward scale is large enough for its guarantee and otherwise
uses the pathwise bound \(\Rfs\le U_0\).

\begin{algorithm}[!ht]
  \caption{Fixed-coordinate policy when \(L_0\) is known}
  \label{alg:known-scale}
  \begin{algorithmic}[1]
    \Require Arms \(K\ge2\), horizon \(T\), scale parameter \(B\in[0,T]\),
    precommitted coordinate \(d'\in[D]\) (e.g., \(d'=1\))
    \If{\(B\ge81K\)}
      \State Start a fresh copy of \(\polyinf(B)\)
      (Algorithm~\ref{alg:polyinf-full}).
      \For{\(t=1,\ldots,T\)}
        \State Draw \(a_t\) from the arm distribution maintained by the
        current copy and observe \(r^{a_t,t}\).
        \State Feed back the scalar reward \(r_{d'}^{a_t,t}\).
      \EndFor
    \Else
      \State Play arm 1 on every round.
    \EndIf
  \end{algorithmic}
\end{algorithm}
\FloatBarrier

Section~\ref{sec:known-l0} next considers modified EXP3 and gain-based
EXP3. Modified EXP3 uses the logarithmic estimator below. Gain-based
EXP3 keeps the original update and changes only the theoretically chosen
exploration parameter.

\begin{algorithm}[!ht]
  \caption{EXP3 with logarithmic reward estimates}
  \label{alg:exp3-log-rewards}
  \begin{algorithmic}[1]
    \Require Arms \(K\ge2\), horizon \(T\), reward bound \(B=U_0=T-L_0\),
      coordinate \(d'\) chosen before play
    \State Set \(H=\log(3K)\).
    \If{\(B\le4KH\)}
      \State Play one fixed arm on every round.
    \Else
      \State Set \(\gamma=2\sqrt{KH/B}\),
        \(\eta=2\sqrt{H/(KB)}\), and \(\beta=\sqrt{H/(2KB)}\).
      \State Initialize \(w_{i,1}=1\) for every arm \(i\).
      \For{\(t=1,\ldots,T\)}
        \State Set \(q_{i,t}=w_{i,t}/\sum_jw_{j,t}\) and
          \(p_{i,t}=(1-\gamma)q_{i,t}+\gamma/K\) for every \(i\).
        \State Draw \(a_t\sim p_t\) and observe \(g_t=r_{d'}^{a_t,t}\).
        \State Set \(v_{a_t,t}=-\beta^{-1}\log(1-\beta g_t/p_{a_t,t})\)
          and \(v_{i,t}=0\) for \(i\ne a_t\).
        \State Set \(w_{i,t+1}=w_{i,t}\exp(\eta v_{i,t})\) for every \(i\).
      \EndFor
    \EndIf
  \end{algorithmic}
\end{algorithm}
\FloatBarrier

\begin{algorithm}[!ht]
 \caption{Gain-based EXP3 with a supplied exploration parameter}
 \label{alg:original-exp3}
 \begin{algorithmic}[1]
  \Require \(K,T\), \(\gamma\in(0,1]\), a coordinate or fixed weights
  \State Set \(\eta=\gamma/K\) and \(w_i=1\) for all \(i\).
  \For{\(t=1,\ldots,T\)}
   \State Set \(q_i=w_i/\sum_jw_j\) and \(p_i=(1-\gamma)q_i+\gamma/K\).
   \State Draw \(a_t\sim p\) and observe its scalar reward \(g_{a_t,t}\).
   \State Set \(\widehat g_{i,t}=g_{a_t,t}\one\{i=a_t\}/p_i\).
   \State Update \(w_i\gets w_i\exp(\eta\widehat g_{i,t})\) for all \(i\).
  \EndFor
 \end{algorithmic}
\end{algorithm}
\FloatBarrier

EXP3.P, EXP3-IX, and the adversarial branch of MO-KS retain their
published updates. Their parameter choices and the resulting bounds are
given with their proofs in Appendix~\ref{app:known-l0-high-probability}.

\subsubsection{Algorithms used when \texorpdfstring{\(L_0\)}{L0} is unknown}

Reward-doubling Poly-INF starts from \(B=81K\) and doubles this value only
when the reward collected by the current copy reaches \(B\).

\begin{algorithm}[!ht]
  \caption{Reward-doubling \(\polyinf\)}
  \label{alg:adaptive}
  \begin{algorithmic}[1]
    \Require Arms \(K\ge2\), horizon \(T\), precommitted coordinate
    \(d'\in[D]\) (e.g., \(d'=1\))
    \State Set \(e\gets0\), \(B_e\gets81K\), \(Y_e\gets0\), and start a
    fresh, independently randomized \(\polyinf(B_e)\).
    \For{\(t=1,\ldots,T\)}
      \State Draw \(a_t\) from the current copy, observe
      \(g_t=r_{d'}^{a_t,t}\), update the copy, and set
      \(Y_e\gets Y_e+g_t\).
      \If{\(Y_e\ge B_e\) and \(t<T\)}
        \State Set \(e\gets e+1\), \(B_e\gets2B_{e-1}\), \(Y_e\gets0\),
        and start a fresh independent \(\polyinf(B_e)\) on round \(t+1\).
      \EndIf
    \EndFor
  \end{algorithmic}
\end{algorithm}
\FloatBarrier

The unknown-\(L_0\) comparison also considers fixed-block versions of
Tsallis-INF, EXP3++, and MO-US. The wrapper below restarts the selected
published algorithm independently in each block.

\begin{algorithm}[!ht]
 \caption{Fixed-block restarts of a scalar algorithm}
 \label{alg:block-restarts}
 \begin{algorithmic}[1]
  \Require Scalar algorithm \(\mathcal A\), \(K,T\), integer
   \(h\in\{1,\ldots,T\}\), a coordinate or fixed weights
  \For{blocks of consecutive rounds of length \(h\), with a shorter last block}
   \State Initialize a fresh copy of \(\mathcal A\), with independent
    random choices and its local time set to one.
   \State If \(\mathcal A\) needs a horizon, supply the length of this block.
   \State Run that copy for the entire block, supplying only its chosen
    arm's scalar reward or loss on each round.
  \EndFor
 \end{algorithmic}
\end{algorithm}
\FloatBarrier

\FloatBarrier
\subsection{Upper bounds when \texorpdfstring{\(L_0\)}{L0} is known}
\label{app:upper-proofs}

We follow Section~\ref{sec:known-l0}. We first record the scalar
reduction used by all comparison algorithms and the Poly-INF guarantee
used by both optimal algorithms. We then prove
Theorem~\ref{thm:conditional-upper} and analyze the methods in
Proposition~\ref{prop:known-l0-comparisons} in their main-text order.

\begingroup
\label{app:comparison}

\subsubsection{Common scalar reduction}

The scalar algorithms compared in Section~\ref{sec:alg} observe only
\(g_{i,t}=r_{d'}^{i,t}\) in the coordinate fixed before play.  A method
written for losses instead receives \(\ell_{i,t}=1-g_{i,t}\).  These
conventions give the same realized scalar regret,
\[
 R_T^{d'}=\max_i\sum_t g_{i,t}-\sum_t g_{a_t,t}
 =\sum_t\ell_{a_t,t}-\min_i\sum_t\ell_{i,t}.
\]
The best-arm loss in this scalar problem is \(L_{d'}\), not necessarily
\(L^\star\).  On the exact class, \(U_0=T-L_0\) and
\[
 0\le\Rfs\le\min\{U_0,\pos{R_T^{d'}}\}.
\]
This inequality is the basis for all upper-bound conversions below.

\subsubsection{Signed regret and positive regret}
\label{app:positive-conversion}

Fix the scalar rewards \(g_{i,t}\in[0,1]\) before play. Write
\(G_i=\sum_tg_{i,t}\), \(G^\star=\max_iG_i\),
\(A=\sum_tg_{a_t,t}\), and \(L_g=T-G^\star\).
For one coordinate, \(L_g=L_{d'}\). For a fixed weighted average it
is the best-arm loss for that average, not necessarily \(L^\star\).
Throughout these comparisons \(G^\star\le U_0\).

\begin{lemma}[Conversion with a bound on negative regret]
\label{lem:signed-positive}
Suppose a scalar algorithm satisfies \(\E[G^\star-A]\le B\), where
\(B\ge0\). Define \(S_g=\sum_t\max_i g_{i,t}-G^\star\). Then
\[
 \E[\Rfs]\le\min\{U_0,B+S_g\}
 \le\min\{U_0,B+L_g\}.
\]
\end{lemma}
\begin{proof}
Put \(R=G^\star-A\). Since \(A\le\sum_t\max_i g_{i,t}\),
\(\pos{-R}\le S_g\le T-G^\star=L_g\).
The identity \(\pos R=R+\pos{-R}\) gives
\[
 \E[\pos{R}]=\E[R]+\E[\pos{-R}]\le B+S_g\le B+L_g.
\]
Separately, the pathwise cap gives \(\E[\Rfs]\le U_0\). Combining this
cap with \(\Rfs\le\pos R\) proves both minima in the lemma.
\end{proof}

The extra term cannot simply be dropped. It measures how much an
algorithm could gain by switching arms, relative to the best fixed arm.
It is zero if one arm has the largest scalar reward on every round.
The lemma gives a valid bound for an unchanged algorithm, but in general
does not give a useful bound uniform over \(L^\star=L_0\).

\endgroup

\subsubsection{The scalar Poly-INF bound}
\label{app:polyinf}

Integrating the scalar tail bound over confidence levels gives the
positive-part expectation bound used by Algorithm~\ref{alg:known-scale}.
For reward doubling, we use the deterministic inequality within the
source proof and control the estimation error across all epochs together.
Neither that inequality nor the full scalar tail bound requires the
scale parameter \(B\) to exceed the best arm's reward. This matters because
an epoch can start with a value of \(B\) below that reward. The assumption
that all rewards are fixed before play makes the full-horizon best arm and
the scale index used below deterministic.

For a value \(B\), \(\polyinf(B)\) denotes the policy in
Algorithm~\ref{alg:polyinf-full}, with exponent \(2\) and parameters
\((\eta,\gamma,\lambda)=(2\sqrt B,\,3\sqrt{K/B},\,1/\sqrt{2KB})\).
The next lemma states the scalar guarantee used by both algorithms.

\begin{lemma}[Positive-regret bound for \(\polyinf\)]
\label{lem:polyinf-main}
For a deterministic \(K\)-armed scalar reward sequence \(g_{i,t}\in[0,1]\)
of length \(n\ge K\), write \(G_n^\star:=\max_i\sum_{t=1}^n g_{i,t}\) and
\(R_n:=G_n^\star-\sum_{t=1}^n g_{a_t,t}\).  For every
\(B\ge81K\) and every \(\delta\in(0,1)\), the policy \(\polyinf(B)\), which
does not depend on \(\delta\), satisfies
\begin{equation}
 \Prob\!\left\{
 R_n>
 4.5G_n^\star\sqrt{\frac KB}+4\sqrt{KB}
 +\sqrt{2KB}\log(\delta^{-1})
 \right\}\le\delta.
 \label{eq:polyinf-full-tail}
\end{equation}
If \(G_n^\star\le B\), then
\begin{equation}
  \E\!\left[\pos{R_n}\right]
  \le(8.5+\sqrt2)\sqrt{KB}<10\sqrt{KB}.
  \label{eq:polyinf-positive}
\end{equation}
\end{lemma}

\begin{proof}[Proof of Lemma~\ref{lem:polyinf-main}]
The deterministic-adversary part of
\citet[Theorem~18, Eq.\ (18)]{audibert2010regret} considers a fixed gain
matrix in \([0,1]^{K\times n}\), where \(n\ge K\ge2\), and a parameter
\(G_0\ge81K\). With
\[
 q=2,\qquad \eta=2\sqrt{G_0},\qquad
 \gamma=3\sqrt{K/G_0},\qquad
 \beta_{\mathrm{AB}}=\frac1{\sqrt{2KG_0}},
\]
that theorem gives, with probability at least \(1-\delta\),
\[
 R_n^{\mathrm{AB}}
 \le4.5G_{\max}^{\mathrm{AB}}\sqrt{\frac K{G_0}}
     +4\sqrt{KG_0}
     +\sqrt{2KG_0}\log(\delta^{-1}).
\]
The source statement does not require
\(G_{\max}^{\mathrm{AB}}\le G_0\). Its notation corresponds to ours as
\[
 G_0\mapsto B,\quad
 G_{\max}^{\mathrm{AB}}\mapsto G_n^\star,\quad
 I_t^{\mathrm{AB}}\mapsto a_t,\quad
 R_n^{\mathrm{AB}}\mapsto G_n^\star-\sum_{t=1}^n g_{a_t,t},\quad
 \beta_{\mathrm{AB}}\mapsto\lambda.
\]
The update in Algorithm~\ref{alg:polyinf-full} is the same policy after
these substitutions. It depends on \(K\) and \(B\), but not on \(n\) or
\(\delta\), so every confidence level concerns the same algorithm. The
source bound is therefore exactly \eqref{eq:polyinf-full-tail}.

It remains to prove \eqref{eq:polyinf-positive}.  If \(G_n^\star\le B\),
the first two terms in
\eqref{eq:polyinf-full-tail} are at most
\(4.5\sqrt{KB}+4\sqrt{KB}=8.5\sqrt{KB}\).  Set
\(a=8.5\sqrt{KB}\) and \(b=\sqrt{2KB}\).  For \(x>a\), substitute
\(\delta=\exp(-(x-a)/b)\in(0,1)\) into \eqref{eq:polyinf-full-tail} to
obtain \(\Prob\{R_n>x\}\le\exp(-(x-a)/b)\).  The layer-cake identity,
which writes the expectation of a nonnegative random variable as the
integral of its tail probabilities, gives
\[
 \E\!\left[\pos{R_n}\right]
 =\int_0^\infty\Prob\{R_n>x\}\,dx
 \le a+\int_a^\infty e^{-(x-a)/b}\,dx
 =a+b,
\]
which is \eqref{eq:polyinf-positive}.
\end{proof}

Lemma~\ref{lem:polyinf-main} speaks about a scalar bandit problem, so we
record how its quantities read once the scalar rewards are those of the
precommitted coordinate.
When Algorithm~\ref{alg:known-scale} runs \(\polyinf(B)\) on the scalar
rewards \(g_{i,t}=r_{d'}^{i,t}\) over the full horizon \(n=T\), the scalar
quantities in Lemma~\ref{lem:polyinf-main} are
\begin{equation}
 G_T^\star=\max_i\sum_{t=1}^Tr_{d'}^{i,t}\le U^\star,
 \qquad
 R_T=R_T^{d'},
 \qquad
 \Rfs\le\pos{R_T^{d'}}.
 \label{eq:fixed-dim-reduction}
\end{equation}
The first identity is the definition of scalar regret for this coordinate.
The inequality \(G_T^\star\le U^\star\) follows because \(U^\star\) is the maximum
cumulative reward over all arm-coordinate pairs. The last inequality is
the coordinate case of Lemma~\ref{lem:scalarization}.

\subsubsection{Known-scale Poly-INF}
\label{app:known-upper}

We now apply the scalar guarantee with the supplied value \(U_0\).

\begin{proof}[Proof of Theorem~\ref{thm:conditional-upper}]
Knowing \(L_0\) determines \(B=U_0\) for \(\polyinf\) and also gives a
pathwise cap.  The two branches of Algorithm~\ref{alg:known-scale} use one
of these facts each.  On the class \(L^\star=L_0\), \eqref{eq:LstarU} gives
\(\max_i\sum_tr_{d'}^{i,t}\le U^\star=U_0\) for every coordinate \(d'\), in
particular for the coordinate \(d'\) fixed before play on which
Algorithm~\ref{alg:known-scale} runs, so the best cumulative reward the
policy faces is at most the scale \(B=U_0\), and the pathwise bound
\eqref{eq:fixed-coordinate} gives \(\Rfs\le U_0\) on every reward sequence
in the class and every play path.  If \(U_0\ge81K\),
Algorithm~\ref{alg:known-scale} runs \(\polyinf(U_0)\) over the full
horizon \(T\ge U_0\ge K\), so Lemma~\ref{lem:polyinf-main} applies, and
its expectation bound \eqref{eq:polyinf-positive} together with the
reduction \eqref{eq:fixed-dim-reduction} gives
\(\E[\Rfs]<10\sqrt{KU_0}\), where \(\sqrt{KU_0}=\min\{U_0,\sqrt{KU_0}\}\)
because \(U_0\ge K\).  If \(U_0<81K\), Algorithm~\ref{alg:known-scale}
plays a fixed arm and the pathwise bound yields
\[
 \Rfs\le U_0\le9\min\{U_0,\sqrt{KU_0}\}.
\]
When \(U_0\le K\), the minimum is \(U_0\), so the inequality holds with
factor one. When \(K<U_0<81K\), the minimum is \(\sqrt{KU_0}\), and
dividing \(U_0\) by \(\sqrt{KU_0}\) gives \(\sqrt{U_0/K}<9\).  In both
branches \(\E[\Rfs]\le10\min\{U_0,\sqrt{KU_0}\}\), which proves the
theorem.
\end{proof}

The remaining proofs follow Proposition~\ref{prop:known-l0-comparisons}.
They analyze modified EXP3, gain-based EXP3, EXP3.P and MO-KS, and
EXP3-IX in that order.

\begingroup
\subsubsection{EXP3 with logarithmic reward estimates}
\label{app:logarithmic-exp3}
\label{app:known-comparisons}

Algorithm~\ref{alg:exp3-log-rewards} uses the exponential-weights policy
and logarithmic reward estimate of
\citet[Theorem~21]{audibert2010regret}.  The fixed-arm branch covers the
values of \(B\) excluded by that theorem.  It changes neither the
guarantee in Proposition~\ref{prop:exp3-log-rewards} nor the information
available to the algorithm.  Only the chosen arm's reward is used in
each update.

\begin{proof}[Modified-EXP3 part of
Proposition~\ref{prop:known-l0-comparisons}]
We prove the explicit bound
\[
 \E[\Rfs]\le
 \min\!\left\{U_0,(3+\sqrt2)\sqrt{KU_0\log(3K)}\right\}.
\]
Write \(B=U_0\), \(H=\log(3K)\), and
\(G_T^\star=\max_i\sum_tg_{i,t}=U_{d'}\le B\).
If \(B=0\), the bound \(\Rfs\le B\) proves the claim.  If
\(0<B\le4KH\), the same bound applies to the fixed-arm rule, and
\[
 B\le2\sqrt{KBH}\le(3+\sqrt2)\sqrt{KBH}.
\]
Thus the minimum in this explicit bound equals \(B\) in this case.

Suppose now that \(B>4KH\), so \(0<\gamma<1\).  The logarithmic
estimate is well defined because
\[
 0\le\frac{\beta g_t}{p_{a_t,t}}
 \le\frac{\beta K}{\gamma}=\frac1{2\sqrt2}<1.
\]
With \(G_0=B\), Algorithm~\ref{alg:exp3-log-rewards} is the policy in
\citet[Theorem~21]{audibert2010regret}.  Indeed, normalizing the
exponential part of its probability rule gives
\(p_{i,t}=(1-\gamma)w_{i,t}/\sum_jw_{j,t}+\gamma/K\).
For every \(\delta\in(0,1)\), that theorem gives
\[
 \Prob\!\left\{
 R_T^{d'}>
 \frac52G_T^\star\sqrt{\frac{KH}{B}}
 +\frac12\sqrt{KBH}
 +\sqrt{\frac{2KB}{H}}\log\frac K\delta
 \right\}\le\delta.
\]
The parameters do not depend on \(\delta\), so all these bounds concern
the same algorithm.  Since \(G_T^\star\le B\), define
\[
 s=\sqrt{KBH},\qquad b=\sqrt{\frac{2KB}{H}},\qquad
 a=3s+b\log K.
\]
Substituting \(G_T^\star\le B\) and
\(\log(K/\delta)=\log K+\log(1/\delta)\) into the preceding display gives
\(\Prob\{R_T^{d'}>a+b\log(1/\delta)\}\le\delta\).
For \(x>a\), set \(\delta=e^{-(x-a)/b}\). Then
\(\Prob\{R_T^{d'}>x\}\le e^{-(x-a)/b}\), so
\begin{align*}
 \E[\pos{R_T^{d'}}]
 &=\int_0^\infty\Prob\{R_T^{d'}>x\}\,dx\\
 &\le a+\int_a^\infty e^{-(x-a)/b}\,dx
 =3s+b(\log K+1)\\
 &\le(3+\sqrt2)s.
\end{align*}
The last inequality uses \(\log K+1\le\log(3K)=H\).
Finally, \(\Rfs\le U_0\) and
\(\Rfs\le\pos{R_T^{d'}}\) prove the modified-EXP3 bound in
Proposition~\ref{prop:known-l0-comparisons}.
\end{proof}

\endgroup

\begingroup
\subsubsection{Original gain-based EXP3}
\label{app:original-exp3}

This is the gain-based EXP3 update of \citet{auer2002nonstochastic}.
Only the choice of \(\gamma\) changes between the parameter settings below; no
bias or logarithmic transformation is added to its reward estimate.

\begin{proof}[Gain-based EXP3 bounds in Propositions~\ref{prop:known-l0-comparisons}
and~\ref{prop:existing-bounds}]
Choose a deterministic \(k\in\arg\max_iG_i\). If \(\gamma=1\), the
claimed bound follows from \(\Rfs\le U_0\), so suppose \(\gamma<1\).
Since \(0\le\eta\widehat g_{i,t}\le1\),
\(e^x\le1+x+(e-2)x^2\) on this interval. If
\(W_t=\sum_iw_{i,t}\), then
\begin{align*}
 \log\frac{W_{t+1}}{W_t}
 &=\log\sum_iq_{i,t}e^{\eta\widehat g_{i,t}}\\
 &\le\log\!\left(
  1+\eta\langle q_t,\widehat g_t\rangle
  +(e-2)\eta^2\langle q_t,\widehat g_t^2\rangle\right)\\
 &\le\eta\langle q_t,\widehat g_t\rangle
  +(e-2)\eta^2\langle q_t,\widehat g_t^2\rangle.
\end{align*}
For the comparator arm \(k\),
\(W_{T+1}\ge\exp(\eta\sum_t\widehat g_{k,t})\), while \(W_1=K\).
Summing the preceding upper bound, applying this lower bound, and dividing
by \(\eta\) gives
\[
 \sum_t\widehat g_{k,t}-\sum_t\langle q_t,\widehat g_t\rangle
 \le\frac{\log K}{\eta}
 +(e-2)\eta\sum_t\langle q_t,\widehat g_t^2\rangle.
\]
Here \(\langle p_t,\widehat g_t\rangle=g_{a_t,t}\) and
\((1-\gamma)q_{i,t}\le p_{i,t}\). Also,
\[
 A=\sum_t\langle p_t,\widehat g_t\rangle
 =(1-\gamma)\sum_t\langle q_t,\widehat g_t\rangle
   +\frac{\gamma}{K}\sum_{t,i}\widehat g_{i,t}.
\]
Multiply the potential inequality by \(1-\gamma\), use
\(\eta=\gamma/K\), and drop the final nonpositive term
\(-(\gamma/K)\sum_{t,i}\widehat g_{i,t}\). This yields
\[
 G^\star-A\le\gamma G^\star+(1-\gamma)M
 +\frac{(1-\gamma)K\log K}{\gamma}+Z,
\]
where
\[
 M=\sum_t(g_{k,t}-\widehat g_{k,t}),\qquad
 Z=(e-2)\frac\gamma K\sum_{t,i}p_{i,t}\widehat g_{i,t}^2\ge0.
\]
Conditional unbiasedness and martingale orthogonality give
\[
 \E[M]=0,\qquad
 \E[M^2]
 =\sum_tg_{k,t}^2\E[1/p_{k,t}-1]
 \le\frac K\gamma\sum_tg_{k,t}^2
 \le\frac{KU_0}{\gamma}.
\]
For any integrable mean-zero variable,
\(\E[\pos{M}]=\E[|M|]/2\).
Also, since each arm's total scalar reward is at most \(U_0\),
\[
 \E[Z]=(e-2)\frac\gamma K\sum_{t,i}g_{i,t}^2
 \le(e-2)\frac\gamma K\sum_iG_i
 \le(e-2)\gamma U_0.
\]
Take positive parts before expectation, use Cauchy--Schwarz for \(M\),
and drop factors \(1-\gamma\le1\). Together with \(\Rfs\le U_0\), this gives
\begin{equation}
 \E[\Rfs]\le\min\left\{U_0,
 (e-1)\gamma U_0+\frac{K\log K}{\gamma}
 +\frac12\sqrt{\frac{KU_0}{\gamma}}\right\}.
\label{eq:gain-exp3-parameter}
\end{equation}
This is the parameter-dependent EXP3 bound used in both comparison
propositions.

For \(V\ge U_0\), take
\(\gamma_V=\min\{1,(K\log K/V)^{1/3}\}\), with
\(\gamma_0=1\). If \(V\le K\log K\), the bound \(U_0\) suffices.
Otherwise, write \(C=K\log K\le V\), so
\(\gamma_V=(C/V)^{1/3}\). Since \(U_0\le V\), the three terms satisfy
\begin{align*}
 (e-1)\gamma_VU_0
 &\le(e-1)C^{1/3}V^{2/3},\\
 \frac{C}{\gamma_V}
 &=C^{2/3}V^{1/3}\le C^{1/3}V^{2/3},\\
 \frac12\sqrt{\frac{KU_0}{\gamma_V}}
 &\le\frac12K^{1/3}(\log K)^{-1/6}V^{2/3}
 \le\frac{1}{2\sqrt{\log2}}C^{1/3}V^{2/3}.
\end{align*}
Thus their sum is
\(O((K\log K)^{1/3}V^{2/3})\). Setting \(V=T\) proves the unknown-\(L_0\)
rate; setting \(V=U_0\) proves the known-\(L_0\) rate.
The displayed parameter-dependent bound also applies to the original
square-root parameter choice. That choice is used in the saved experiments.
\end{proof}

\endgroup

\begingroup
\subsubsection{EXP3.P, MO-KS, and EXP3-IX with known \texorpdfstring{\(L_0\)}{L0}}
\label{app:known-l0-high-probability}

We next check whether supplying \(L_0\) improves the parameter-dependent
bounds for EXP3.P or EXP3-IX. This gives the EXP3.P bound with
\(U_0\)-based parameters in Proposition~\ref{prop:known-l0-comparisons} and explains
why setting the EXP3-IX learning rate using \(U_0\) does not improve the bound
in our analysis. Neither calculation
replaces the method's reward or loss estimate.

\begin{lemma}[EXP3.P and EXP3-IX when \(L_0\) is known]
\label{lem:known-l0-comparisons}
Fix \(K\ge2,T\ge2,D\ge1\), \(L_0\in[0,T]\), and a coordinate \(d'\)
before play.  Put \(U_0=T-L_0\).  On every reward sequence with
\(L^\star=L_0\), the following bounds hold.
\begin{enumerate}
  \item Run EXP3.P on rewards in coordinate \(d'\), with
  \[
   q=\log\!\bigl(KT(T+1)^2\bigr),\qquad \alpha=2\sqrt q,\qquad
   \gamma=\min\!\left\{\frac35,
                  3\sqrt{\frac{K\log K}{5U_0}}\right\}.
  \]
  At \(U_0=0\), set \(\gamma=3/5\).  Then
  \begin{equation}
   \E[\Rfs]\le\min\!\left\{
   U_0,\;2\sqrt{5KU_0\log K}+4\sqrt{KTq}+8q
                       +\frac{U_0}{(T+1)^2}\right\}.
   \label{eq:exp3p-known-explicit}
  \end{equation}
  The same result applies to the adversarial branch of MO-KS with this
  parameter choice. In particular, this is an
  \(O(\min\{U_0,\sqrt{KT\log(KT)}\})\) upper bound.
  \item Run EXP3-IX on losses \(\ell_{i,t}=1-r_{d'}^{i,t}\), with any
  constant \(\eta>0\) and \(\gamma=\eta/2\).  Then
  \begin{equation}
   \E[\Rfs]\le\min\!\left\{
   U_0,\;\frac{2\log K+1+\log2}{\eta}+\eta KT+1+\log2\right\}.
   \label{eq:exp3ix-parameter-bound}
  \end{equation}
  In particular, the choice \(\eta=\sqrt{2\log K/(KT)}\) gives
  \(O(\min\{U_0,\sqrt{KT\log K}\})\).
\end{enumerate}
\end{lemma}

\begin{proof}
For EXP3.P, the sampling probabilities and updates are
\[
 p_{i,t}=(1-\gamma)\frac{w_{i,t}}{\sum_jw_{j,t}}+\frac\gamma K,
 \qquad
 w_{i,t+1}=w_{i,t}\exp\!\left[
 \frac{\gamma}{3K}\left(
 \frac{g_{i,t}\one\{a_t=i\}}{p_{i,t}}
 +\frac{\alpha}{p_{i,t}\sqrt{KT}}\right)\right],
\]
with \(w_{i,1}=1\).  The common initial factor in the source algorithm
cancels from its probabilities.  Set \(\delta_0=(T+1)^{-2}\), so
\(q=\log(KT/\delta_0)\).  The inequality \(q\le KT\) holds for
\(K,T\ge2\): \(KT-\log(KT(T+1)^2)\) is positive at \((2,2)\) and
increasing in each variable on this range.  Thus
\(\alpha\le2\sqrt{KT}\), as required by
\citet[Lemmas~6.1 and~6.2]{auer2002nonstochastic}.

Let \(G_T^\star=U_{d'}\le U_0\), and let \(\widehat U\) be the largest
upper confidence estimate in those lemmas.  With probability at least
\(1-\delta_0\), \(\widehat U\ge G_T^\star\), and the collected reward
\(G_T^{\rm alg}\) satisfies
\[
 G_T^{\rm alg}\ge
 \left(1-\frac{5\gamma}{3}\right)\widehat U
 -\frac{3K\log K}{\gamma}-2\alpha\sqrt{KT}-2\alpha^2.
\]
Since \(\gamma\le3/5\), this gives on the same event
\[
 R_T^{d'}\le
 \frac{5\gamma}{3}U_0+\frac{3K\log K}{\gamma}
 +4\sqrt{KTq}+8q.
\]
If \(U_0\ge5K\log K\), the chosen \(\gamma\) makes the first two
terms sum to \(2\sqrt{5KU_0\log K}\).  On the complementary event,
\(\Rfs\le U_0\), so its contribution to the expectation is at most
\(U_0\delta_0\).  This proves \eqref{eq:exp3p-known-explicit} in this
range.  If \(0<U_0<5K\log K\), then
\(2\sqrt{5KU_0\log K}>2U_0\), so the displayed minimum is \(U_0\),
which bounds every algorithm's regret.  The case \(U_0=0\) follows
from the same bound.  MO-KS runs this scalar method on a coordinate
chosen before play \citep[Algorithm~1]{xu2023pareto}, so the argument
also applies after conditioning on that choice.

To simplify \eqref{eq:exp3p-known-explicit}, first note that
\[
 \log(KT)\le q=\log\!\bigl(KT(T+1)^2\bigr)\le3\log(KT),
\]
because \(T+1\le KT\) for \(K,T\ge2\). Moreover, \(q\le KT\), as
verified above. Hence
\[
 \sqrt{KU_0\log K}\le\sqrt{KTq},\qquad
 q\le\sqrt{KTq},\qquad
 \frac{U_0}{(T+1)^2}\le1\le\sqrt{KTq}.
\]
Every term outside \(U_0\) in \eqref{eq:exp3p-known-explicit} is therefore
at most a numerical constant times \(\sqrt{KT\log(KT)}\). This proves the
compact upper bound stated in the lemma.

For EXP3-IX, the parameter-dependent inequality in the proof of
\citet[Theorem~1]{neu2015explore}, with \(\gamma=\eta/2\), gives
for every \(\delta\in(0,1)\)
\[
 \Prob\!\left\{R_T^{d'}>
 \frac{2\log K}{\eta}+\eta KT
 +\left(\frac1\eta+1\right)\log\frac2\delta\right\}\le\delta.
\]
For fixed \(\eta\), all confidence levels concern the same algorithm.
Put \(c=2\log K/\eta+\eta KT\) and \(d=1/\eta+1\). For
\(x>c+d\log2\), choosing
\(\delta=2e^{-(x-c)/d}\) gives
\(\Prob\{R_T^{d'}>x\}\le e^{-(x-c-d\log2)/d}\). Thus
\[
 \E[\pos{R_T^{d'}}]
 \le c+d\log2+\int_{c+d\log2}^{\infty}
       e^{-(x-c-d\log2)/d}\,dx
 =c+(1+\log2)d.
\]
Combining this with \(\Rfs\le\min\{U_0,\pos{R_T^{d'}}\}\) proves
\eqref{eq:exp3ix-parameter-bound}.
The \(\eta\)-dependent part of this bound is
\((2\log K+1+\log2)/\eta+\eta KT\), whose minimizing value is
\(\eta=\sqrt{(2\log K+1+\log2)/(KT)}\). It does not depend on \(U_0\),
which is why supplying \(L_0\) does not improve this EXP3-IX bound.
\end{proof}

\endgroup

\subsection{Upper bounds when \texorpdfstring{\(L_0\)}{L0} is unknown}

We now follow Section~\ref{sec:unknown-l0}. The optimal reward-doubling
method comes first, followed by the horizon-scale and fixed-action
baselines and then Propositions~\ref{prop:existing-bounds},
\ref{prop:tsallis-mous-positive}, and~\ref{prop:restarted-methods}.

\subsubsection{Reward-doubling Poly-INF}
\label{app:adaptive}

Restart times depend on the rewards collected by the algorithm, whereas
Lemma~\ref{lem:polyinf-main} concerns a fixed number of rounds. We
therefore use the deterministic inequality within the proof of
\citet[Appendix~C.9]{audibert2010regret} and control estimation error
across all epochs with one supermartingale. Keeping the same comparator
arm in every epoch makes the regrets add up to the full scalar regret.
When that regret is positive, the best arm's cumulative reward bounds
the largest epoch scale. This bound lets us choose one supermartingale
weight that is valid throughout those epochs, without taking a union
bound over restart times. We first prove a high-probability bound, then
integrate it to obtain Theorem~\ref{thm:adaptive}.

\begin{lemma}[Explicit high-probability form of Theorem~\ref{thm:adaptive}]
\label{lem:adaptive-explicit}
For \(K\ge2,T\ge2,D\ge1\) and \(\delta\in(0,1)\), put
\(A(\delta):=48+\sqrt2\log(\delta^{-1})\).  On every deterministic reward
sequence fixed before play, Algorithm~\ref{alg:adaptive} satisfies, with
probability at least \(1-\delta\),
\begin{equation}
 \Rfs\le\pos{R_T^{d'}}
 \le\min\Bigl\{U^\star,\;A(\delta)\bigl(\sqrt{KU_{d'}}+9K\bigr)\Bigr\},
 \label{eq:adaptive-hp}
\end{equation}
where \(d'\) is the coordinate fixed before play in
Algorithm~\ref{alg:adaptive} and \(U_{d'}\) is the best cumulative reward
in that coordinate.
\end{lemma}

Fix the coordinate \(d'\) of Algorithm~\ref{alg:adaptive} and write
\(g_{i,t}:=r_{d'}^{i,t}\in[0,1]\).  The reward sequence is fixed before
play, so the smallest index \(k^\star\) attaining
\(\max_{i\in[K]}\sum_{t=1}^Tg_{i,t}\) is a deterministic arm, with
\(U_{d'}=\sum_{t=1}^Tg_{k^\star,t}\le U^\star\), and by
\eqref{eq:coordinate-regret}
\(R_T^{d'}=\sum_{t=1}^T(g_{k^\star,t}-g_{a_t,t})\).  The started epochs
are \(e=0,\ldots,e_{\max}\) with \(B_e=81K\,2^e\).  Epoch \(e\) starts at
round \(t_e\), occupies the rounds
\(\mathcal T_e=\{t_e,\ldots,t_{e+1}-1\}\) with \(t_{e_{\max}+1}:=T+1\),
and has length \(n_e:=|\mathcal T_e|\).  Its \(\polyinf(B_e)\) copy uses
\((\eta_e,\gamma_e,\lambda_e)=(2\sqrt{B_e},\,3\sqrt{K/B_e},\,1/\sqrt{2KB_e})\).
If the threshold of epoch \(e_{\max}\) is crossed on round \(T\), no
further epoch is created, because the restart condition requires
\(t<T\).  Write \(e(t)\) for the epoch of round \(t\), \(p_t\) for the
arm distribution of the copy active on round \(t\), and \(v_{i,t}\) for
that copy's estimate, so that
\(v_{a_t,t}=-\lambda_{e(t)}^{-1}\log(1-\lambda_{e(t)}g_{a_t,t}/p_{a_t,t})\)
and \(v_{i,t}=0\) for \(i\ne a_t\).  The normalization of
Algorithm~\ref{alg:polyinf-full} gives \(p_{i,t}\ge\gamma_{e(t)}/K\),
hence \(\lambda_{e(t)}g_{i,t}/p_{i,t}\le\lambda_{e(t)}K/\gamma_{e(t)}=1/(3\sqrt2)\).

Let \(\mathcal F_t\) be the sigma-field generated by \(a_1,\ldots,a_t\),
with \(\mathcal F_0\) trivial.  Because the reward sequence is
deterministic, the epoch index \(e(t)\), the vector \(p_t\), and each
event \(\{t\in\mathcal T_e\}\) are \(\mathcal F_{t-1}\)-measurable, since
\(\{t\in\mathcal T_e\}=\{t_e\le t<t_{e+1}\}\) is decided by the rewards
collected before round \(t\).  Each copy draws its action from \(p_t\)
with randomness independent of the past, so
\(\Prob\{a_t=i\mid\mathcal F_{t-1}\}=p_{i,t}\) for every \(i\).  This is
the only property of the learner's randomization used below.

For every started epoch \(e\) put
\[
 Y_e:=\sum_{t\in\mathcal T_e}g_{a_t,t},\qquad
 G_e:=\sum_{t\in\mathcal T_e}g_{k^\star,t},\qquad
 V_i^{(e)}:=\sum_{t\in\mathcal T_e}v_{i,t},
\]
\[
 r_e:=G_e-Y_e,\qquad
 D_e:=G_e-V_{k^\star}^{(e)} .
\]
We call \(r_e\) the epoch regret against \(k^\star\) and \(D_e\) the epoch
deficit of the estimator of \(k^\star\).  Both may be negative.  Because
the epochs partition \([T]\) and the comparator is the same arm in every
epoch,
\begin{equation}
 R_T^{d'}=\sum_{e=0}^{e_{\max}}r_e,
 \qquad
 \sum_{e=0}^{e_{\max}}D_e
 =\sum_{t=1}^T\bigl(g_{k^\star,t}-v_{k^\star,t}\bigr).
 \label{eq:fixed-comparator}
\end{equation}

Every started epoch satisfies \(Y_e\le B_e+1\), because a crossing adds
one reward in \([0,1]\) to a total below \(B_e\), while an epoch that
reaches round \(T\) without crossing has \(Y_e<B_e\).  Every epoch
\(e<e_{\max}\) was completed, so \(Y_e\ge B_e\).  Doubling keeps the sum
of the square roots of the scales within a constant factor of the last
one,
\begin{equation}
\sum_{e=0}^{e_{\max}}\sqrt{B_e}
=\sqrt{B_{e_{\max}}}\sum_{s=0}^{e_{\max}}2^{-s/2}
\le(2+\sqrt2)\sqrt{B_{e_{\max}}}.
\label{eq:geometric-scales}
\end{equation}
The last inequality uses
\(\sum_{s=0}^{\infty}2^{-s/2}
=(1-2^{-1/2})^{-1}=2+\sqrt2\).
If \(R_T^{d'}>0\), then \(\sum_{e\le e_{\max}}Y_e<U_{d'}\), and since
rewards are nonnegative,
\begin{equation}
 B_{e_{\max}}-81K=\sum_{e<e_{\max}}B_e\le\sum_{e<e_{\max}}Y_e<U_{d'},
 \label{eq:banked-gain}
\end{equation}
so that \(2^{e_{\max}}<1+U_{d'}/(81K)\).  Define the deterministic index
and scale
\begin{equation}
 e^\star:=\Bigl\lfloor\log_2\Bigl(1+\frac{U_{d'}}{81K}\Bigr)\Bigr\rfloor,
 \qquad
 B_{e^\star}=81K\,2^{e^\star}\le U_{d'}+81K .
 \label{eq:high-water}
\end{equation}
Then \(e_{\max}\le e^\star\) on the event \(\{R_T^{d'}>0\}\), and
\(\sqrt{KB_{e^\star}}\le\sqrt{KU_{d'}}+9K\) by
\(\sqrt{a+b}\le\sqrt a+\sqrt b\) and \(\sqrt{81K^2}=9K\).

\begin{lemma}[Epoch inequality]
\label{lem:epoch-potential}
For every started epoch \(e\), on every play path,
\begin{equation}
 r_e\le D_e+14\sqrt{KB_e}.
 \label{eq:epoch-potential}
\end{equation}
\end{lemma}

\begin{proof}
Write \(s_e:=\sqrt{K/B_e}\le1/9\), so that \(\gamma_e=3s_e\),
\(\lambda_eK=s_e/\sqrt2\), and \(s_eB_e=\sqrt{KB_e}\). Define
\[
 x_0:=\frac1{3\sqrt2},\qquad
 c_{\mathrm{AB}}:=\frac{-\log(1-x_0)}{x_0},\qquad
 \nu:=\frac1{x_0}+\frac1{\log(1-x_0)},
\]
and, with \(y_e:=(K/B_e)^{1/4}\),
\[
 \mu_e:=\exp\!\left\{\sqrt3c_{\mathrm{AB}}y_e
 \left(1-\frac{c_{\mathrm{AB}}y_e}{4\sqrt3}\right)^{-2}\right\},
 \qquad
 \xi_e:=\left(\frac{c_{\mathrm{AB}}\mu_e(1+\mu_e)}{12}\right)^2.
\]
These are the constants in the pathwise Poly-INF inequality used below
after setting its exponent to \(2\).

The copy of epoch
\(e\) is \(\polyinf(B_e)\) started from \(V=0\) and run for \(n_e\) rounds

on the realized reward segment \((g_{i,t})_{t\in\mathcal T_e}\).
The segment starts and ends at random rounds. The

inequality below is applied on each play path with the realized value of
\(n_e\), which is legitimate because it is a deterministic statement for
every number of rounds \(n\ge1\); Theorem~2 of
\citet{audibert2010regret} is stated for every positive integer time. Its
estimates have the form \(v_{i,t}=(c_t/p_{i,t})\one\{i=a_t\}\) with
\[
 c_t=-\frac{p_{a_t,t}}{\lambda_e}
 \log\Bigl(1-\frac{\lambda_eg_{a_t,t}}{p_{a_t,t}}\Bigr)
 =g_{a_t,t}\cdot\frac{-\log(1-x_t)}{x_t},
 \qquad
 x_t:=\frac{\lambda_eg_{a_t,t}}{p_{a_t,t}}\in\Bigl[0,\frac1{3\sqrt2}\Bigr],
\]
and \(c_t=0\) when \(g_{a_t,t}=0\).  The map \(x\mapsto-\log(1-x)/x\) is
increasing on \((0,1)\) and \(g_{a_t,t}\le1\), so
\(0\le c_t\le c_{\mathrm{AB}}<1.1405\). Moreover,
\[
 2\eta_e\sqrt{\gamma_e/K}
 =4\sqrt3\,(B_e/K)^{1/4}\ge12\sqrt3>c_{\mathrm{AB}}.
\]
These are the hypotheses of
\citet[Theorem~2, Eq.~(8)]{audibert2010regret} with \(q=2\),
\(\eta=\eta_e\), \(\gamma=\gamma_e\), and \(c=c_{\mathrm{AB}}\).  That
statement concerns an arbitrary finite sequence of nonnegative reals
together with the probabilities that the \(\polyinf\) recursion computes
from it, holds on every play path, and imposes no lower bound on the
number of rounds.  It gives
\begin{equation}
 (1-\gamma_e)\max_iV_i^{(e)}
 -(1+\gamma_e\xi_e)\sum_{t\in\mathcal T_e}p_{a_t,t}v_{a_t,t}
 \le2\eta_e\sqrt K=4\sqrt{KB_e}.
 \label{eq:source-eight}
\end{equation}
This is Eq.~(39) in the source's Appendix~C.9 and is the only external
inequality used in this epoch proof.

We next derive a pathwise bound on the learner's mixed estimate. Fix
\(t\in\mathcal T_e\) and write \(p=p_{a_t,t}\),
\(g=g_{a_t,t}\), \(x=x_t\).  If \(g=0\) then \(v_{a_t,t}=0\) and
\eqref{eq:pathwise-forty} below is trivial, so let \(g>0\), hence
\(x>0\).  The map \(\phi(y)=1/y+1/\log(1-y)\) is increasing on \((0,1)\),
because \(\phi'(y)\ge0\) is equivalent, with \(u=-\log(1-y)>0\), to
\(2\sinh(u/2)\ge u\). Hence \(\phi(x)\le\phi(x_0)=\nu<0.5224\).
Multiplying
\(\phi(x)\le\nu\) by \(-g\log(1-x)\ge0\) gives
\(-(g/x)\log(1-x)\le g-\nu g\log(1-x)\).  Since
\(p\,v_{a_t,t}=-(p/\lambda_e)\log(1-x)=-(g/x)\log(1-x)\),
\(-\log(1-x)=\lambda_ev_{a_t,t}\), and \(g\le1\),
\begin{equation}
 \sum_{i=1}^Kp_{i,t}v_{i,t}
 =p\,v_{a_t,t}
 \le g_{a_t,t}+\nu\lambda_e\,g\,v_{a_t,t}
 \le g_{a_t,t}+\nu\lambda_e\sum_{i=1}^Kv_{i,t}.
 \label{eq:pathwise-forty}
\end{equation}
Summing
\eqref{eq:pathwise-forty} over \(t\in\mathcal T_e\), using
\(\sum_iV_i^{(e)}\le K\max_iV_i^{(e)}\), which holds because every
\(v_{i,t}\) is nonnegative, and inserting into \eqref{eq:source-eight}
gives
\begin{equation}
 a_e\max_iV_i^{(e)}\le(1+\gamma_e\xi_e)Y_e+4\sqrt{KB_e},
 \qquad
 a_e:=1-\gamma_e-(1+\gamma_e\xi_e)\nu\lambda_eK.
 \label{eq:source-fortyone}
\end{equation}

We now bound the constants. The functions defining \(\mu_e\) and \(\xi_e\)
are increasing in \(y_e\), and \(y_e\le1/3\) because \(B_e\ge81K\).
The preceding formulas give
\begin{align*}
 c_{\mathrm{AB}}
 &=\frac{-\log(1-1/(3\sqrt2))}{1/(3\sqrt2)}
   =1.140412\ldots<1.1405,\\
 \nu
 &=3\sqrt2+\frac1{\log(1-1/(3\sqrt2))}
   =0.522372\ldots<0.5224,\\
 \mu_e
 &\le\exp\!\left\{\frac{\sqrt3c_{\mathrm{AB}}}{3}
       \left(1-\frac{c_{\mathrm{AB}}}{12\sqrt3}\right)^{-2}\right\}
   =2.089831\ldots<2.090,\\
 \xi_e
 &\le\left(\frac{c_{\mathrm{AB}}(2.089831\ldots)
          (1+2.089831\ldots)}{12}\right)^2
   =0.376576\ldots<0.3766.
\end{align*}
The quantity \(a_e\) is decreasing in
\(\xi_e\), \(\nu\), and \(s_e\), so
\[
 a_e\ge1-\frac13-\Bigl(1+\frac{0.3766}3\Bigr)\frac{0.5224}{\sqrt2}\cdot\frac19
 >0.6204,
\]
and
\[
 1+\gamma_e\xi_e-a_e
 =s_e\Bigl[3(1+\xi_e)+(1+\gamma_e\xi_e)\frac{\nu}{\sqrt2}\Bigr]
 \le s_e\bigl[3(1.3766)+(1.1256)(0.3694)\bigr]
 \le4.546\,s_e .
\]
Since \(a_e>0\) and
\(\max_iV_i^{(e)}\ge V_{k^\star}^{(e)}=G_e-D_e=Y_e+r_e-D_e\),
inequality \eqref{eq:source-fortyone} gives
\(a_e(r_e-D_e)\le(1+\gamma_e\xi_e-a_e)Y_e+4\sqrt{KB_e}\), that is,
\[
 r_e\le D_e+\beta_eY_e+\frac4{a_e}\sqrt{KB_e},
 \qquad
 \beta_e:=\frac{1+\gamma_e\xi_e-a_e}{a_e}\le\frac{4.546}{0.6204}\,s_e\le7.33\,s_e .
\]
Insert \(Y_e\le B_e+1\), \(s_eB_e=\sqrt{KB_e}\), and \(s_e\le1/9\),
\[
 r_e\le D_e+\Bigl(7.33+\frac4{0.6204}\Bigr)\sqrt{KB_e}+\frac{7.33}9
 \le D_e+13.78\sqrt{KB_e}+0.815
 \le D_e+14\sqrt{KB_e},
\]
where the last step uses \(\sqrt{KB_e}\ge9K\ge18\), so that
\(0.815\le0.046\sqrt{KB_e}\).
\end{proof}

\begin{lemma}[Monotone exponential supermartingale]
\label{lem:monotone-supermartingale}
For every round \(t\in[T]\) and every deterministic \(\theta>0\), on the
\(\mathcal F_{t-1}\)-measurable event \(\{\theta\le\lambda_{e(t)}\}\),
\[
 \E\!\left[\exp\bigl(\theta(g_{k^\star,t}-v_{k^\star,t})\bigr)
 \,\middle|\,\mathcal F_{t-1}\right]\le1 .
\]
\end{lemma}

\begin{proof}
Write \(\lambda=\lambda_{e(t)}\), \(p=p_{k^\star,t}\), \(g=g_{k^\star,t}\),
\(x=\lambda g/p\in[0,1/(3\sqrt2)]\), and \(a=\theta/\lambda\), which lies in
\((0,1]\) on the event considered.
Given \(\mathcal F_{t-1}\), the estimate \(v_{k^\star,t}\) equals
\(-\lambda^{-1}\log(1-x)\) with probability \(p\) and \(0\) otherwise, so
\[
 \E\!\left[e^{-\theta v_{k^\star,t}}\,\middle|\,\mathcal F_{t-1}\right]
 =p(1-x)^a+(1-p)
 \le p(1-ax)+(1-p)
 =1-\theta g
 \le e^{-\theta g},
\]
where the first inequality is Bernoulli's inequality \((1-x)^a\le1-ax\)
for \(a\in(0,1]\) and \(x\in[0,1)\), the concavity of \(x\mapsto(1-x)^a\),
and the middle equality uses \(apx=\theta g\).  Multiplying by
\(e^{\theta g}\) proves the claim.
\end{proof}

Let
\begin{equation}
 \theta^\star:=\lambda_{e^\star}=\frac1{\sqrt{2KB_{e^\star}}},
 \qquad
 \tau:=\begin{cases}
 t_{e^\star+1}-1&\text{if epoch \(e^\star+1\) starts,}\\
 T&\text{otherwise,}
 \end{cases}
 \label{eq:estar}
\end{equation}
and \(M_0:=1\),
\begin{equation}
 M_t:=\exp\Bigl(\theta^\star\sum_{s=1}^{t\wedge\tau}
 \bigl(g_{k^\star,s}-v_{k^\star,s}\bigr)\Bigr),
 \qquad t=1,\ldots,T .
 \label{eq:single-supermartingale}
\end{equation}

\begin{lemma}[One maximal inequality]
\label{lem:ville-once}
\((M_t)_{0\le t\le T}\) is a nonnegative supermartingale for
\((\mathcal F_t)_{0\le t\le T}\) with \(M_0=1\).  Consequently, for every
\(\delta\in(0,1)\), the event
\begin{equation}
 \mathcal E_\delta:=\Bigl\{
 \sum_{s=1}^{t}\bigl(g_{k^\star,s}-v_{k^\star,s}\bigr)
 \le\sqrt{2KB_{e^\star}}\,\log(\delta^{-1})
 \ \text{for every }t\le\tau\Bigr\}
 \label{eq:ville-event}
\end{equation}
has probability at least \(1-\delta\).
\end{lemma}

\begin{proof}

For \(t<T\), the event \(\{\tau\le t\}\) occurs exactly when epoch
\(e^\star+1\) starts by round \(t+1\). This is determined by the
rewards collected through round \(t\), so the event belongs to
\(\mathcal F_t\). For \(t=T\), the event is the whole sample space.
Thus \(\tau\) is a stopping time for \((\mathcal F_t)\).  On
\(\{\tau\le t-1\}\in\mathcal F_{t-1}\) we have \(M_t=M_{t-1}\).  On
\(\{\tau\ge t\}\), round \(t\) lies in an epoch \(e(t)\le e^\star\), so \(\{\tau\ge t\}\) is contained in \(\{\theta^\star\le\lambda_{e(t)}\}\) and
Lemma~\ref{lem:monotone-supermartingale} with \(\theta=\theta^\star\)
gives \(\E[M_t\mid\mathcal F_{t-1}]\le M_{t-1}\).  Since \(v_{k^\star,s}\ge0\) and \(g_{k^\star,s}\le1\), we have
\(0\le M_t\le e^{\theta^\star T}\), so each \(M_t\) is integrable and the
supermartingale property holds.  Let
\(\sigma:=\min\{t\le T:M_t\ge\delta^{-1}\}\), with \(\sigma:=T\) if no
such \(t\) exists.  Then \(\sigma\) is a bounded stopping time, the
optional stopping theorem gives \(\E[M_\sigma]\le\E[M_0]=1\), and
\(\{\max_{0\le t\le T}M_t\ge\delta^{-1}\}=\{M_\sigma\ge\delta^{-1}\}\), so
Markov's inequality gives
\[
 \Prob\Bigl\{\max_{0\le t\le T}M_t\ge\delta^{-1}\Bigr\}\le\delta ,
\]
which is Ville's maximal inequality for nonnegative supermartingales.  On
the complement,
\(\theta^\star\sum_{s\le t}(g_{k^\star,s}-v_{k^\star,s})<\log(\delta^{-1})\)
for every \(t\le\tau\), which is \eqref{eq:ville-event}.
\end{proof}

The event \(\mathcal E_\delta\) controls the total estimation error
across the epochs up to \(e^\star\). The next proof uses it at
\(t=T\) when \(\tau=T\). There is no union bound over epochs or
prefix lengths, and no continuation of an epoch beyond its end.

\begin{proof}[Proof of Lemma~\ref{lem:adaptive-explicit}]
Fix \(\delta\in(0,1)\) and work on the event \(\mathcal E_\delta\) of
Lemma~\ref{lem:ville-once}, which has probability at least \(1-\delta\).
On every play path \(\pos{R_T^{d'}}\le U_{d'}\le U^\star\), because the
learner's cumulative reward in coordinate \(d'\) is nonnegative, and
\(\Rfs\le\pos{R_T^{d'}}\) by \eqref{eq:fixed-coordinate}.  It remains to
bound \(\pos{R_T^{d'}}\) on \(\mathcal E_\delta\).  If \(R_T^{d'}\le0\)
there is nothing to prove, so suppose \(R_T^{d'}>0\).  Then
\(e_{\max}\le e^\star\) by \eqref{eq:high-water}, epoch \(e^\star+1\)
never starts, and \(\tau=T\).  Lemma~\ref{lem:epoch-potential} applies to
every started epoch, so by \eqref{eq:fixed-comparator},
\eqref{eq:geometric-scales}, and \(B_{e_{\max}}\le B_{e^\star}\),
\begin{align}
 R_T^{d'}
 =\sum_{e=0}^{e_{\max}}r_e
 &\le\sum_{t=1}^T\bigl(g_{k^\star,t}-v_{k^\star,t}\bigr)
 +14\sum_{e=0}^{e_{\max}}\sqrt{KB_e}\nonumber\\
 &\le\sqrt{2KB_{e^\star}}\,\log(\delta^{-1})
 +14(2+\sqrt2)\sqrt{KB_{e^\star}}
 \le\bigl(47.8+\sqrt2\log(\delta^{-1})\bigr)\sqrt{KB_{e^\star}},
 \label{eq:assembled}
\end{align}
where the comparator sum is the sum in \eqref{eq:ville-event} at
\(t=T=\tau\), and \(14(2+\sqrt2)<47.8\).  With
\(\sqrt{KB_{e^\star}}\le\sqrt{KU_{d'}}+9K\) from \eqref{eq:high-water},
on \(\mathcal E_\delta\),
\begin{equation}
 \pos{R_T^{d'}}
 \le\bigl(47.8+\sqrt2\log(\delta^{-1})\bigr)\bigl(\sqrt{KU_{d'}}+9K\bigr)
 \le A(\delta)\bigl(\sqrt{KU_{d'}}+9K\bigr),
 \label{eq:hp-positive-part}
\end{equation}
which proves \eqref{eq:adaptive-hp}.
\end{proof}

\begin{proof}[Proof of Theorem~\ref{thm:adaptive}]
Put \(a:=47.8\sqrt{KB_{e^\star}}\) and \(b:=\sqrt{2KB_{e^\star}}\), both
deterministic. On \(\mathcal E_\delta\), a nonpositive
\(R_T^{d'}\) has zero positive part, while a positive \(R_T^{d'}\)
satisfies \eqref{eq:assembled}. Thus, for \(x>a\), the choice
\(\delta:=\exp(-(x-a)/b)\in(0,1)\) gives
\(\pos{R_T^{d'}}\le x\) on \(\mathcal E_\delta\), and hence
\(\Prob\{\pos{R_T^{d'}}>x\}\le\Prob\{\mathcal E_\delta^c\}\le e^{-(x-a)/b}\).
The layer-cake identity gives
\begin{align*}
 \E\!\left[\pos{R_T^{d'}}\right]
 &=\int_0^\infty\Prob\{\pos{R_T^{d'}}>x\}\,dx
 \le a+\int_a^\infty e^{-(x-a)/b}\,dx
 =a+b\\
 &=(47.8+\sqrt2)\sqrt{KB_{e^\star}}
 <50\bigl(\sqrt{KU_{d'}}+9K\bigr).
\end{align*}
Together with \(\pos{R_T^{d'}}\le U^\star\) on every play path and
\(\Rfs\le\pos{R_T^{d'}}\), this gives
\[
 \E[\Rfs]\le\min\{U^\star,50(\sqrt{KU_{d'}}+9K)\}.
\]
To obtain \eqref{eq:adaptive-expected}, use \(U_{d'}\le U^\star\). If
\(U^\star\le81K\), the minimum is at most \(U^\star\), and
\(U^\star\le9\min\{U^\star,\sqrt{KU^\star}\}\), because \(U^\star\) is the minimum when
\(U^\star\le K\), while \(U^\star=\sqrt{U^\star}\sqrt{U^\star}\le9\sqrt K\sqrt{U^\star}\) when
\(K<U^\star\le81K\).  If \(U^\star>81K\), then \(9K<\sqrt{KU^\star}\), so
\(50(\sqrt{KU_{d'}}+9K)<100\sqrt{KU^\star}=100\min\{U^\star,\sqrt{KU^\star}\}\).  On a
reward sequence with \(L^\star=L_0\) we have \(U^\star=U_0=T-L_0\) by
\eqref{eq:LstarU}, which gives the last claim.
\end{proof}

The absence of logarithmic factors in \(K\) or \(T\) comes from
controlling all estimation errors together. The index \(e^\star\)
is deterministic, and its weight \(\lambda_{e^\star}\) is valid in
every earlier epoch. Thus the proof needs neither a union over arms
nor a union over restart times.

\begin{proof}[Proof of the horizon-scale Poly-INF comparison]
The two branches of Algorithm~\ref{alg:known-scale} with \(B=T\) are
handled separately. If \(T\ge81K\), the best cumulative reward in
coordinate \(d'\) is at most \(T=B\), so
\eqref{eq:fixed-coordinate} and \eqref{eq:polyinf-positive} give
\(\E[\Rfs]<10\sqrt{KT}=10\WKT\), where \(\WKT=\sqrt{KT}\) because
\(T>K\). If \(T<81K\), the algorithm plays arm 1 throughout, and the
pathwise cap \(\Rfs\le U^\star\le T\) is at most \(9\WKT\). Indeed,
\(\WKT=T\) when \(T\le K\), while for \(K<T<81K\),
\(T/\WKT=\sqrt{T/K}<9\).
\end{proof}

\begingroup
\subsubsection{Uniform and fixed-arm play}
\label{app:uniform-fixed}

\begin{proof}[Simple-baseline bounds]
Fix a coordinate and a deterministic best arm \(k\) for that coordinate,
with total reward \(G^\star\le U_0\). Since every reward is nonnegative,
the algorithm's collected reward is at least the reward it collects
on rounds when it plays \(k\). Hence, pathwise,
\[
 \Rfs\le\pos{G^\star-A}
 \le G^\star-\sum_t g_{k,t}\one\{a_t=k\}.
\]
Uniform play selects \(k\) with probability \(1/K\) on every round,
so its expected regret is at most \((1-1/K)G^\star\le(1-1/K)U_0\).
The pathwise bound \(\Rfs\le U_0\) also covers specified fixed-arm
play, including the endpoint \(U_0=0\).
\end{proof}

\endgroup

\begingroup
\subsubsection{EXP3.P, MO-KS, EXP3-IX, and GREEN-IX}
\label{app:high-probability-methods}
\label{app:unknown-comparisons}

\begin{proof}[High-probability part of
Proposition~\ref{prop:existing-bounds}]
First consider EXP3.P.  Set \(\delta_0=(T+1)^{-2}\), and use
\[
 \gamma=\min\!\left\{\frac35,
             2\sqrt{\frac{3K\log K}{5T}}\right\},
 \qquad \alpha=2\sqrt{\log(KT/\delta_0)}.
\]
Theorem~6.3 of \citet{auer2002nonstochastic} bounds the realized scalar
regret by
\[
 b_{K,T}=4\sqrt{KT\log(KT/\delta_0)}
       +4\sqrt{\frac53KT\log K}+8\log(KT/\delta_0)
\]
with probability at least \(1-\delta_0\).  On that event,
\(\Rfs\le b_{K,T}\); on its complement, \(\Rfs\le U_0\le T\).
Thus
\[
 \E[\Rfs]\le(1-\delta_0)b_{K,T}+\delta_0U_0
 \le b_{K,T}+T\delta_0.
\]
Combining this inequality with the pathwise bound \(\Rfs\le U_0\) gives
\[
\begin{gathered}
 \E[\Rfs]\le\min\{U_0,b_{K,T}+T\delta_0\}
 \lesssim\min\{U_0,\sqrt{KT\log(KT)}\}.
\end{gathered}
\]
Here \(q=\log(KT(T+1)^2)\le KT\) and
\(q\le3\log(KT)\) for \(K,T\ge2\), so the additive \(q\)
term is absorbed by \(\sqrt{KTq}\).
The parameter \(\delta_0\) is fixed before play, so this argument does
not integrate bounds for different algorithms.  The adversarial branch
of MO-KS is this scalar EXP3.P procedure on a fixed coordinate
\citep[Algorithm~1]{xu2023pareto}; choosing its parameters as above gives
the same conclusion.  Randomizing the coordinate before play also
preserves this bound by conditioning on that choice.

For EXP3-IX, set \(\eta=2\gamma=\sqrt{2\log K/(KT)}\), use the loss
estimate
\[
 \widetilde\ell_{i,t}
 =\frac{\ell_{i,t}\one\{a_t=i\}}{p_{i,t}+\gamma},
 \qquad
 p_{i,t}\propto\exp\!\left(-\eta\sum_{s<t}
                                     \widetilde\ell_{i,s}\right).
\]
Theorem~1 of \citet{neu2015explore} gives, for the same algorithm at every
\(\delta\in(0,1)\),
\[
 \Prob\{R_T^{d'}>a+b\log(2/\delta)\}\le\delta,
 \quad
 a=2\sqrt{2KT\log K},\qquad
 b=\sqrt{\frac{2KT}{\log K}}+1.
\]
Put \(y=a+b\log2\). For \(x>y\), choose
\(\delta=2e^{-(x-a)/b}\in(0,1)\). The preceding probability bound then
gives \(\Prob\{R_T^{d'}>x\}\le e^{-(x-y)/b}\). Therefore
\[
 \E[\pos{R_T^{d'}}]
 \le y+\int_y^\infty e^{-(x-y)/b}\,dx
 =a+(1+\log2)b.
\]
For \(K,T\ge2\) this is \(O(\sqrt{KT\log K})\). Combining it with
the bound by \(U_0\) proves the EXP3-IX bound in
Proposition~\ref{prop:existing-bounds}.

For GREEN-IX, apply Corollary~4.2 of \citet{lykouris2018small} to the
losses in coordinate \(d'\), with \(\delta_0=(T+1)^{-2}\).  Its number
of arms, denoted \(d\) in that source, is our \(K\).  The adaptive
loss-doubling procedure gives, with probability at least \(1-\delta_0\),
\[
\begin{gathered}
R_T^{d'}\lesssim
 \sqrt{K L_{d'}\log(KT)}+K\log(KT),
\end{gathered}
\]
where
\(\log(K/\delta_0)=\log(K(T+1)^2)\le3\log(KT)\) for
\(K,T\ge2\). The loss-doubling overhead in Lemma~3.8 of the same source
enters through \(\log(1+\log(L_{d'}+1))\), which is also absorbed into
\(O(\log(KT))\) since \(L_{d'}\le T\).
The procedure does not need \(L_{d'}\) or \(L_0\) as an input. Let
\(b_G\) denote the displayed right side, including the numerical constant
in the cited result. On the event, which has probability at least
\(1-\delta_0\),
\(\Rfs\le\pos{R_T^{d'}}\le b_G\); on the complement, \(\Rfs\le U_0\).
Hence
\[
 \E[\Rfs]\le\min\{U_0,b_G+U_0\delta_0\}.
\]
Since \(U_0\delta_0\le T/(T+1)^2<1\), this proves the GREEN-IX bound
in Proposition~\ref{prop:existing-bounds}. Taking a positive part
introduces no extra loss-level dependence.
\end{proof}

\endgroup

\begingroup
\subsubsection{Gain-based EXP3}

The calculation in Appendix~\ref{app:original-exp3} proves
\eqref{eq:gain-exp3-parameter} for every supplied \(V\ge U_0\). Choosing
\(V=T\) gives
\[
\begin{gathered}
\E[\Rfs]\lesssim \min\{U_0,(K\log K)^{1/3}T^{2/3}\}
\end{gathered},
\]
which is the gain-based EXP3 statement in
Proposition~\ref{prop:existing-bounds}.

\subsubsection{Original EXP3.1}
\label{app:original-exp31}

EXP3.1 \citep{auer2002nonstochastic} restarts the gain-based EXP3
weights, but retains its cumulative reward estimates across epochs.
For completeness, set
\[
 c=\frac{K\log K}{e-1},\qquad b_r=c4^r,\qquad\gamma_r=2^{-r}.
\]
Epoch \(r\) starts fresh EXP3 weights with parameter \(\gamma_r\), but
keeps the cumulative importance-weighted reward estimates
\(\widehat G_i\) from all earlier epochs, initially zero at \(r=0\).
Before every action, including the first action of an epoch, check whether
\(\max_i\widehat G_i\le b_r-K/\gamma_r\). If this fails, increment \(r\),
reset the weights, and repeat the check without taking an action or
changing the retained estimates. Otherwise, take one EXP3 action and add
its importance-weighted reward to \(\widehat G_i\). This is the pre-action
check in \citet[Figure~2]{auer2002nonstochastic}; it permits empty epochs.
The retained estimates, rather than collected rewards, determine the
restart times.

\begin{proof}[Proof of the EXP3.1 part of Proposition~\ref{prop:original-gain-bounds}]
Write \(\ell=\log K\), \(c=K\ell/(e-1)\). All the inequalities
in the first part of the proof hold on every realized history and at
every prefix of that history. Let \(r\) be the epoch of the last
completed action, \(x=\max_i\widehat G_i\), and \(A\) the collected
reward at that prefix.

Within a nonempty epoch \(s\), every action begins with
\(\widehat G_i\le b_s-K/\gamma_s\). Since its estimate increment is at most
\(K/\gamma_s\), every global estimate immediately after that action is
at most \(b_s\), including at the end of the epoch or any prefix of it.
Nonnegativity gives the same bound for the epoch increment, denoted
\(\widehat G_{i,s}\).
The EXP3 potential argument above, with
\(p_{i,t}\widehat g_{i,t}^2\le\widehat g_{i,t}\), gives for every arm
\(j\)
\[
 A_s\ge\widehat G_{j,s}
 -(e-1)\gamma_s b_s-K\ell/\gamma_s
 =\widehat G_{j,s}-2K\ell\,2^s.
\]
Empty epochs satisfy this inequality as well. This reproduces the
epoch inequality of \citet[Lemma~4.2]{auer2002nonstochastic} and also
shows explicitly why it holds for a truncated last epoch. Summing over
epochs and maximizing over \(j\) yields
\begin{equation}
 A\ge x-4K\ell\,2^r.
 \label{eq:exp31-prefix-potential}
\end{equation}
If \(r\ge1\), leaving epoch \(r-1\) required
\(x>c4^{r-1}-K2^{r-1}\). Put \(y=2^r\). This inequality is
\(cy^2-2Ky-4x<0\), so \(y\) is smaller than its positive root:
\[
 y<\frac{K+\sqrt{K^2+4cx}}{c}
 \le\frac{2K}{c}+2\sqrt{\frac xc},
\]
where the last inequality uses
\(\sqrt{K^2+4cx}\le K+2\sqrt{cx}\). When \(r=0\), \(2^r=1\).
Combining the two cases gives
\begin{equation}
 2^r\le1+2K/c+2\sqrt{x/c}.
 \label{eq:exp31-prefix-scale}
\end{equation}
Consequently
\begin{equation}
 A\ge x-a\sqrt x-b,\qquad
 a=8\sqrt{(e-1)K\ell},\quad b=4K\ell+8(e-1)K.
 \label{eq:exp31-prefix-lower}
\end{equation}
Indeed, substituting \eqref{eq:exp31-prefix-scale} into
\eqref{eq:exp31-prefix-potential} gives
\[
 A\ge x-4K\ell-\frac{8K^2\ell}{c}
          -\frac{8K\ell}{\sqrt c}\sqrt x,
\]
and \(c=K\ell/(e-1)\) gives the displayed values of \(a\) and \(b\).
For \(0\le y\le x\),
\(x-a\sqrt x\ge y-a\sqrt y-a^2/4\). To see this, write \(x=y+z\)
with \(z\ge0\), use \(\sqrt{y+z}\le\sqrt y+\sqrt z\), and note that
\[
 z-a\sqrt z=(\sqrt z-a/2)^2-a^2/4\ge-a^2/4.
\]
Thus \eqref{eq:exp31-prefix-lower} implies, for
each fixed comparator \(k\),
\begin{equation}
 A\ge\widehat G_k-a\sqrt{\widehat G_k}-b-a^2/4.
 \label{eq:exp31-fixed-comparator}
\end{equation}

We next bound the estimation error without assuming that the algorithm
knows \(U_0\). Fix a deterministic best arm \(k\), and let \(\tau\)
be the first round at which collected reward reaches \(G^\star\),
or \(T\) if this never occurs. If \(G^\star=0\), nonnegative rewards
imply that every arm receives zero reward on every round, so
\(\Rfs=0\). We therefore assume \(G^\star>0\).
This stopping time is used only in the proof. Nonnegative rewards imply
\[
 \pos{G^\star-A_T}
 =\pos{\sum_{t\le\tau}g_{k,t}-A_\tau},\qquad
 A_\tau\le U_0+1.
\]
Indeed, if the threshold is reached, then
\(A_\tau\ge G^\star\) and \(A_T\ge A_\tau\), so both positive parts are
zero. If it is not reached, then \(\tau=T\), and the two expressions are
identical. Moreover, immediately before a crossing the collected reward is
smaller than \(G^\star\le U_0\), and the crossing reward is at most one;
hence \(A_\tau\le U_0+1\). From
\eqref{eq:exp31-prefix-lower} and
\(a\sqrt x\le(x+a^2)/2\), at this stopped prefix
\[
 x\le H:=2(U_0+1)+a^2+2b.
\]
Epoch indices only increase. Hence every action up to \(\tau\)
occurs in an epoch \(u\le r\), and its sampling probability satisfies
\(p_{k,t}\ge\gamma_u/K=1/(K2^u)\ge1/(K2^r)\). By
\eqref{eq:exp31-prefix-scale}, \(2^r\le s\), where the deterministic
quantity \(s\) is
\[
 s=1+2K/c+2\sqrt{H/c}.
\]
The stopped estimation error
\(M_\tau=\sum_{t\le\tau}(g_{k,t}-\widehat g_{k,t})\) has mean zero.
Its increments include the predictable indicator \(\one\{t\le\tau\}\),
so conditional unbiasedness and orthogonality give
\[
 \E[M_\tau^2]\le Ks\sum_tg_{k,t}^2\le KsU_0,
 \qquad
 \E[\widehat G_{k,\tau}]=\E\!\left[\sum_{t\le\tau}g_{k,t}\right]
 \le U_0.
\]
At the stopped prefix, \eqref{eq:exp31-fixed-comparator} gives
\[
 \sum_{t\le\tau}g_{k,t}-A_\tau
 \le M_\tau+a\sqrt{\widehat G_{k,\tau}}+b+a^2/4.
\]
Taking positive parts and expectations, using
\(\E[\pos{M_\tau}]\le\tfrac12\sqrt{\E[M_\tau^2]}\), and applying Jensen's
inequality to \(\sqrt{\widehat G_{k,\tau}}\) proves
\begin{equation}
 \E[\Rfs]\le
 \min\{U_0,a\sqrt{U_0}+b+a^2/4+\tfrac12\sqrt{KsU_0}\}.
 \label{eq:exp31-explicit}
\end{equation}
For \(K\ge2\), \(\ell=\log K\ge\log2\),
\(a^2=64(e-1)K\ell\), and \(b=4K\ell+8(e-1)K=O(K\ell)\).
Consequently,
\[
 H=O(U_0+K\ell),\qquad
 s=1+\frac{2(e-1)}{\ell}+2\sqrt{\frac{H(e-1)}{K\ell}}
 =O\!\left(1+\sqrt{\frac{U_0}{K\ell}}\right).
\]
Since \(s=O(1+\sqrt{U_0/(K\ell)})\), taking square roots gives
\(\sqrt{s}=O(1+(U_0/(K\ell))^{1/4})\). Therefore
\[
 \sqrt{KsU_0}
 =O\!\left(
   \sqrt{KU_0}
   +
   \left(\frac{KU_0^3}{\ell}\right)^{1/4}\right).
\]
The term \(a\sqrt{U_0}\) is
\(O(\sqrt{KU_0\ell})\), while \(b+a^2/4=O(K\ell)\).
Since \(\sqrt{KU_0}\le(\log2)^{-1/2}\sqrt{KU_0\ell}\), substitution in
\eqref{eq:exp31-explicit} gives
\[
\begin{gathered}
\E[\Rfs]\lesssim \min\left\{U_0,\sqrt{KU_0\log K}
 +\left(\frac{KU_0^3}{\log K}\right)^{1/4}+K\log K\right\},
\end{gathered}
\]
which is the EXP3.1 bound in
Proposition~\ref{prop:original-gain-bounds}. This proof concerns the original
EXP3.1 procedure, not reward-doubling Poly-INF and not a modified EXP3
estimator.
\end{proof}

\subsubsection{Tsallis-INF, EXP3++, and MO-US}
\label{app:tsallis-exp3pp}

\begin{proof}[Proof of Proposition~\ref{prop:tsallis-mous-positive}]
Use Tsallis-INF with parameter \(\alpha=1/2\), symmetric regularization,
and the estimators in \citet[Theorem~1]{zimmert2021tsallis}. Its learning
rates are \(2/\sqrt t\) for importance-weighted (IW) estimates and
\(4/\sqrt t\) for reduced-variance (RV) estimates. For deterministic
losses the cited pseudo-regret bounds are signed expected regret bounds.
Write \(\overline B_{\mathcal A}\) for the following explicit bounds. For Tsallis-INF,
\begin{align*}
 \overline B_{\mathrm{IW}}(n)&=4\sqrt{Kn}+1,\\
 \overline B_{\mathrm{RV}}(n)&=2\sqrt{Kn}+10K\log n+16.
\end{align*}
Lemma~\ref{lem:signed-positive} proves the stated positive-regret
bounds for these unchanged algorithms.

For EXP3++, \citet[Theorem~1]{seldin2017improved} gives
\(\overline B_{++}(n)=4\sqrt{Kn\log K}\). The assumptions on its exploration
parameters are important. Set
\[
 \eta_t=\tfrac12\sqrt{\log K/(Kt)},\quad
 q_{i,t}\propto e^{-\eta_t\widehat L_{i,t-1}},\quad
 p_{i,t}=(1-\epsilon_t)q_{i,t}+\epsilon_{i,t},
 \quad \epsilon_t=\sum_i\epsilon_{i,t},
\]
where \(0\le\epsilon_{i,t}\le\min\{1/(2K),\eta_t\}\) is determined
from past observations. Its loss estimate is
\(\widehat\ell_{i,t}=\ell_{i,t}\one\{a_t=i\}/p_{i,t}\).
The adversarial argument does not require the empirical gaps used to
choose these exploration probabilities to be accurate.

For completeness, this argument also covers the initialization used by
MO-US \citep[Algorithm~2]{xu2023pareto}. Charge at most \(K\) regret
for its initial pulls (or \(T\) if \(T<K\)). Its initial cumulative
estimates lie in \([0,1]\), so their range costs at most one additional
unit in the exponential-weights comparison. We now derive the needed
inequality on the remaining rounds.

Let \(H_{i,0}\in[0,1]\) be the estimate after initialization, let
\(H_{i,t}=H_{i,0}+\sum_{s=1}^t\widehat\ell_{i,s}\), and define
\[
 F_t:=-\frac1{\eta_t}
 \log\!\left(\frac1K\sum_i e^{-\eta_tH_{i,t-1}}\right).
\]
We set \(\eta_{T+1}=\eta_T\) when defining the potential after the final
update. Conditional on the initialization, \(H_{i,0}\) is fixed.
At a fixed value of \(\eta_t\), put
\[
 u_t:=\eta_t\langle q_t,\widehat\ell_t\rangle
 -\frac{\eta_t^2}{2}\langle q_t,\widehat\ell_t^2\rangle.
\]
The inequality \(e^{-x}\le1-x+x^2/2\) gives
\(\sum_iq_{i,t}e^{-\eta_t\widehat\ell_{i,t}}\le1-u_t\).
The sum on the left is positive, so \(u_t<1\), and
\(-\log(1-u_t)\ge u_t\). Consequently,
\[
 -\frac1{\eta_t}\log\sum_iq_{i,t}e^{-\eta_t\widehat\ell_{i,t}}
 \ge \langle q_t,\widehat\ell_t\rangle
 -\frac{\eta_t}{2}\langle q_t,\widehat\ell_t^2\rangle.
\]
For a fixed nonnegative vector \(H\), the function
\[
 \eta\longmapsto
 -\frac1\eta\log\!\left(\frac1K\sum_i e^{-\eta H_i}\right)
\]
is nonincreasing. One way to see this is to write the derivative as
\(-\operatorname{KL}(q_\eta\Vert\operatorname{Unif}[K])/\eta^2\le0\),
where \(q_{\eta,i}\propto e^{-\eta H_i}\). Since
\(\eta_{t+1}\le\eta_t\), replacing \(\eta_t\) by \(\eta_{t+1}\) after
the update can only increase \(F_{t+1}\). Therefore
\[
 \langle q_t,\widehat\ell_t\rangle
 \le F_{t+1}-F_t
 +\frac{\eta_t}{2}\langle q_t,\widehat\ell_t^2\rangle.
\]
Summing this inequality telescopes. Moreover,
\[
 F_{T+1}\le H_{k,T}+\frac{\log K}{\eta_T},
 \qquad
 F_1\ge\min_iH_{i,0},
\]
so \(H_{k,0}-F_1\le1\). Subtracting
\(\sum_t\widehat\ell_{k,t}=H_{k,T}-H_{k,0}\) gives
\[
 \sum_t\langle q_t,\widehat\ell_t\rangle-\sum_t\widehat\ell_{k,t}
 \le1+\frac{\log K}{\eta_T}
 +\frac12\sum_t\eta_t\sum_iq_{i,t}\widehat\ell_{i,t}^{\,2}.
\]

Since \(\epsilon_t\le1/2\), \(p_{i,t}\ge q_{i,t}/2\). Thus the
conditional expectation of the quadratic term is
\[
 \E_t\!\left[\sum_iq_{i,t}\widehat\ell_{i,t}^{\,2}\right]
 =\sum_i\frac{q_{i,t}\ell_{i,t}^2}{p_{i,t}}
 \le2K.
\]
Moreover
\(\E_t[\ell_{a_t,t}-\langle q_t,\widehat\ell_t\rangle]
=\langle p_t-q_t,\ell_t\rangle\le\epsilon_t\le K\eta_t\).
The first inequality follows from
\(p_{i,t}-q_{i,t}=\epsilon_{i,t}-\epsilon_tq_{i,t}\) and
\(\ell_{i,t}\in[0,1]\). Unbiasedness gives
\(\E_t[\widehat\ell_{k,t}]=\ell_{k,t}\). Finally,
\(\sum_{t\le T}t^{-1/2}\le2\sqrt T\), and hence
\[
 \sum_{t\le T}\eta_t
 \le\sqrt{\frac{T\log K}{K}}.
\]
Taking expectations in the potential inequality and adding the sampling
difference therefore gives
\[
 \E[R_T^{d'}]\le\overline B_{\mathrm{MO}}(T)
 :=4\sqrt{KT\log K}+K+1.
\]
This also covers any of MO-US's empirical exploration choices satisfying
the displayed bounds. A zero denominator in the exploration expression
is interpreted as \(+\infty\) before taking its minimum, so the
probabilities remain well defined.
Apply Lemma~\ref{lem:signed-positive} once more. Substituting the four
explicit bounds gives the corresponding statements in
Proposition~\ref{prop:tsallis-mous-positive} for
\(K\ge2,n\ge1\). This argument does not replace
\(L_{d'}\) by \(L_0\).
\end{proof}

\subsubsection{Independent restarts in fixed blocks}
\label{app:block-restarts}

\begin{lemma}[Positive regret after independent block restarts]
\label{lem:block-positive}
Suppose \(\mathcal A\) has signed expected regret at most
\(B_{\mathcal A}(n)\ge0\) on every deterministic scalar sequence of
length \(n\). For the block lengths \(n_1,\ldots,n_m\) in
Algorithm~\ref{alg:block-restarts},
\begin{equation}
 \E[\Rfs]\le\min\left\{U_0,
 \sum_{j=1}^m B_{\mathcal A}(n_j)
 +\frac12\sqrt{\min\left\{\frac14\sum_jn_j^2,hKU_0\right\}}
 \right\}.
 \label{eq:block-positive-exact}
\end{equation}
\end{lemma}
\begin{proof}
Fix a best scalar arm \(k\) for the entire sequence. Let \(A_j\)
be the collected reward in block \(j\), \(G_{k,j}\) that arm's
reward there, and \(X_j=G_{k,j}-A_j\). Each \(\E[X_j]\) is at most
\(B_{\mathcal A}(n_j)\), since a block's best arm earns at least
\(G_{k,j}\). The blocks and reward sequence are fixed, and the copies
use independent random choices without shared state. Hence the
\(X_j\) are independent, even though actions within a block are adaptive.
For \(X=\sum_jX_j\), put \(Y=X-\E[X]\). Since \(Y\) has mean zero,
\(\E[\pos{Y}]=\E[|Y|]/2\le\sqrt{\operatorname{Var}(X)}/2\). Also
\(\operatorname{Var}(X)=\sum_j\operatorname{Var}(X_j)
=\sum_j\operatorname{Var}(A_j)\), because the \(G_{k,j}\) are
deterministic. Therefore
\[
 \E[\pos{X}]\le\pos{\E[X]}+\E[\pos{X-\E[X]}]
 \le\sum_j B_{\mathcal A}(n_j)+\tfrac12\sqrt{\sum_j\operatorname{Var}(A_j)}.
\]
Since \(A_j\in[0,n_j]\),
\(\operatorname{Var}(A_j)\le n_j^2/4\). Also
\(A_j^2\le hA_j\), and pathwise
\(\sum_j A_j\le\sum_{i,t}g_{i,t}\le KU_0\). Therefore
\[
 \sum_j\operatorname{Var}(A_j)
 \le\sum_j\E[A_j^2]
 \le h\,\E\!\left[\sum_jA_j\right]
 \le hKU_0.
\]
Combining the two variance bounds and
\(\Rfs\le\min\{U_0,\pos X\}\) proves the result.
\end{proof}

\begin{proof}[Proof of Proposition~\ref{prop:restarted-methods}]
Use the four explicit bounds \(\overline B_{\mathcal A}\) established in
Section~\ref{app:tsallis-exp3pp} as \(B_{\mathcal A}\) in
Lemma~\ref{lem:block-positive}. Let
\(h=\min\{T,\lceil\sqrt{KT}\rceil\}\). If \(K\ge T\),
\(K^{1/4}T^{3/4}\ge T\), and the bound \(U_0\) suffices.
For \(K<T\), there are at most \(T/h+1\) blocks, and
\[
 \sum_j\sqrt{Kn_j}\le\sqrt K(T/\sqrt h+\sqrt h),\qquad
 \sum_j\operatorname{Var}(A_j)\le hT/4.
\]
Here
\[
 \sqrt{KT}\le h\le\sqrt{KT}+1\le2\sqrt{KT}.
\]
The lower bound on \(h\) gives
\[
 \sqrt K\,\frac{T}{\sqrt h}\le K^{1/4}T^{3/4}.
\]
The upper bound on \(h\), together with \(K<T\), gives
\[
 \sqrt{Kh}\le\sqrt2\,K^{3/4}T^{1/4}
 \le\sqrt2\,K^{1/4}T^{3/4},
 \qquad
 \sqrt{hT}\le\sqrt2\,K^{1/4}T^{3/4}.
\]
Thus the square-root terms in
\eqref{eq:block-positive-exact} are
\(O(K^{1/4}T^{3/4})\).

The number of blocks is at most
\[
 \frac Th+1\le\sqrt{\frac TK}+1\le2\sqrt{\frac TK}.
\]
For IW, the additive constant in each block therefore contributes at most
\(O(\sqrt{T/K})\), which is
\(O(K^{1/4}T^{3/4})\). For RV and MO-US, the non-square-root terms sum to
\[
 O\!\left(K\log(eT)\left(\frac Th+1\right)\right)
 =O(\sqrt{KT}\log(eT))
 =O(K^{1/4}T^{3/4}\log(2T)).
\]
For EXP3++, multiplying the square-root calculation by
\(\sqrt{\log K}\) gives
\(O(K^{1/4}T^{3/4}\sqrt{\log K})\). Substitution into
\eqref{eq:block-positive-exact} proves all four bounds in
Proposition~\ref{prop:restarted-methods}. These statements concern the
independently restarted copies in Algorithm~\ref{alg:block-restarts}.
\end{proof}

\endgroup

\begingroup
For a fixed weighted average, replace each scalar reward by
\(g_{i,t}=\sum_dw_dr_d^{i,t}\). The same scalar proofs apply, with best
reward at most \(U_0\). For GREEN-IX the relevant loss is
\(\min_i\sum_t(1-g_{i,t})\), which need not equal \(L^\star\).
Thus averaging is valid for the algorithms with a positive-regret
guarantee. This is not a consequence of a signed expected-regret bound.
If all coordinates coincide, Pareto regret is already the positive part
of scalar regret, so that distinction remains necessary.

\endgroup

\subsection{Proof of the minimax theorem}
\label{app:minimax-proof}

\begin{proof}[Proof of Theorem~\ref{thm:minimax}]
For the exact class \(L^\star=L_0\),
Theorem~\ref{thm:conditional-lower} applies to every algorithm, including
algorithms that know \(L_0\). It gives
\[
 \mathcal V_{K,D,T}(L_0)
 \ge2^{-27}\min\{U_0,\sqrt{KU_0}\}.
\]
Algorithm~\ref{alg:known-scale} may use \(L_0\), and
Theorem~\ref{thm:conditional-upper} gives
\[
 \mathcal V_{K,D,T}(L_0)
 \le10\min\{U_0,\sqrt{KU_0}\}.
\]
Together these inequalities prove
\eqref{eq:conditional-curve}. Algorithm~\ref{alg:adaptive} does not know
\(L_0\), and Theorem~\ref{thm:adaptive} gives the uniform classwise bound
\[
 \E[\Rfs]\le100\min\{U_0,\sqrt{KU_0}\}
\]
for every \(L_0\), so the same order is attainable without this input.

At \(L_0=0\), \(U_0=T\), and the preceding two-sided bound is
\(\Theta(\WKT)\). For the unrestricted class of reward sequences, the
zero-loss class is a subset, so Theorem~\ref{thm:zero-lower} gives the
same lower bound. For every unrestricted sequence, \(0\le U^\star\le T\), and
the function \(u\mapsto\min\{u,\sqrt{Ku}\}\) is nondecreasing. Applying
Theorem~\ref{thm:adaptive} with \(U^\star=T-L^\star\) therefore gives
\[
 \E[\Rfs]\le100\min\{U^\star,\sqrt{KU^\star}\}
 \le100\min\{T,\sqrt{KT}\}=100\WKT.
\]
This proves the unrestricted upper bound. The displayed numerical
constants are not claimed to be sharp.
\end{proof}

\end{document}